\documentclass[12pt]{article}
\usepackage{amsmath}
\usepackage{graphicx}
\usepackage{enumerate}
\usepackage{natbib,nccmath}
\usepackage{url} 

\newcommand{\blind}{1}

\usepackage[utf8]{inputenc} 
\usepackage[T1]{fontenc}    
\usepackage{hyperref}       
\usepackage{url}            
\usepackage{booktabs}       
\usepackage{amsfonts}       
\usepackage{nicefrac}       
\usepackage{xcolor}         
\usepackage{comment}
\usepackage{amsthm}
\usepackage{enumitem}
\usepackage{graphicx}
\usepackage{bm}
\usepackage{enumerate,enumitem}
\usepackage{multirow}
\usepackage{caption}
\usepackage{subcaption}
\usepackage{float}

\usepackage{setspace}
\usepackage{amsmath}
\usepackage[ruled]{algorithm}
\usepackage{algpseudocode}
\usepackage{amssymb}
\usepackage{comment}
\usepackage{rotating}
\usepackage{mathtools}

\usepackage{authblk}
\usepackage{appendix}
\usepackage{tabularx}
\def\argmax{\text{argmax}}
\def\argmin{\text{argmin}}
\newcommand{\indep}{\perp \!\!\! \perp}

\newtheorem{theorem}{Theorem}
\newtheorem{lemma}{Lemma}
\newtheorem{proposition}{Proposition}

\newtheorem{definition}{Definition}
\newtheorem{remark}{Remark}
\newtheorem{condition}{Condition}

\def\cvgn{\mathop{\longrightarrow}_{n\to+\infty}}
\def\ds1{{\mathrm{1 \hspace{-2.6pt} I}}}
\def\dsB{\mathbb {B}}
\def\dsC{\mathbb {C}}
\def\dsE{\mathbb {E}}
\def\dsN{\mathbb {N}}
\def\dsP{\mathbb {P}}
\def\dsQ{\mathbb {Q}}
\def\dsR{\mathbb {R}}
\def\dsS{\mathbb {S}}
\def\dsV{\mathbb {V}}
\def\dsZ{\mathbb {Z}}
\def\calA{{\cal A}}
\def\calB{{\cal B}}
\def\calBX{{{\cal B}_\calX}}
\def\calBY{{{\cal B}_\calY}}
\def\calBR{{\cal B}}
\def\calBZ{{{\cal B}_\calZ}}
\def\calC{{\cal C}}
\def\tcalC{{\tilde{\cal C}}}
\def\calCm{{{\cal C}_{max}}}
\def\calD{{\cal D}}
\def\calE{{\cal E}}
\def\calF{{\cal F}}
\def\calFXY{{\cal F}}
\def\calG{{\cal G}}
\def\calH{{\cal H}}
\def\calI{{\cal I}}
\def\calK{{\cal K}}
\def\calL{{\cal L}}
\def\calM{{\cal M}}
\def\calN{{\cal N}}
\def\calP{{\cal P}}
\def\calQ{{\cal Q}}
\def\calR{{\cal R}}
\def\tcalR{{\tilde{\cal R}}}
\def\calS{{\cal S}}
\def\calT{{\cal T}}
\def\calU{{\cal U}}
\def\calV{{\cal V}}
\def\calW{{\cal W}}
\def\calX{{\cal X}}
\def\calY{{\cal Y}}
\def\calZ{{\cal Z}}

\def\Dpi{\calD^\pi}
\def\Dpin{ \hat{\calD}^\pi_n}
\def\barotimes{\,\bar{\otimes}\,}
\def\Var{\text{Var}}
\def\EE{\mathbb{E}}
\def\w{\boldsymbol{w}}
\def\e{\boldsymbol{e}}
\def\f{\boldsymbol{f}}
\def\vpsi{\boldsymbol{\psi}}
\def\XX{\textbf{X}}
\def\ZZ{\textbf{Z}}
\def\rwrd{\textbf{R}}
\def\vphi{\boldsymbol{\phi}}
\def\R{\mathbb{R}}
\def\Regret{\operatorname{Regret}}
\def\Rn{\mathbf{R}_n}
\def\Rem{\operatorname{Rem}}
\def\Yn{\mathbf{Y}_n}
\def\I{\,\mathcal{I}}
\def\ones{\mathbf{1}}
\def\zeros{\mathbf{0}}
\def\B{\mathcal{B}}
\def\argmin{\operatorname{argmin}}
\def\argminb{\operatorname*{argmin}}
\def\argmax{\operatorname{argmax}}
\def\N{\mathcal{N}}
\def\argmaxb{\operatorname*{argmax}}
\def\wcvg{\Rightarrow}
\renewcommand{\liminf}{\varliminf}
\renewcommand{\limsup}{\varlimsup}
\def\op{\operatorname{op}}
\def\X{\mathcal{X}}
\def\Y{\mathcal{Y}}
\def\bs{\mathbf{s}}
\def\jiao{\cap}
\def\bing{\cup}
\def\C{\mathcal{C}}

\newcommand{\samfixed}[1]{}

\def\linfty{l^{\infty}}
\def\Q{\mathcal{Q}}
\def\given{\, | \,}
\def\Given{\, \Big| \,}
\def\GG{\mathbb{G}}
\def\Gn{\mathbb{G}_n}
\def\V{\mathcal{V}}
\def\transpose{\top}
\def\H{\mathcal{H}}
\def\Morth{\mathcal{M}^{\perp}}
\def\opt{\text{opt}}
\def\Span{\text{Span}}
\def\Vn{\hat{\mathcal{V}}_n}
\def\Mn{\hat{M}_N}

\newcommand{\sspace}{\mathcal{S}}
\newcommand{\svar}{\mathbf{S}}
\newcommand{\aspace}{\mathcal{A}}
\newcommand{\rfun}{\mathcal{R}}
\newcommand{\rspace}{\mathbb{R}}
\newcommand{\E}{\mathbb{E}}

\usepackage{xr}
\begin{document}

\def\spacingset#1{\renewcommand{\baselinestretch}%
{#1}\small\normalsize} \spacingset{1}


\if1\blind
{
  \title{\bf Sequential Knockoffs for Variable Selection in Reinforcement Learning}
\author[1]{Tao Ma$^\ast$}
\author[1]{Jin Zhu\thanks{The first two authors contribute to this work equally.}}
\author[2]{Hengrui Cai}
\author[3]{Zhengling Qi}
\author[1]{Yunxiao Chen}
\author[1]{Chengchun Shi}
\author[4]{Eric B. Laber}
\affil[1]{London School of Economics and Political Science}
\affil[2]{University of California, Irvine}
\affil[3]{George Washington University}
\affil[4]{Duke University}
\date{}
  \maketitle
} \fi

\if0\blind
{
  \bigskip
  \bigskip
  \bigskip
  \begin{center}
    {\LARGE\bf Sequential Knockoffs for Variable Selection in Reinforcement Learning}
\end{center}
  \medskip
} \fi

\bigskip

\begin{abstract}

In real-world applications of reinforcement learning,  it is often challenging to obtain a state representation that is parsimonious and satisfies the Markov property without prior knowledge. Consequently, it is common practice to construct a state larger than necessary, e.g., by concatenating measurements over contiguous time points. However, needlessly increasing the dimension of the state may slow learning and obfuscate the learned policy. We introduce the notion of a minimal sufficient state in a Markov decision process (MDP) as the subvector of the original state under which the process remains an MDP and shares the same reward function as the original process. We propose a novel \underline{se}qu\underline{e}ntial \underline{k}nockoffs (SEEK) algorithm that estimates the minimal sufficient state in a system with high-dimensional complex nonlinear dynamics. In large samples, the proposed method achieves selection consistency. As the method is agnostic to the reinforcement learning algorithm being applied, it benefits downstream tasks such as policy learning.  Empirical experiments verify theoretical results and show the proposed approach outperforms several competing methods regarding variable selection accuracy and regret. 
\end{abstract}

\noindent%
{\it Keywords:}  Reinforcement learning, Variable selection, Sequential knockoffs, False discovery rate control, 
Power analysis
\vfill

\newpage
\spacingset{1.50} 
\section{Introduction}
\label{sec: 1}

Interest in reinforcement learning \citep[RL,][]{sutton2018reinforcement} has increased dramatically in recent years due in part to several high-profile successes in games \citep[][]{mnih2015human}, autonomous driving \citep[][]{sallab2017deep}, and precision medicine \citep[][]{tsiatis2019dynamic}.  However, despite theoretical and computational advances, real-world applications of RL remain difficult.  A primary challenge is dealing with high-dimensional state representations. Such representations occur naturally in systems with high-dimensional measurements, like images or audio, but can also occur when the system state is constructed by concatenating a series of measurements over a contiguous block of time.  A high-dimensional state---when a more parsimonious one would suffice---dilutes the efficiency of learning algorithms. 
For instance, in policy learning, it has been shown that the performance of the estimated optimal policy deteriorates rapidly with the state dimension \citep{fan2020theoretical,hu2024fast}. Meanwhile, in policy evaluation, the accuracy of the estimated value declines as state dimension grows \citep{chen2022well}\footnote{Refer to Sections \ref{sec:policy-learning} and \ref{sec:policy-evaluation} of the Supplementary Materials for formal statements.}. Thus, methods for removing uninformative or redundant variables from the state are of tremendous practical value.  

This work is partly motivated by applying RL to develop optimal treatment strategies for sepsis patients in the ICU, which concerns determining the appropriate dosage of vasopressors and intravenous fluids. Sepsis, a critical and often fatal condition, occurs when the body's immune response to an infection causes damage to its tissues and organs. Its symptoms can deteriorate quickly, leading to multiple organ failures and a rapid decline in patient health, significantly increasing the risk of mortality. Therefore, prompt and effective treatment is essential for improving patient outcomes and reducing mortality rates.
We consider the Medical Information Mart for Intensive Care III (MIMIC-III) dataset \citep{johnson2016mimic}, which contains 47 sepsis variables, including patients' clinical information and personal demographics. As highlighted in a variable importance ranking report in \citet{komorowski2018artificial}, not all variables significantly influence the optimal policy. Specifically, both RL-based and clinician-prescribed optimal policies highlight the importance of two urine output variables in making dosage decisions for vasopressors and intravenous fluids. To illustrate the usefulness of variable selection in RL, we apply the proposed method (denoted by SEEK) to this dataset alongside several baseline variable selection methods detailed in Section \ref{sec:expermentdesign}. Table~\ref{tab:mimic3} reports the estimated cumulative rewards for the policies derived from these variable selection methods. The results show that utilizing all 47 state variables yields the second lowest estimated cumulative reward. In contrast, the SEEK method identifies three variables, including the two urine output variables, yielding the highest return.

\begin{table}[!t]
\centering
	\linespread{1.25}\selectfont
	\footnotesize
  \caption{\small The number of selected variables for each variable selection method and the estimated cumulative rewards of the learned policies based on the selected variables. Deep-RL uses all variables to learn the optimal policy, whereas AE uses an estimated state representation based on an auto-encoder.}\label{tab:mimic3}
  \begin{tabular}{c|ccccc|c|c}
    \toprule
	Method & SEEK & Reward-only & One-step & SFS & VS-RF & Deep-RL & AE \\
    \midrule
    Number of selected states    & 3    & 2     & 45    & 3 &  46 & 47 & 47 \\ \hline
	Estimated cumulative reward & 0.31 & -0.01 & -1.96 & -1.26 & -2.16 & -2.34 & -12.01 \\
    \bottomrule
    \end{tabular}
\end{table}

In this paper, we develop SEEK, a general variable selection algorithm for offline RL, which aims to learn an optimal policy using only logged data without any additional online interaction.  Our contributions are as follows: (i)  we define a {\em minimal sufficient state} for an MDP and argue that it is an appropriate target by which to design and evaluate variable selection methods in RL; (ii) we show that na\"ive variable selection methods based on the state or reward alone need not recover the minimal sufficient state; (iii)  we propose a novel sequential knockoffs (SEEK) algorithm that applies with general black-box learning methods, and, under a $\beta$-mixing condition, consistently recovers the minimal sufficient state; 
and (iv) we develop a novel algorithm to estimate the $\beta$-mixing coefficient of an MDP (see Section \ref{sec:bestK} of the Supplementary Materials). The algorithm in (iv) is important in its own right as it applies to many applications beyond RL \citep{mcdonald2015estimating}. 

\vspace{-0.5cm}

\subsection{Related Work}\label{sec:related-work}
\addtolength{\textheight}{.3in}
Our work uses knockoff variables, also known as pseudo-variables, for variable selection.  Intuitively, knockoff methods augment the observed data with noise variables (knockoffs) and use the rate at which these noise variables are selected by a variable selection algorithm to characterize the algorithm's operating characteristics \citep[][]{wu2007controlling, barber2015controlling}. A key strength of knockoff methods is their generality. Knockoffs can be used to estimate (and subsequently tune) the FDR of general black-box variable selection algorithms \citep[][]{candes2018panning,  lu2018deeppink, romano2020deep, zhu2021deeplink}. Being algorithm agnostic is especially appealing in a vibrant field like RL, where algorithms are emerging and evolving rapidly. Unfortunately, it is not possible to directly export existing knockoff methods to RL because knockoffs have been developed almost exclusively in the regression setting under the assumption of independent observations, which is violated in RL settings.  One exception is a recent proposal for knockoffs in time series data \citep[][]{chi2021high}.  However, as will be shown in Section \ref{sec:minimal}, applying this one-step approach to the states in an MDP need not recover the minimal sufficient state.

In the statistics literature, RL is closely related to a growing line of works on learning optimal dynamic treatment regimes in precision medicine \citep[see, e.g.,][for an overview]{kosorok2019precision}. Several variable selection approaches have been proposed to estimate a parsimonious treatment regime \citep{gunter2011variable,qian2011performance,song2015sparse,fan2016sequential,shi2018high,zhang2018variable,qi2020multi,bian2021variable}. However, these methods are designed for a short time horizon (e.g., 2-5 time points), and do not apply to long or infinite horizons that are common in RL.  
 
Relative to its practical importance, variable selection for RL is under-explored. In the computer science literature,  
\citet{kroon2009automatic} and \citet{guo2017sample} studied variable selection in factored Markov decision processes (MDPs). 
\citet{nguyen2013online} applied penalized logistic regression for learning transition functions in structured MDPs. 
In MDPs without any specific structure, a common approach is penalized estimation.  
This includes temporal-difference (TD) learning with LASSO \citep{tibshirani1996regression} or the Dantzig selector \citep{candes2007dantzig}; see, e.g., \citet{kolter2009regularization,ghavamzadeh2011finite}, and \citet{geist2012dantzig}. In a similar spirit, 
\citet{hao2021sparse} recently proposed to combine LASSO with fitted Q-iteration 
\citep[FQI,][]{ernst2005tree}.
However, these methods are designed for linear function approximation, whereas our goal is to develop variable selection methods 
that do not rely on a specific algorithm or functional form.  
\vspace{-0.1cm}

Our proposal is also related to a line of research on abstract representation learning in the RL literature, where the goal is to construct a state abstraction, defined as a mapping from the original state space to a much smaller abstraction space that preserves many properties of the original MDP \citep[see, e.g.,][for an overview]{jiang2018notes}. In particular, bisimulation corresponds to a class of abstraction methods, where the system preserves the Markov property, and the optimal policy based on the resulting abstract MDP can achieve the same value as that of the original MDP \citep{dean1997model, li2006towards}. However, identifying the ``minimal'' bisimulation is NP-hard \citep{givan2003equivalence}. Existing solutions focus either on building an approximate bisimulation \citep{ferns2011bisimulation, ferns2014bisimulation, zhang2020learning} or learning a Markov state abstraction without the reward signal \citep{allen2021learning}. These methods 
are also related to existing dimension reduction and representation learning methods developed in the RL literature \citep[see, e.g.,][]{tangkaratt2016model, wang2017sufficient, uehara2021representation}. 

In contrast to the aforementioned state abstraction and dimension reduction methods, we focus on selecting a subvector of the original state rather than engineering new features.  This approach is advantageous in settings such as mobile health, where an interpretable learned policy is important for generating new clinical insights.  However,
our methods can also be used with basis expansions constructed from the original state, such as splines or tile codings. Another advantage of our approach is its computational efficiency; we show that SEEK is guaranteed
to terminate after no more than $d$ iterations, where $d$ is the dimension of the original state.

\vspace{-0.5cm}
\subsection{Organization of the Paper}
	
In Section \ref{sec:2}, we introduce the Model-X knockoffs for variable selection under a contextual bandit setting \citep{lu2010contextual} and describe the challenges of adapting this framework to the RL setting. In Section \ref{sec:minimal}, we formulate the variable selection problem in RL and introduce the concept of a minimal sufficient state. In Section \ref{sec:seek}, we present the sequential knockoffs procedure for identifying the minimal sufficient state. In Section \ref{sec:theory}, we show that, asymptotically, our method will not select null variables, and its power approaches one. We investigate the empirical performance of our method via extensive simulations and real data analysis in Sections \ref{sec:exp} and \ref{sec:mimic3}.

{\singlespacing
\section{Preliminaries: Contextual Bandits, Variable Selection with Knockoffs and Challenges of Adapting to RL}\label{sec:2}
}

We denote a contextual bandit model by a triplet $\mathcal{B}= (\sspace, \aspace, \rfun)$ with a compact contextual space $\sspace = \prod_{j=1}^d\sspace_j \subseteq \rspace^d$, a discrete action space $\aspace$ and a reward function $\rfun: \sspace \times \aspace \to \rspace$ that characterizes the conditional mean of the reward $R$ given the contextual information $\mathbf{S} \in \sspace$ and the action $A \in \aspace$. We further assume that the reward is bounded. 

Consider an offline setting with access to a historical dataset $\mathcal{D}$, 
consisting of independent and identically distributed (i.i.d.) copies of contextual-action-reward triplet $(\svar, A, R)$.  
These triplets are generated as follows: at each time, an agent observes a contextual variable $\svar$ drawn from $\sspace$, chooses an action $A$ according to a behavior policy $\pi_b(\cdot|\svar)$ that corresponds to the conditional probability mass function (pmf) of $A$ given $\svar$, and receives an immediate reward $R$ with mean  $\rfun(\svar, A)$ for the chosen action $A$. The optimal policy is defined as the pmf $\pi^*$ that maximizes the resulting expected reward, expressed as $\sum_a \E [\pi^*(a|\svar) \rfun(\svar,a)]$. With some calculations, it is straightforward to show that $\pi^*$ will assign probability $1$ to the action $a$ that maximizes $\rfun(\svar,a)$ for each $\svar$.

We aim to select a subset of contextual variables that determine the optimal policy $\pi^*$. Given that $\pi^*$ is uniquely determined by the reward function $\rfun$, it suffices to select those that appear in $\rfun$. 
Consequently, we say $\svar_j$ is a null variable if it is conditionally independent of $\rfun(a,\svar)$ for any $a$ given the other variables; a non-null variable is said to be significant.  Let $\mathcal{H}_0 \subseteq \left\lbrace 1,\ldots, d\right\rbrace$ be the indices of the null variables and $\mathcal{I} = \mathcal{H}_0^c$ the indices of significant variables. Given an estimator $\widehat{\mathcal{I}}$  of $\mathcal{I}$ constructed from $\mathcal{D}$ and some pre-specified level $q\in (0,1)$,  the false discovery rate (FDR) and modified FDR (mFDR) associated with $\widehat{\mathcal{I}}$ are given by
$\mathbb{E} \lbrace \# (\widehat{\mathcal{I}} \bigcap \mathcal{H}_0 )/ \max(1, \#\widehat{\mathcal{I}})\rbrace$ 
and $\mathbb{E} \lbrace \# (\widehat{\mathcal{I}} \bigcap \mathcal{H}_0 )/ (1/q+\#(\widehat{\mathcal{I}})) \rbrace$, respectively, where $\#(\cdot)$ denotes the cardinality of a set.  The knockoff approach aims to construct an estimator of $\mathcal{I}$ that ensures the (m)FDR to be below $q$.

A key innovation of the Model-X knockoff method \citep{candes2018panning} lies in the construction of knockoff variables. When adapted to contextual bandits, we define a knockoff vector $\widetilde{\mathbf{S}} = (\widetilde{\mathbf{S}}_1,\ldots,\widetilde{\mathbf{S}}_d )^\top\in\mathbb{R}^d$ which mimics the original $\mathbf{S}$ except that they are conditionally independent of the rewards given the action and thus are recognized as null variables. These knockoffs are constructed so that: (1) $\widetilde{\mathbf{S}}\indep R \given \mathbf{S}, A$ (i.e., conditional independence of the reward); (2) $(\mathbf{S}, \widetilde{\mathbf{S}} )_{\text{swap}(B)} \given A \overset{d}{=} ( \mathbf{S}, \widetilde{\mathbf{S}}) \given A$ for any subset $B\subset \{1,2,\ldots,d\}$ (i.e., exchangeability  in distribution), where $(\mathbf{S}, \widetilde{\mathbf{S}} )_{\text{swap}(B)}$ is the vector formed by switching the $j$th entries of $\mathbf{S}$ and $\widetilde{\mathbf{S}}$ for each $j\in B$. Given $\widetilde{\mathbf{S}}$, we obtain an augmented dataset $\widetilde{\mathcal{D}}$ that contains the original dataset $\mathcal{D}$ and the knockoff contextual variables.

From $\widetilde{\mathcal{D}}$ we compute feature importance statistics $Z_j$ and $\widetilde{Z}_j$ for each variable $\mathbf{S}_j$ and its knockoff $\widetilde{\mathbf{S}}_j$, respectively, and define $W_j = Z_j - \widetilde{Z}_j$ so that a higher value of $W_j$ is associated with stronger evidence that $\mathbf{S}_j$ is significant. Given a target FDR level $q$, the $j$th variable is selected if $W_j$ is no less than some threshold $\tau_q$ (or $\tau_{q+}$ in a knockoffs+ procedure). 

Model-X knockoffs possess several desirable statistical properties.  First, it provides non-asymptotic control of the FDR given the covariate distribution.  Second, it does not depend on any specific model structure between the outcome and covariates; thus, the method applies to nonlinear and high-dimensional settings.  Finally, the power is nearly optimal \citep[][]{fan2019rank, weinstein2020power, wang2020power, ke2020power}.  However, the derivation of these theoretical properties relies critically on the assumption that the observed data are independent -- an assumption violated in RL. Another challenge is that the definition of a significant variable in an MDP is less obvious.  Attempting to mimic the contextual bandit setting, for example, by defining significant variables as those associated with the reward, will fail to capture variables that impact cumulative reward indirectly through state transitions.  Conversely, suppose one defines the ``outcome'' as the reward concatenated with the next state. In that case, the knockoff procedure will incorrectly select those associated with the next state while having no effect on the (present or future) reward. We propose to define variable significance in RL based on the concept of the minimal sufficient state, introduced in the next section.

\vspace{-0.5cm}
\section{Minimal Sufficient State and Model-based Selections}\label{sec:minimal}
In this section, we switch to the RL setting where the data, summarized as a sequence of state-action-reward triplets over time $\{(\svar_t,A_t,R_t)\}_{t\ge 0}$, is modeled by a time-homogeneous MDP \citep{puterman2014markov}, denoted by $\mathcal{M} = (\mathcal{S}, \mathcal{A}, \rfun, \mathcal{P}, \gamma )$. Here, $\mathcal{S}$ refers to the same $d$-dimensional compact subspace as in the contextual bandit model described in Section \ref{sec:2}, now termed the state space to reflect the sequential nature of the data. The action space $\aspace$ and reward function $\rfun$ remain unchanged from those described in the bandit model. Additionally, $\mathcal{P}:\mathcal{S}  \times \mathcal{S} \times \mathcal{A} \rightarrow \mathbb{R}^+$ denotes the state transition function, representing the probability density function (pdf) or pmf of $\svar_{t+1}$ given $\svar_t$ and $A_t$. Lastly, $\gamma\in[0,1)$ denotes the discount factor, which balances the trade-off between immediate and long-term rewards. Specifically, for a given policy $\pi$, we measure its return by the $\gamma$-discounted expect cumulative reward $\mathbb{E}^\pi(\sum_{t=0}^{\infty} \gamma^t R_t)$ where $\mathbb{E}^\pi$ denotes the expectation assuming that the actions are chosen according to $\pi$. Under the MDP assumption, the dynamic system is Markovian such that both the next state $\mathbf{S}_{t+1}$ and the conditional mean function of the immediate reward $R_t$ are independent of the data history given the current state and action, $\mathbf{S}_t$ and $A_t$. Throughout this paper, we use $(\mathbf{S}, A, R, \mathbf{S}^\prime)$ to refer to a generic state-action-reward-next-state tuple. For any  $\mathbf{\nu} \in \mathbb{R}^d$ and index set $G\subseteq \{1, 2, \ldots, d\}$, let $\mathbf{\nu}_{G} = \{\nu_{g} \, \mid \, g \in G \}$; thus, $\mathbf{S}_{t,G}$ is the subvector of state $\mathbf{S}_t$ indexed by components in $G$. Let $G^c = \lbrace 1,\ldots, d\rbrace \setminus G$ be the complement of $G$. 
\begin{definition}[Sufficient State]\label{def:sufficient-state}
   We say $\mathbf{S}_{G}$ is a sufficient state in an MDP if, for all $t \ge 0$, (i) $\mathbb{E}(R_{t} \given \mathbf{S}_{t}, A_t) = \mathbb{E}(R_{t} \given \mathbf{S}_{t, G}, A_t)$ almost surely and (ii) $\mathbf{S}_{t+1,G} \indep \mathbf{S}_{t, G^c} \given (\mathbf{S}_{t, G}, A_t)$ hold. 
\end{definition}
This definition of sufficiency is weaker than the condition that $R_t$ and $\mathbf{S}_{t+1,G}$ are \textit{jointly} conditionally independent of $\mathbf{S}_{t, G^c}$. In addition, the first condition is weaker than requiring $R_t \indep \mathbf{S}_{t, G^c} \given (\mathbf{S}_{t, G}, A_t)$. Essentially, it identifies significant variables that influence the conditional mean of reward rather than other aspects of the conditional distribution. Furthermore, a sufficient state always exists as $\mathbf{S}_t$ itself is sufficient. The following results justify the use of the term {\em sufficient} in Definition \ref{def:sufficient-state}.
\begin{proposition}\label{prop1}
	If $\mathbf{S}_{G}$ is a sufficient state, then there exists an optimal policy depending only on  $\mathbf{S}_G$ that maximizes the $\gamma$-discounted expected cumulative reward.
\end{proposition}
\begin{proposition}\label{prop2}
	If $\mathbf{S}_{G}$ is a sufficient state, then the process $\{\left(\mathbf{S}_{t, G}, A_t, R_{t}\right)\}_{t\geq0}$ forms a time-homogeneous MDP. 
\end{proposition}
Proposition \ref{prop1} allows us to focus on policies that are functions of sufficient states only. Most existing RL algorithms require the data-generating process to follow an MDP model. Proposition \ref{prop2} implies that once a sufficient state is identified, these algorithms can be directly applied 
to the reduced process $\{\left(\mathbf{S}_{t, G}, A_t, R_{t}\right)\}_{t\geq0}$. 
As discussed in the introduction, a low-dimensional sufficient state can enhance the performance of policy optimization and reduce computational costs. Therefore, we aim to 
identify a \textit{minimal sufficient state}, defined as follows.

\begin{definition}[Minimal Sufficient State]
	We say $\mathbf{S}_G$ is a minimal sufficient state if it is the sufficient state with the minimal dimension among all of the sufficient states.
\end{definition}

The proposed minimal sufficient state is model-based, defined by the reward and state transition functions (see Definition \ref{def:sufficient-state}), and tailored for policy learning, as reflected in Proposition \ref{prop1}. To the best of our knowledge, it has not yet been introduced in the literature. Specifically, existing variable selection approaches are typically model-free, focusing on identifying important variables within the Q- or value function \citep[see e.g.,][]{hao2021sparse}. Moreover, many of these approaches are designed primarily for policy evaluation rather than policy learning \citep{kolter2009regularization,ghavamzadeh2011finite,geist2012dantzig}. We discuss this further in Section \ref{sec:modelbasedmodelfree} of the Supplementary Materials. Below, we show that the minimal sufficient state is well-defined and unique under mild conditions.
\begin{proposition}\label{prop3}
	The minimal sufficient state always exists. Furthermore, if the 
	transition kernel is strictly positive, the minimal sufficient state is unique. 
\end{proposition}
Our goal is to identify the minimal sufficient state using batch data. To provide insights into this problem and motivate the proposed approach, we first discuss two extensions of supervised variable selection methods to MDPs that are not guaranteed to recover the minimal sufficient state. 

\textbf{Reward-only approach}: One approach is to  
apply the variable selection methods for the contextual bandit model described in Section \ref{sec:2}  
to the MDP setting. Specifically, we treat the state and action as the covariates and the reward as the outcome, and then apply a regression-based variable selection method as if the data were comprised of i.i.d. input (state, action) and output (reward) pairs. This approach will fail to identify state variables that do not affect the immediate reward but rather 
affect long-term rewards through their impact on the next state. Thus, this approach can have a high false negative rate and, thus, need not be consistent for selecting a sufficient state. Consequently, a process indexed by a state constructed using this approach may not satisfy the Markov property.  

\textbf{One-step approach}: To capture components of the state associated with both short- and long-term impacts on the reward, one can treat the reward and the next state jointly as a multivariate outcome and the current state-action pair as the input. With this setup, one can apply variable selection techniques for multiple-outcome regression. This approach was considered by \cite{tangkaratt2016model} in the context of dimension reduction.  While this approach can consistently remove pure noise variables, it can fail to remove time-dependent variables that do not affect cumulative reward. Consequently, as shown below, it will identify a sufficient state but need not identify a minimally sufficient one. 

\textbf{Proposed iterative approach}: The one-step approach fails to identify the minimal sufficient state because the response includes the entire next state vector (including those unrelated to the cumulative reward). Conceptually, one could avoid this by taking every possible subset $G$ of $\left \lbrace 1,\ldots, d \right\rbrace$ and testing the sufficiency condition in Definition \ref{def:sufficient-state}.
However, such an approach is not scalable as it requires examining $2^d$ subsets and conducting a multivariate conditional independence test for each subset. As a result, we propose a sequential knockoff procedure that is (asymptotically) equivalent to such a procedure while requiring at most $d$ iterations to converge.  At a high level, the algorithm proceeds as follows. At the first iteration, a variable selection algorithm is applied to identify $\widehat G_1$, the significant state variables for predicting the reward function. The second iteration uses the selection algorithm to identify $\widehat G_2$, the significant state variables for predicting $\mathbf{S}_{t+1,\widehat G_1}$. Then recursively for $j\geq 2$, we identify the set of variables $\widehat G_j$ which predict $\mathbf{S}_{t+1, \widehat G_{j-1}}$. The algorithm terminates when $\widehat G_j = \widehat G_{j+1}$. As the sets $\widehat G_{j}$ are monotonically non-decreasing, the algorithm terminates in no more than $d$ steps.

\textbf{A numerical example}: We conduct a simulation example to illustrate the advantage of the iterative selection method. In this example, the dimension of the state is $20$, and the index set of the minimal sufficient state is $G_{\text{M}}= \{1,2\}$.
The generative model is such that $S_{t, 1}$ directly influences $R_t$, whereas $S_{t, 2}$  affects $S_{t+1,1}$ and hence has an indirect effect on the reward.
The remaining state variables can be divided into two groups. The first
group, $G_{\text{AR}}=\{3,\ldots,11\}$, indexes state variables whose transitions are independent and follow a first-order autoregressive (AR(1)) model, i.e., \useshortskip
\begin{align*}
	\mathbb{P}(S_{t+1,j}\le s|\mathbf{S}_t,A_t) = \mathbb{P}(S_{t+1,j}\le s|S_{t,j}), \qquad \textrm{for all}~j= 3,\ldots,11\,\,\hbox{and}\,\, s\in \mathbb{R}.
\end{align*}\vspace{-4em}

\noindent The second group $G_{\text{WN}}=\{12,\ldots,20\}$ indexes state variables which are i.i.d. white noise across all decision points.  
We conduct a numerical study using the above three approaches and summarize the results in Figure \ref{fig:toy}  and Table \ref{tab:toy} under a fixed target FDR of 0.3. These results are obtained by aggregating over 20 independent simulations. The implementation details of these methods are presented in Section \ref{sec:seek}. It can be seen from Table \ref{tab:toy} that
the reward-only method fails to select $S_{t, 2}$, as expected. Furthermore, the one-step method selects all variables in $G_{\text{AR}}$ because they are time-dependent and contribute to the state transition function. In contrast, the iterative procedure (based on the proposed algorithm) avoids selecting redundant variables in $G_{\text{AR}} \cup G_{\text{WN}}$ and thus is more appealing. It can be seen from Figure \ref{fig:toy} that the proposed algorithm controls the FDR and has a much larger area under the receiver operating characteristic curve among the three methods.


\begin{figure}
	\centering
	\begin{subfigure}{0.45\linewidth}
		\centering
		\includegraphics[width=1\linewidth, height=4.5cm]{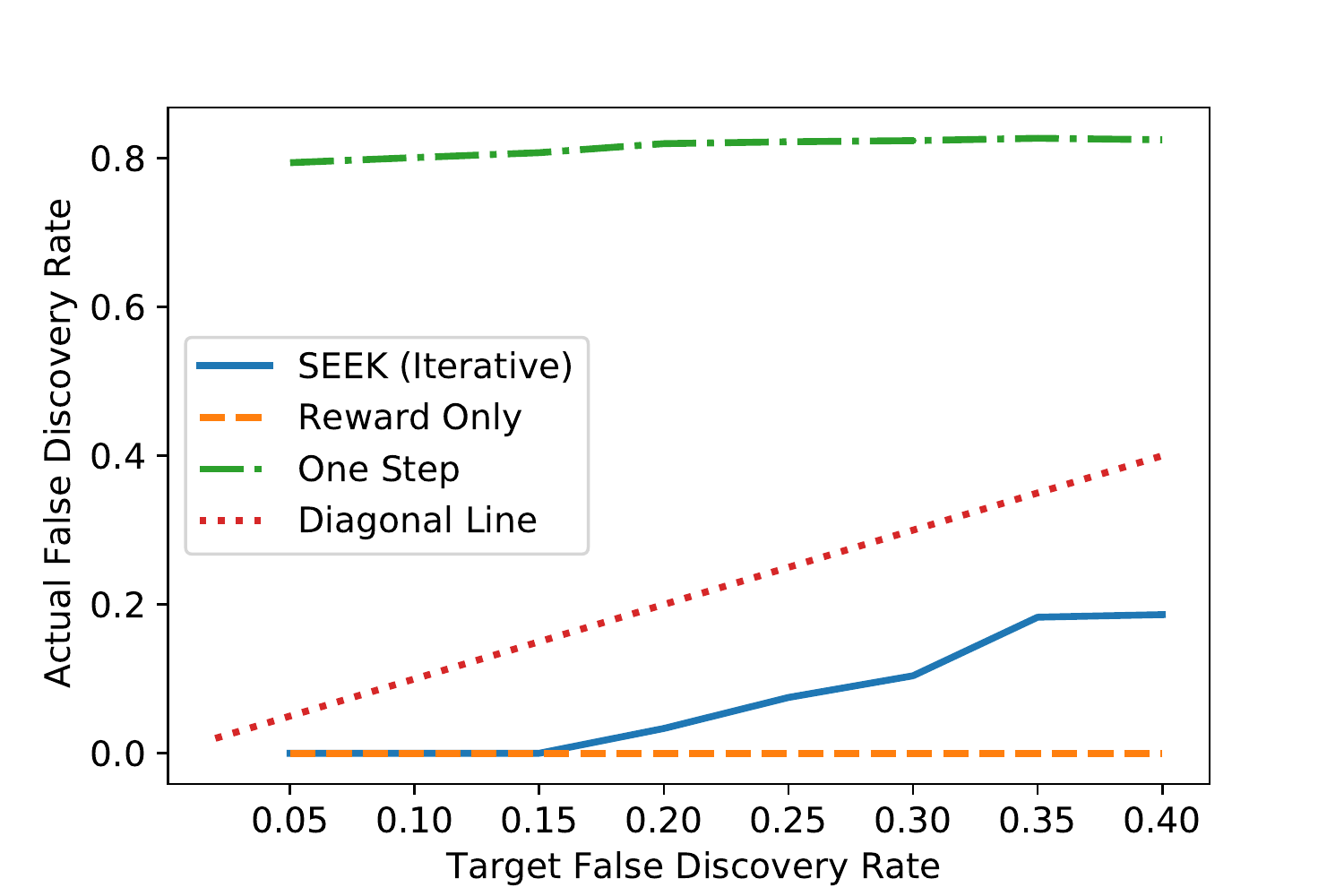} 
	\end{subfigure}%
	\begin{subfigure}{0.45\linewidth}
		\centering
		\includegraphics[width=1\linewidth, height=4.5cm]{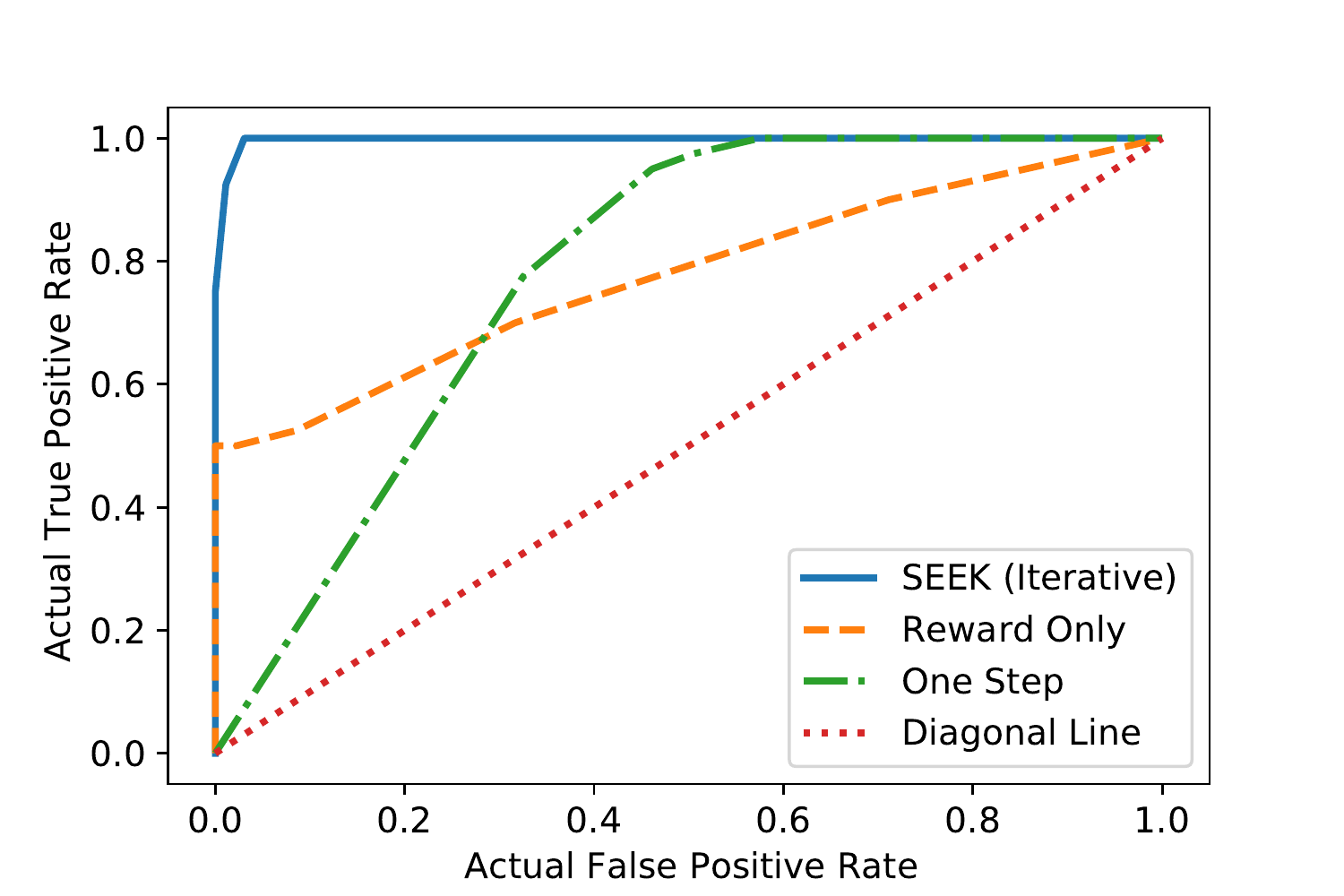} 
	\end{subfigure}%
	\caption{\small Receiver operating characteristic (ROC) curves for different methods in the toy example.}
\label{fig:toy}
\end{figure}


\begin{table}
	\linespread{1.25}\selectfont
	\caption{\small Most frequently selected states by different methods in the toy example.}\label{tab:toy}
	\centering
	\footnotesize
	\begin{tabular}{l|ccc}
		\toprule  
		Method     & SEEK & Reward-only & One-step 
		\\
		\midrule
		Selected  States & \{1, 2\} & \{1\} & \{1, 2, 3, 4, 5, 6, 7, 8, 9, 10, 11\} \\
		\bottomrule
	\end{tabular}
\end{table}


The following proposition summarizes the properties of these three approaches at the population level, i.e., assuming infinitely many observations.

\begin{proposition}\label{prop4}
Assume that one applies a 
selection-consistent algorithm for the above three approaches. Then, (i) there exists an MDP such that the state selected by the reward-only method is not sufficient, and (ii) the state selected by the one-step method is sufficient but not minimally sufficient.  In addition, the iterative approach
can correctly recover the minimal sufficient state for any MDP provided the
transition kernel is strictly positive.   
\end{proposition}

The preceding result does not depend on the selection algorithm being applied, except for its selection consistency. Thus, this result is of independent interest as it provides a template for constructing variable selection methods for MDPs.  In the remainder of this paper, we focus on developing the \underline{se}qu\underline{e}ntial \underline{k}nockoffs (SEEK) algorithm, which extends Model-X knockoffs to the MDP setting. This algorithm will be a building block of the iterative procedure described earlier. 

\begin{remark}
    Our use of knockoffs is primarily motivated by its capacity to control the FDR or type-I error. It is especially appealing when the objective is to identify the minimum sufficient state or its subset in scientific discoveries, such as gaining clinical insights with batch data collected in biomedical applications, including mobile health, where meaningful interpretations are crucial. Suppose one is interested in training an optimal policy that maximizes the cumulative reward. It is important to balance both type-I and type-II errors, as managing false negatives becomes equally important. Our method can accommodate this scenario well; see the last paragraph of Section \ref{sec:seek} for more details. Finally, if the goal is solely to screen out the state variables that are not in the minimum sufficient state, emphasis should be shifted towards controlling the type-II error. To address this, we develop an alternative method for identifying the variables in the minimal sufficient state with controlled type-II error; see Section~\ref{sec:control-type-II-error} of the Supplementary Materials.
\end{remark}

\vspace{-0.5cm}
\section{SEEK: Sequential Knockoffs for Variable Selection}
\label{sec:seek}
This section introduces the SEEK algorithm, which consists of three key steps. An outline is displayed in Algorithm \ref{alg:1}. 
Suppose we are given batch data $\calD$ consisting of $N$ i.i.d. finite-horizon trajectories, each 
containing $T$ state-action-reward-next-state transition tuples. Note that observations within each trajectory are temporally dependent. Similar to the setting in Section \ref{sec:2}, the actions in $\calD$ are generated by a behavior policy $\pi_b$. We assume $\pi_b$ is stationary, i.e., the distribution of each action depends on the history only through the current state, and this dependence is invariant over time. Meanwhile, our method can accommodate cases where $T$ differs across trajectories; see Section~\ref{sec:bestK-different-horizon} of the Supplementary Materials for details. Indeed, we managed such cases in our MIMIC-III data analysis.

\textbf{Step 1: Data splitting}. 
We first split all $NT$ transition tuples in the batch data into $K$ non-overlapping sub-datasets, leading to a partition $\calD=\{\calD_k, k\in[K]\}$ such that for each tuple, $\left(\mathbf{S}_{t}, A_{t}, R_{t}, \mathbf{S}_{t+1}\right) \in \calD_k$ if $t \, \textrm{mod} \, K = {k-1}$. After the split, every two tuples within the same subset $\calD_k$ either come from the same trajectory with a time gap of at least $K$ or belong to two different trajectories. When the system is $\beta$-mixing \citep{bradley2005basic}, by an innovative use of Berbee's coupling lemma \citep{berbee1987convergence} and a careful choice of $K$, we can guarantee that all transition tuples within each subset $\calD_k$ are ``\textit{approximately independent}''. 
See Section \ref{apdx:proof} of the Supplementary Materials for a rigorous statement of this result. It justifies our iterative use of Model-X knockoffs on each $\calD_k$  for finding the minimal sufficient state in the following steps. Finally, we aggregate the selection results from all the data subsets using a majority vote.

\textbf{Step 2: Iterative selection}. 
On each data subset $\calD_k$, we implement our iterative selection approach using Model-X knockoffs. We first apply the aggregated knockoffs procedure (across actions, see Algorithm \ref{alg:2}) to select the significant state variables for predicting the \textit{immediate reward} only. Denote this initial selection by $\widehat{G}_{k,1}$. Subsequently, we treat each selected state variable in $\widehat{G}_{k,1}$ in the next decision point as the response and use aggregated knockoffs again to select state variables that are significant to the response. The final index set is denoted by $\widehat{G}_{k,2}$. We repeat this procedure to obtain $\widehat{G}_{k,3}, \widehat{G}_{k,4}, \ldots,$ where at each iteration we combine all previously selected state variables in the next time point to construct the new responses, until no more significant variables are found. The resulting index set of selected variables is denoted by $\widehat{G}_k$.

 In our implementation, we adopt a second-order machine to construct knockoff variables. However, other construction approaches are feasible \citep{romano2020deep}. 
A key difference from the regression framework is that our knockoffs must be constructed in an ``action-wise'' manner. This difference stems from our focus on states that directly impact the reward and future relevant states rather than those influencing these outcomes indirectly through the action. Consequently, it is crucial to ensure that the independence $\widetilde{\svar}\indep R \given A,\svar$ and exchangeability conditions $(\mathbf{S},\widetilde{\mathbf{S}})_{\textrm{swap}(B)} \given A\stackrel{d}{=}(\mathbf{S},\widetilde{\mathbf{S}}) \given A$ are maintained, conditionally given the action, rather than marginally. As outlined in Step 1 of Algorithm~\ref{alg:2}, by dividing the data into subsets $\{\mathcal{D}_k^{(a)}\}$ grouped by action and separately constructing knockoff variables for each $\mathcal{D}_k^{(a)}$, we achieve both conditional independence and exchangeability. Formal justifications are provided in Section \ref{sec:secondorder} of the Supplementary Materials. We also remark that when the action space is small, such an action-wise approach is commonly used in RL for both policy learning \citep[see, e.g.,][]{ernst2005tree} and evaluation \citep[see, e.g.,][]{shi2020statistical}. 
 
Finally, for linear systems, we set each feature importance statistics $Z_j$ (or $\widetilde{Z}_j$) to be the $\ell_{\infty}$-norm of the estimated regression coefficients obtained via LASSO across different actions. In nonlinear systems, we set $Z_j$ (or $\widetilde{Z}_j$) by the variable importance with general machine learning methods. By setting $W_j=Z_j-\widetilde{Z}_j$, it automatically satisfies the flip-sign property, which is crucial for knockoffs to control the FDR/mFDR; see details in Sections \ref{sec:genericML} and \ref{sec:flipsignproperty} of the Supplementary Materials.

\renewcommand{\algorithmicrequire}{\textbf{Input:}}
\renewcommand{\algorithmicensure}{\textbf{Output:}}
\begin{algorithm}[!t]
\linespread{1.2}\selectfont
\caption{\underline{Se}qu\underline{e}ntial \underline{K}nockoffs (SEEK)}\label{alg:1}
\begin{algorithmic}[1]
\Require Batch data $\calD$, number of data subsets $K$, a target FDR level $q \in (0, 0.5)$, and a threshold $\alpha\in(0,1)$ for the majority vote.
\State Split $\calD$ evenly into $\{\calD_k, k\in[K]\}$. Initialize all $\widehat{G}_k$'s as empty sets.
\For {$k=1,2,\ldots,K$}
\State Apply Algorithm \ref{alg:2} to all $(\mathbf{S}, A, R)$ tuples in $\calD_k$ with the target FDR level $q$. Denote the selected index set by $\widehat{G}_{k, 1}$ and let $\widehat{G}_k = \widehat{G}_{k, 1}$. 
\For{$j = 2, 3, \ldots$}
\State For each $l \in \widehat{G}_{k,j-1}$, apply Algorithm \ref{alg:2} to all $(\mathbf{S}, A, Y)$ tuples in $\calD_k$ with $Y= S_{l}'$ and the target FDR level $q$. Denote the union of the selected index sets by $\widehat{G}_{k, j}$. 
\State \textbf{Break} if $\widehat{G}_{k, j} \subseteq \widehat{G}_{k}$.
\State Let $\widehat{G}_k = \widehat{G}_k \cup \widehat{G}_{k, j}$. 
\EndFor
\EndFor
\Ensure $\widehat{G} := \{j \in [d]: \sum_{1\leq k\leq K}\mathbb{I}(j \in \widehat{G}_k) \geq \alpha K\}.$
\end{algorithmic} 
\end{algorithm}

\begin{algorithm}[t]
\linespread{1.2}\selectfont
\caption{Aggregated Knockoffs} 
\label{alg:2}
\begin{algorithmic}[1]
\Require Batch data $\calD_k$, and a target FDR level $q \in (0, 0.5)$.
\State Split $\calD_k$ into $\{\calD_k^{(a)},a\in \mathcal{A}\}$ where each $\calD_k^{(a)}$ contains state-response $(\svar,Y)$ pairs for which the action is equal to $a$. 
\State For each $(\svar,Y)$ pair in $\calD_k^{(a)}$, construct the associated knockoff variable $\widetilde{\mathbf{S}}$. 
\State Apply any machine learning algorithms (e.g., LASSO or random forest) to all $(\mathbf{S}, \widetilde{\mathbf{S}}, Y)$ triplets to construct feature importance statistics $Z_j^{(a)}$ and $\widetilde{Z}_j^{(a)}$ for $S_j$ and its knockoff variable $\widetilde{S}_j$, respectively, for each $j\in [d]$.
\State For each $j \in [d]$, set $Z_j=\max\limits_{a \in \mathcal{A}} Z_j^{(a)}$, $\widetilde{Z}_j=\max\limits_{a \in \mathcal{A}} \widetilde{Z}_j^{(a)}$ and $W_j = Z_j - \widetilde{Z}_j$.
\State To implement the (standard) knockoffs method, set the threshold $\tau$ to 
\vspace*{-0.6em}
\begin{align*}
	\tau = \min\left\{t>0: |\{i\in [d]: W_i \leq -t\}| \leq q |\{i\in [d]: W_i \geq t\}| \right\},
\end{align*}
\Ensure $\widehat{G}_k = \{i\in [d]: W_i \geq \tau\}$.
\State To implement the knockoffs+ method, set the threshold $\tau_+$ to: 
\vspace*{-0.6em}
\begin{align*}
	\tau_+ = \min\left\{t>0: 1 + |\{i\in [d]: W_i \leq -t\}| \leq q |\{i\in [d]: W_i \geq t\}| \right\}, 
\end{align*}
\Ensure $\widehat{G}_k = \{i\in [d]: W_i \geq \tau_+\}$.
\end{algorithmic} 
\end{algorithm}

\textbf{Step 3: Majority vote.}  Compute the proportion of subsets in which variable $j$ is selected, $\widehat{p}_j = K^{-1}\sum_{k=1}^K \mathbb{I}\big(
j \in \widehat{G}_k
\big)$, and define the final selected set 
$\widehat{G} = \left\lbrace j\,:\, \widehat{p}_j \ge \alpha \right\rbrace$,
where $\alpha \in (0, 1)$ is a pre-specified threshold (e.g., 0.5).

To conclude this section, we notice that each step of our algorithm relies on a potentially important tuning parameter: the number of splits $K$ in Step 1, the target FDR level $q$ in Step 2, and the majority vote threshold $\alpha$ in Step 3. To select $K$, we provide an adaptive (data-driven) procedure in Section \ref{sec:bestK} of the Supplementary Materials. Meanwhile, the latter two parameters can be chosen to balance the type-I and type-II errors so as to optimize the performance of the resulting estimated optimal policy. Specifically, we consider finite candidate values for $q$ and/or $\alpha$. We run the proposed algorithm for each candidate of $q$ and/or $\alpha$ and obtain the corresponding estimated optimal policy. We next evaluate these policies on an additional validation set and estimate their cumulative reward. Finally, we choose the FDR target and/or the majority vote threshold that leads to the highest cumulative reward and rerun SEEK on the entire dataset to compute the selected variables and estimate the optimal policy. This procedure has been implemented in our simulation study, which demonstrates that this procedure effectively improves the performance of the estimated optimal policy in small samples. For detailed results and discussions, refer to Section~\ref{sec:simulation-balancing} of the Supplementary Materials.


{\singlespacing\section{Theoretical Results}
\label{sec:theory}

\subsection{FDR and Type-I Error}\label{sec:false-discovery-analysis}}
This section analyzes SEEK's properties related to the FDR and type-I error. We begin by introducing some technical conditions. 
\begin{condition}[Stationarity and exponential $\beta$-mixing]\label{ass:mixing}
The process $\{(\mathbf{S}_t, A_t, R_t)\}_{t\ge 0}$ is stationary and exponential $\beta$-mixing, i.e., its $\beta$-mixing coefficient at time lag $k$ is of the order $O(\rho^k)$ and $0<\rho<1$. 
\end{condition}

Since the behavior policy $\pi_b$ is stationary, the exponential $\beta$-mixing condition in Condition \ref{ass:mixing} is equivalent to requiring the Markov chain $\{(\mathbf{S}_t, A_t, R_t)\}_{t\ge 0}$ to be geometrically ergodic \citep{bradley2005basic}. Similar conditions have been frequently employed in the RL literature for theoretical analysis \citep[see e.g.,][]{antos2007fitted,luckett2019estimating,liao2021off,chen2022well,kallus2022efficiently,shi2020statistical}. These conditions allow estimators to achieve a rate of convergence nearly equivalent to those with i.i.d. data. Meanwhile, they can be relaxed to weaker conditions, such as polynomial $\beta$-mixing, detailed further in Section~\ref{sec:poly-beta-mixing} of the Supplementary Materials.

\begin{condition}[External dataset]\label{ass:external-dataset}
    An external dataset $\mathcal{U}$, consisting of state-action-next-state triplets independent of $\mathcal{D}$, is available to construct the knockoff variables $\widetilde{\mathbf{S}}$. Note that this dataset 
    can be unlabeled because the construction of the knockoff variables does not need reward data. 
    
\end{condition}
Condition \ref{ass:external-dataset} simplifies our theoretical analysis by using the external dataset, $\mathcal{U}$, to construct the knockoff variables. The independence between $\mathcal{U}$ and $\mathcal{D}$ eliminates the stringent requirement for a known covariate distribution in the original Model-X knockoff framework. Note that $\mathcal{U}$ can be an unlabeled dataset, which is often easy to obtain. In practice, we are often provided with a small amount of high-quality reward annotated demonstration data, and a huge amount of unlabeled data for policy learning \citep{yu2022leverage}. For instance, in fine-tuning language models like Chat-GPT, there exists abundant unlabeled text data, which can reach up to 852 million words. In healthcare applications, the quantity of unlabeled electronic health records can be up to ten times greater than that of labeled data \citep{sonabend2023semi}. Similar assumptions are frequently imposed under the regression framework \citep{barber2020robust,huang2020relaxing}. Furthermore, even without an external dataset, cross-fitting can be applied to achieve comparable results \citep{fan2019rank,chi2021high,fan2023ark}.
\begin{condition}[Null variable distributional irrelevance]\label{condnull}
    Any null variable maintains conditional independence of the reward, given the minimal sufficient state and action. 
\end{condition}
By definition, the conditional mean function of the reward does not depend on the null variables when the significant variables and the action are given. Condition \ref{condnull} essentially requires that these null variables should not impact the conditional variance or higher moments of the reward either. If this condition is not met, SEEK may still identify a sufficient state but not necessarily the minimal sufficient state. This limitation arises from the use of knockoffs in Line 3 of Algorithm \ref{alg:1} to filter out variables irrelevant to the reward; however, knockoffs target variables that are conditionally independent rather than those only having no effect on the conditional mean. To overcome this limitation, one could employ tests for conditional mean independence to remove null variables more effectively.

\begin{theorem}[FDR]\label{thm1}
Suppose Conditions \ref{ass:mixing} to \ref{condnull} hold. Set the number of sample splits $K=k_0 \log(NT)$ for some $k_0>- \log^{-1}(\rho)$, where $\rho$ is defined in Condition \ref{ass:mixing}. Then for any response $Y$ ($R$ or $S'_{l}$ for a given $1\le l\le d$), $\widehat{G}_k$ obtained by Algorithm \ref{alg:2} with standard knockoffs satisfies  
\begin{eqnarray*}
	\textrm{mFDR}(\widehat{G}_k)\leq q \exp(\epsilon)+\mathbb{P}\Big(\max_{j\in \mathcal{H}_0}\widehat{\textrm{KL}}_{j}>\epsilon\Big)+O\Big\{ K^{-1} (NT)^{1-k_0 \log(\rho^{-1})} \Big\},
\end{eqnarray*}\vspace{-3.5em}

\noindent for any $\epsilon>0$, where $\widehat{\textrm{KL}}_{j}$ denotes the observed KL divergence \citep{barber2020robust} 
given by\vspace{-1em}
\begin{eqnarray*}
    \sum_{a\in \aspace}\sum_{(\mathbf{S},\widetilde{\mathbf{S}})\in \mathcal{D}_k^{(a)}} \log\Big[\frac{p_j(S_j|\mathbf{S}_{-j},A=a)}{q_j(S_j|\mathbf{S}_{-j},A=a)}\frac{q_j(\widetilde{S}_j|\mathbf{S}_{-j},A=a)}{p_j(\widetilde{S}_j|\mathbf{S}_{-j},A=a)}\Big].
\end{eqnarray*}\vspace{-3em}

\noindent Here $\mathbf{S}_{-j}$ ($\widetilde{\mathbf{S}}_{-j}$) denotes the subvector of $\mathbf{S}$ ($\widetilde{\mathbf{S}}$) obtained by removing the $j$th component, and $p_j$ and $q_j$ denote the oracle pdf/pmf of $S_j$ given $\mathbf{S}_{-j}$ and $A=a$ and its estimate for constructing knockoffs, respectively. In addition, $\widehat{G}_k$ obtained by implementing the knockoffs+ in Algorithm \ref{alg:2} satisfies \vspace{-1em}
\begin{align*}
	\textrm{FDR}(\widehat{G}_k)\leq q \exp(\epsilon)+\mathbb{P}\Big(\max_{j\in \mathcal{H}_0}\widehat{\textrm{KL}}_{j}>\epsilon\Big)+O\Big\{ K^{-1} (NT)^{1-k_0 \log(\rho^{-1})} \Big\}.
\end{align*}
\end{theorem}\vspace{-1.5em}

\noindent The upper error bounds presented in Theorem \ref{thm1} can be decomposed into three 
terms: \vspace{-0.5em}
\begin{itemize}[leftmargin=*]
    \item The first term approaches the target FDR level $q$ as $\epsilon$ decays to zero. The inclusion of the additional term $\exp(\epsilon)$ arises because we estimate the covariate distribution  to construct the knockoff variables rather than assume a known covariate distribution. \vspace{-0.5em}
    \item The second term provides a high probability error bound for \(\widehat{\textrm{KL}}_{j}\), which quantifies the discrepancy between the true conditional distribution of $S_j$ given $A$ and $\mathbf{S}_{-j}$ and its 
    estimate 
    based on the external unlabeled dataset. Assuming these conditional distributions are multivariate Gaussian, this term becomes negligible by setting $\epsilon=\sqrt{L|\mathcal{D}_k|(\log^2 d)/|\mathcal{U}|}$, where $L$ denotes the maximum number of nonzero entries in any column of the Gaussian precision matrix across all columns \citep{barber2020robust}. Given that the number of observations in the unlabeled data \(|\mathcal{U}|\) is typically much larger than that in the labeled data \(|\mathcal{D}|\), \(\epsilon\) becomes negligible, making the first term approach the target FDR level. \vspace{-0.5em}
    \item The last term accounts for the temporal dependence among transition tuples, which decays to zero as the number of observations in $\calD_k$ grows to infinity.
\end{itemize}  

Let $d_0$ denote the dimension of the unique minimum sufficient state. Next, we upper bound the family-wise type-I error of the SEEK algorithm. 
\begin{theorem}[Type-I error]\label{thm2}
	Suppose the conditions in Theorem~\ref{thm1} hold. Then, as the FDR target $q$ approaches zero, the probability that Algorithm~\ref{alg:1} with knockoffs+ selects any null variable is upper bounded by $O(\alpha^{-1} d_0^2 q) + O\Big\{\alpha^{-1}K^{-1}(NT)^{1-k_0 \log(\rho^{-1})}\Big\}$.
\end{theorem}
According to Theorem \ref{thm2}, SEEK asymptotically avoids selecting any null variables provided that $q\ll d_0^{-2}$ or $q\to 0$ for a fixed $d_0$. Compared to the upper bounds in Theorem~\ref{thm1}, the one in Theorem~\ref{thm2} incorporates an additional factor of $O(\alpha^{-1} d_0^2)$ for $q$. This factor emerges for two reasons: (i) Theorem \ref{thm2} bounds the probability of selecting any null state rather than focusing on FDR, introducing a factor of $O(d_0)$; (ii) Theorem \ref{thm2} applies to the iterative Algorithm \ref{alg:1}, as opposed to Algorithm \ref{alg:2} which deals with a single iteration. Since SEEK terminates after at most $1+d_0$ steps when no null variable presents, the uniform control across $1+d_0$ iterations contributes to another factor of $O(\alpha^{-1} d_0)$, where $\alpha^{-1}$ arises from the majority voting step in Algorithm~\ref{alg:1}.

\subsection{Power Analysis}\label{sec:poweranalysis}

In this section, we study the power of SEEK. 
The analysis depends on the machine learning algorithm for constructing the feature importance statistics (see Step 2 in Section \ref{sec:seek}). We focus on the power analysis based on the LASSO. In Section \ref{sec:genericML} of the Supplementary Materials, we provide sufficient conditions to ensure that the power approaches one when a generic machine learning algorithm is used. 

To analyze LASSO, we consider a linear system where the reward and next state satisfy the following model:
\begin{equation*}
    [\mathbf{S}_{t+1}^\top,~R_t]
    =\sum_{a\in\mathcal{A}} \mathbb{I}(A_t=a) \mathbf{S}_t^\top \bm{B}^{(a)} 
 + \bm{\varepsilon}_t,
\end{equation*}
for some $d \times (d+1)$ matrices $\{\bm{B}^{(a)} \}_{a\in\mathcal{A}}$ 
and zero-mean error vectors $\{\bm{\varepsilon}_t\}_t$. 

The indices of the minimal sufficient state $G_{\textrm{M}}$ can be identified as follows. We begin with $G_1=\{j\in [d]: \max\limits_{a\in \aspace} |B_{j,d+1}^{(a)}| \neq 0\}$. Next, for any $\ell\ge 2$, we iteratively define $G_{\ell}=\{j\in [d]: \max\limits_{a\in \aspace, i\in G_{\ell-1}} |B_{j,i}^{(a)}| \neq 0\}\cup G_{\ell-1}$ and set $G_{\textrm{M}}=G_{\ell}$ whenever $G_{\ell}=G_{\ell-1}$. These coefficient matrices $\{\bm{B}^{(a)} \}_{a\in\mathcal{A}}$ define a directed graph. In particular, consider an augmented  $(d+1)\times (d+1)$ coefficient matrix $\bm{B}$ such that $B_{j,i}=\max\limits_{a\in \aspace} |B_{j,i}^{(a)}|$ for any $1\le j\le d,1\le i\le d+1$ and $B_{d+1,i}=0$ for any $1\le i\le d+1$, and a directed graph $\mathcal{G}(\bm{B})$ whose weight adjacency matrix is given by $\bm{B}$. By definition, there exists an edge from the $j$th node to the $i$th node if and only if $B_{j, i}=\max\limits_{a \in \mathcal{A}} |B_{j,i}^{(a)}|\neq 0$, and it is straightforward to see that $j\in G_{\textrm{M}}$ if and only if there exists a directed path from the $j$th node (i.e., the $j$th state variable) to the $(d+1)$th node (i.e., the reward).

To simplify notation, we use $n=NT/K$ to denote the sample size of each data subset $\calD_k$. For $1\le i\le d+1$, let $d_i$ denote the number of nonzero elements in the $i$th column of $\bm{B}$ and let $d_*=\max\limits_{1\le i\le d+1} d_i$. We impose two minimal signal strength conditions, one designed for Algorithm \ref{alg:2}, which applies knockoffs through a single iteration, and the other tailored for the sequential procedure in Algorithm \ref{alg:1}. Additional regularity conditions, such as knockoff construction errors, sparsity constraints, tuning parameters, and the absence of ties, are frequently imposed in the variable selection literature \citep[see, e.g.,][]{fan2019rank}. To save space, they are relegated to Section \ref{thm:powercond} of the Supplementary Materials.

\begin{condition}[Minimal signal strength for one iteration]\label{ass:signal} 
(i) There exists a sequence $\kappa_n\to \infty$ as $n\to \infty$ such that $\min_{i,j: B_{j,i}\neq 0} |B_{j,i}|\ge \kappa_n \sqrt{n^{-1}\log (dn)}$. 
 %
(ii) For each $1\le i\le d+1$, the number of nonzero elements in the $i$th column of $\bm{B}$ with absolute values larger than $\sqrt{d_* n^{-1}\log (dn)}$ shall exceed $2q^{-1}$. 
\end{condition}

\begin{condition}[Minimal signal strength for the sequential procedure]\label{ass:signal2} 
    The subgraph of $\mathcal{G}(\bm{B})$, obtained by removing edges with weights smaller than $\sqrt{d_* n^{-1}\log (d n)}$, still identifies $ G_{M}$, i.e.,  there exists a directed path in this subgraph from each variable in $G_M$ to the reward.
\end{condition}

\begin{figure}[!t]
\centering
\begin{subfigure}{0.5\linewidth}
\centering
\includegraphics[width=0.35\linewidth]{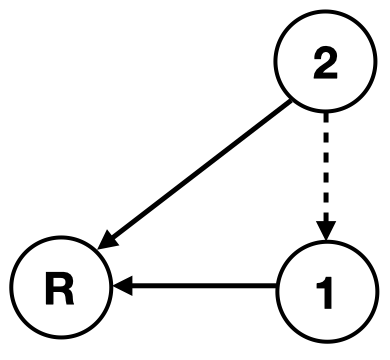}
\caption{2 state variables}
\end{subfigure}%
\begin{subfigure}{0.5\linewidth}
\centering
\includegraphics[width=0.35\linewidth]{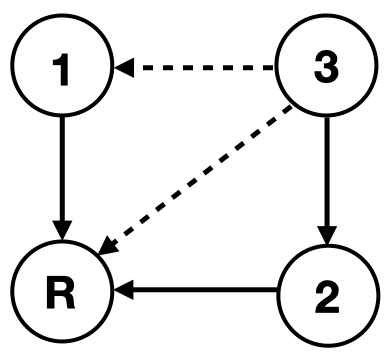}
\caption{3 state variables}
	\end{subfigure}
\caption{Two examples of dependence graph among reward and true state variables.}
\label{fig: graph_eg}
\end{figure}

\vspace{-0.5em}
Both conditions are weaker than the conventional minimal signal strength condition \citep[see e.g.,][Equation (8)]{zhao2006model}, which requires the absolute values of all nonzero coefficients to be much larger than $\sqrt{d_* n^{-1}\log (dn)}$ for LASSO's sure screening property \citep{fan2008sure}. To elaborate, we classify all nonzero coefficients into three categories:\vspace{-0.8em}
\begin{enumerate}[leftmargin=*]
    \item \textbf{Strong signals}: Coefficients whose absolute values are much larger than $\sqrt{d_*n^{-1}\log (dn)}$;\vspace{-0.3cm}
    \item \textbf{Moderate signals}: Coefficients whose absolute values are much larger than $\sqrt{n^{-1}\log (dn)}$, but are of the order $O(\sqrt{d_*n^{-1}\log (dn)})$; \vspace{-0.3cm}
    \item \textbf{Weak signals}: Coefficients that are of the order $O(\sqrt{n^{-1}\log (dn)})$. \vspace{-0.5em}
\end{enumerate}
Based on this classification, the conventional minimal signal strength condition requires all nonzero coefficients to be strong signals. Condition \ref{ass:signal} relaxes this requirement to include moderate signals since the sequence $\kappa_n$ in (i) can be any slowly diverging sequence. However, (ii) requires a proportion of nonzero coefficients that must still be strong signals. Similar conditions have been imposed to analyze the power of knockoffs in regression settings \citep{fan2019rank, fan2020ipad}.

Condition \ref{ass:signal2} further relaxes the requirements above to include weak signals. To illustrate, consider a simple example where the first two variables are the minimal sufficient state, both $S_{t,1}$ and $S_{t,2}$ directly influence $R_t$, and $S_{t,2}$ affects $S_{t,1}$ as well. In that case, Condition~\ref{ass:signal2} requires (i) $|B_{1,d+1}| \gg \sqrt{d_* n^{-1}\log (dn)}$ and (ii) either $|B_{2,d+1}| \gg \sqrt{d_* n^{-1}\log (dn)}$ or $ |B_{2,1}|\gg \sqrt{d_* n^{-1}\log (dn)}$. Notice that (ii) allows one of the two coefficients to be a weak signal. It occurs because there exist two directed paths from $S_{t, 2}$ to $R_t$, given by $S_{t,2}\to R_t$ and $S_{t,2}\to S_{t,1}\to R_t$, and we only require one of them to appear in the subgraph after removing edges with weak and moderate signals. Thus, we refer to this phenomenon as ``the blessings of multiple paths'', which allows us to impose a weaker signal strength condition. Two graphical examples are given in Figure \ref{fig: graph_eg}, where solid lines denote edges with strong signals, while dashed lines denote edges with moderate or weak signals. 

The following two theorems establish the true positive rate (TPR) of Algorithm \ref{alg:2} and the type-II error of Algorithm \ref{alg:1}, respectively. 
\begin{theorem}[TPR]\label{thm:power-lasso}
    For any $1\le i\le d+1$, let $G(\{i\})$ denote the support of the $i$th column of $\bm{B}$. Under Conditions \ref{ass:mixing}, \ref{ass:external-dataset}, \ref{ass:signal}, and \ref{ass:knockoffs} to \ref{ass:noties} in the Supplementary Materials, $\widehat{G}_k$ returned by Algorithm \ref{alg:2} -- with $Y$ set to $S_i'$ for $i\leq d$ and to $R$ otherwise -- satisfies \vspace{-1em}
    \begin{align}\label{eqn:TPR}
        \mathbb{E}\left(\frac{|\widehat{G}_k\cap G(\{i\})|}{|G(\{i\})|}\right) \geq 1 -O(\kappa_n^{-1}) - O\left\{K^{-1}(NT)^{1-k_0 \log(\rho^{-1})}\right\}-O\left\{(nd)^{-C}\right\},
    \end{align}\vspace{-3em}
    
    \noindent for any large constant $C>0$. 
\end{theorem}
\begin{theorem}[Type-II error]\label{thm:type-II}
    Under Conditions \ref{ass:mixing}, \ref{ass:external-dataset}, \ref{ass:signal2}, and \ref{ass:knockoffs} to \ref{ass:tuning}, the probability that the selected set $\widehat{G}$ returned by Algorithm \ref{alg:1} contains the minimal sufficient state is lower bounded by $1-O\left\{(1-\alpha)^{-1}(nd)^{-C}\right\}-O\left\{(1-\alpha)^{-1}K^{-1}(NT)^{1-k_0 \log(\rho^{-1})}\right\}$. 
\end{theorem}
The last term in \eqref{eqn:TPR} is negligible since the constant $C$ can be chosen arbitrarily large. The orders of magnitude of the second and third terms on the RHS are determined by the signal strength and the mixing properties of the MDP, both of which decay to zero under Conditions \ref{ass:mixing} and \ref{ass:signal}. Consequently, Theorem \ref{thm:power-lasso} proves that SEEK's TPR asymptotically approaches $1$. Meanwhile, Theorem \ref{thm:type-II} establishes SEEK's sure screening property. Together with Theorem \ref{thm2}, it shows that SEEK consistently identifies the minimal sufficient state.

\vspace{-0.4cm}
\section{Simulation Experiments}\label{sec:exp}
\vspace*{-1em}
\subsection{Experiment Design, Benchmarks, and Evaluation Metrics}\label{sec:expermentdesign}
\vspace*{-0.2cm}
In this section, we investigate the variable selection performance of SEEK via simulation studies and demonstrate its usefulness in improving existing RL algorithms. For the latter purpose, we apply FQI or implicit Q-learning \citep[IQL,][]{kostrikov2022offline} to the entire dataset and its subsets defined by the selected state variables.

We consider four environments. The first (denoted by AR) is based on an autoregressive model,  where all state variables follow an AR(1) model detailed in Section \ref{apdx:a} of the Supplemental Materials. The second environment (denoted by Mixed) is described in Section \ref{sec:minimal}, which includes autoregressive and i.i.d. state variables. In both environments, the dimension of the minimal sufficient state $d_0$ is 2. The third and fourth environments are from the OpenAI Gym\footnote{\url{https://github.com/openai/gym}}: CartPole-v0 (CP)  and LunarLander-v2 (LL), where $d_0=4$ for CP and $d_0=8$ for LL. For these environments, we manually include null state variables such that $d$ equals the specified value. Each null state variable either follows a white noise process or an AR(1) process. For AR and Mixed, we fix the horizon $T$ at 150. In CP and LL, the horizons are approximately 130 and 340 decision points, respectively. 

We investigate the performance of SEEK under large-scale settings, in which $N=d = 600, 800,\textup{ or } 1000$\footnote{We also evaluate SEEK on small-scale datasets; see Section~\ref{sec:addsimuresults} in Supplementary Materials.}. Each batch dataset is generated using an $\epsilon$-greedy algorithm as the behavior policy, which is a mixture of the optimal policy (with a probability of 0.7) and the random policy (with a probability of 0.3). The optimal policies have closed-form expressions for the first two environments. For the last two environments, the optimal policy is approximated with a deep Q-network \citep{mnih2015human} agent for the CP environment and the LL environment. 


To implement SEEK, we fix the majority vote threshold $\alpha=0.5$ and the target FDR level $q=0.1$. In the first two environments, the exponential $\beta$-mixing condition (see Condition \ref{ass:mixing}) is satisfied, and we apply the method proposed in Section \ref{sec:bestK} of the Supplementary Materials to determine the number of splits $K$ adaptively.  In addition, as the system is linear for these environments, we apply LASSO in Step 3 of Algorithm \ref{alg:2} for the implementation of SEEK. However, the last two environments are neither ergodic nor linear. As such, we apply random forests (RFs) when implementing Algorithm \ref{alg:2}. We fix $K =\log(NT)$ across different simulation replications. As shown in a sensitivity analysis in Section \ref{sec:addsimuresults} of the Supplemental Materials, SEEK has robust performance against the choice of $K$, as long as it is chosen in the order of $O(\log(NT))$.

Four benchmark methods are included. The first two are the Reward-only and One-step methods described in Section \ref{sec:minimal}, for which the same $K, \alpha$, knockoff filter and $q$ are used for variable selection as SEEK. The pseudocode of the two algorithms are provided in Section~\ref{sec:algo-rewardonly-onestep} of the Supplemental Materials. The third benchmark is also an implementation of the one-step approach but with a different variable selection method. More specifically, it conducts variable selection by LASSO (denoted by VS-LASSO) if the system is linear and by RF (denoted by VS-RF) otherwise. The last benchmark is the sparse feature selection method \citep[SFS,][]{hao2021sparse} based on LASSO. 

The methods above are evaluated by the following criteria: (i) FDR/mFDR, false positive rate (FPR, the percentage of falsely selected variables), and (ii) true positive rate (TPR, the percentage of correctly selected variables). We also investigate how the quality of the learned policies is affected by variable selection. For this purpose, we compute the value difference (VD) by taking the difference between the cumulative rewards of the policies learned by selected states and all states. For fair comparison, all these policies are learned via FQI in the first three environments. In the last environment, however, we found that FQI had a poor performance. Thus, we turn to employ IQL, a state-of-the-art offline RL algorithm, in this environment. Their cumulative rewards are calculated using a Monte Carlo method. Information about tuning parameters of LASSO and RF is given in Table \ref{tab:hyper-knockoff} in the Supplemental Materials. Details of the neural network architectures in the FQI and IQL are reported in Table \ref{tab:hyper-policy-learning} in the Supplemental Materials.


\vspace{-0.3cm}
\subsection{Results}
\vspace{-0.2cm}
\begin{table}[!t]
	\linespread{1.25}\selectfont
	\setlength{\tabcolsep}{4.8pt}
	\caption{Performances of SEEK and benchmark methods in the AR environment. The cumulative rewards of policies learned by the all states are 1.25, 1.30, and 1.32, when $N$ is $600, 800$, and $1000$, respectively.}
	\label{tab:large-scale-ar}
	\centering
	\scriptsize
	\vspace{-0.1cm}
	\begin{tabular}{l|ccc|ccc|ccc|ccc|ccc}
	\toprule
	& \multicolumn{3}{c}{SEEK} & \multicolumn{3}{c}{Reward-only} & \multicolumn{3}{c}{One-step} & \multicolumn{3}{c}{SFS} & \multicolumn{3}{c}{VS-LASSO} \\ \hline
	$N=d$ & 600 & 800 & 1000 & 600 & 800 & 1000 & 600 & 800 & 1000 & 600 & 800 & 1000 & 600 & 800 & 1000 \\ \hline
	$K$                  & 52.60 & 49.00 & 45.30 & 52.60 & 49.00 & 45.20 & 52.30 & 49.10 & 45.20 & /    & /    & /    & 52.60 & 49.00 & 49.00 \\ \hline
	mFDR                 & 0.00  & 0.00  & 0.00  & 0.00  & 0.00 & 0.00   & 0.98  & 0.99  & 0.99  & 0.00 & 0.00 & 0.00 & 0.98  & 0.99  & 0.99 \\ \hline
	FDR                  & 0.00  & 0.00  & 0.00  & 0.00  & 0.00 & 0.00   & 1.00  & 1.00  & 1.00  & 0.00 & 0.00 & 0.00 & 1.00  & 1.00  & 1.00 \\ \hline
	FPR                  & 0.00  & 0.00  & 0.00  & 0.00  & 0.00 & 0.00   & 1.00  & 1.00  & 1.00  & 0.00 & 0.00 & 0.00 & 1.00  & 1.00  & 1.00 \\ \hline
	TPR                  & 1.00  & 1.00  & 1.00  & 1.00  & 1.00 & 1.00   & 1.00  & 1.00  & 1.00  & 1.00 & 1.00 & 1.00 & 1.00  & 1.00  & 1.00 \\ \hline
	VD  & 0.76 & 0.80 & 0.77 & 0.76 & 0.80 & 0.77 & -0.02 & 0.03  & 0.02  & 0.74 & 0.77 & 0.80 & 0.07 & 0.00  & -0.08 \\ 
	\bottomrule
	\end{tabular}
\end{table}
\begin{table}[!t]
	\linespread{1.25}\selectfont
	\caption{Results on the Mixed environment. The cumulative rewards of policies learned by the original state are 2.07, 2.12, and 2.15, when $N$ is $600, 800$, and $1000$, respectively.}
	\label{tab:large-scale-mixed}
	\centering
	\scriptsize
	\vspace{-0.1cm}
	\begin{tabular}{l|ccc|ccc|ccc|ccc|ccc}
	\toprule
	& \multicolumn{3}{c}{SEEK} & \multicolumn{3}{c}{Reward-only} & \multicolumn{3}{c}{One-step} & \multicolumn{3}{c}{SFS} & \multicolumn{3}{c}{VS-LASSO} \\ \hline
	$N=d$ & 600  & 800  & 1000 & 600  & 800  & 1000 & 600   & 800  & 1000 & 600 & 800 & 1000 & 600 & 800 & 1000 \\ \hline
	$K$   & 3.40 & 3.40 & 3.60 & 3.60 & 3.60 & 3.40 & 3.40  & 3.48 & 3.60 & / & / & / & 3.22 & 4.16 & 4.56 \\ \hline
	mFDR  & 0.00 & 0.00 & 0.00 & 0.00 & 0.00 & 0.00 & 0.98  & 0.99 & 0.99 & 0.00 & 0.00 & 0.00 & 0.98 & 0.99 & 0.99 \\ \hline
	FDR   & 0.00 & 0.00 & 0.00 & 0.00 & 0.00 & 0.00 & 1.00  & 1.00 & 1.00 & 0.00 & 0.00 & 0.00 & 1.00 & 1.00 & 1.00 \\ \hline
	FPR   & 0.00 & 0.00 & 0.00 & 0.00 & 0.00 & 0.00 & 1.00  & 1.00 & 1.00 & 0.00 & 0.00 & 0.00 & 1.00 & 1.00 & 1.00 \\ \hline
	TPR   & 1.00 & 1.00 & 1.00 & 0.50 & 0.50 & 0.50 & 1.00  & 1.00 & 1.00 & 0.50 & 0.50 & 0.50 & 1.00 & 1.00 & 1.00 \\ \hline
	VD   & 0.02 & 0.02 & 0.02 & 0.02 & 0.02 & 0.02 & -0.03 & 0.00 & 0.02 & 0.02 & 0.01 & 0.01 & -0.02 & -0.01 & 0.00 \\ 
	\bottomrule
	\end{tabular}
\end{table}
\begin{table}[!t]
\linespread{1.25}\selectfont
\setlength{\tabcolsep}{4.5pt}
\caption{Results on the CP environment. The cumulative rewards of the policies based on the original state are summarized in Table~\ref{tab:cartpole-drl-value-large} of the Supplementary Materials.}
\label{tab:large-scale-cartpole}
\scriptsize
\centering
\vspace{-0.1cm}
\begin{tabular}{l|l|ccc|ccc|ccc|ccc|ccc}
\toprule
\multirow{2}{*}{Type}
& Method & \multicolumn{3}{c}{SEEK} & \multicolumn{3}{c}{Reward-only} & \multicolumn{3}{c}{One-step} & \multicolumn{3}{c}{SFS} & \multicolumn{3}{c}{VS-RF} \\ 
\cline{2-17}
& $N=d$ & 600 & 800 & 1000 & 600 & 800 & 1000 & 600 & 800 & 1000 & 600 & 800 & 1000 & 600 & 800 & 1000 \\ 
\hline
\multirow{5}{*}{IID} 
& mFDR  & 0.00   & 0.00   & 0.00   & 0.00   & 0.00  & 0.00    & 0.00 & 0.00 & 0.00 & 0.00 & 0.00 & 0.00 & 0.00 & 0.00 & 0.00 \\ \cline{2-17}
& FDR   & 0.00   & 0.00   & 0.00   & 0.00   & 0.00  & 0.00    & 0.00 & 0.00 & 0.00 & 0.00 & 0.00 & 0.00 & 0.00 & 0.00 & 0.00 \\ \cline{2-17}
& FPR   & 0.00   & 0.00   & 0.00   & 0.00   & 0.00  & 0.00    & 0.00 & 0.00 & 0.00 & 0.00 & 0.00 & 0.00 & 0.00 & 0.00 & 0.00 \\ \cline{2-17}
& TPR   & 1.00   & 1.00   & 1.00   & 0.75   & 0.75  & 0.75    & 1.00 & 1.00 & 1.00 & 0.75 & 0.75 & 0.75 & 0.00 & 0.00 & 0.00 \\ \cline{2-17}
& VD    & 859 & 928 & 767 & 825 & 710 & 487 & 856 & 929 & 764 & 804 & 603 & 531 & -13 & -14 & -15 \\ \hline
\multirow{5}{*}{AR(1)}
& mFDR  & 0.00   & 0.00   & 0.00   & 0.00   & 0.00   & 0.00    & 0.98 & 0.99  & 0.99 & 0.00   & 0.00 & 0.00 & 0.00 & 0.00 & 0.00 \\ \cline{2-17}
& FDR   & 0.00   & 0.00   & 0.00   & 0.00   & 0.00   & 0.00    & 1.00 & 1.00  & 1.00 & 0.00   & 0.00 & 0.00 & 0.00 & 0.00 & 0.00 \\ \cline{2-17}
& FPR   & 0.00   & 0.00   & 0.00   & 0.00   & 0.00   & 0.00    & 1.00 & 1.00  & 1.00 & 0.00   & 0.00 & 0.00 & 0.00 & 0.00 & 0.00 \\ \cline{2-17}
& TPR   & 1.00   & 1.00   & 1.00   & 0.75   & 0.75   & 0.75    & 1.00 & 1.00  & 1.00 & 0.75   & 0.75 & 0.75 & 0.00 & 0.00 & 0.00 \\ \cline{2-17}
& VD   & 115 & 130 & 96 & 101 & 101 & 75 & 1 & -1 & 2 & 107 & 86 & 81 & -12 & -13 & -12 \\ 
\bottomrule
\end{tabular}
\end{table}
\begin{table}[!t]
\linespread{1.25}\selectfont
\setlength{\tabcolsep}{4.5pt}
	\caption{Results on the LL environment. The cumulative rewards of the policies based on the original state are summarized in Table~\ref{tab:cartpole-drl-value-large} of the Supplementary Materials.}
	\label{tab:large-scale-lunarlander}
	\scriptsize
	\centering
	\vspace{-0.1cm}
	\begin{tabular}{l|l|ccc|ccc|ccc|ccc|ccc}
	\toprule
	\multirow{2}{*}{Type}
	& Method & \multicolumn{3}{c}{SEEK} & \multicolumn{3}{c}{Reward-only} & \multicolumn{3}{c}{One-step} & \multicolumn{3}{c}{SFS} & \multicolumn{3}{c}{VS-RF} \\ 
	\cline{2-17}
	& $N=d$ & 600 & 800 & 1000 & 600 & 800 & 1000 & 600 & 800 & 1000 & 600 & 800 & 1000 & 600 & 800 & 1000 \\ 
	\hline
	\multirow{5}{*}{IID} 
	& mFDR & 0.00  & 0.00  & 0.00  & 0.00   & 0.00  & 0.00    & 0.00 & 0.00  & 0.00 & 0.00 & 0.00 & 0.00 & 0.00 & 0.00 & 0.00 \\ \cline{2-17}
	& FDR  & 0.00  & 0.00  & 0.00  & 0.00   & 0.00  & 0.00    & 0.00 & 0.00  & 0.00 & 0.00 & 0.00 & 0.00 & 0.00 & 0.00 & 0.00 \\ \cline{2-17}
	& FPR  & 0.00  & 0.00  & 0.00  & 0.00   & 0.00  & 0.00    & 0.00 & 0.00  & 0.00 & 0.00 & 0.00 & 0.00 & 0.00 & 0.00 & 0.00 \\ \cline{2-17}
	& TPR  & 1.00  & 1.00  & 1.00  & 0.12   & 0.12  & 0.12    & 1.00 & 1.00  & 1.00 & 0.38 & 0.38 & 0.38 & 0.00 & 0.00 & 0.00 \\ \cline{2-17}
	& VD   & 57 & 32 & 43 & -686 & -682 & -669 & 57 & 32 & 43 & -686 & -682 & -5 & -723 & -718 & -785 \\ \hline
	\multirow{5}{*}{AR(1)}
	& mFDR  & 0.00   & 0.00   & 0.00   & 0.00   & 0.00   & 0.00    & 0.98 & 0.99  & 0.99 & 0.00   & 0.00 & 0.00 & 0.00 & 0.00 & 0.00 \\ \cline{2-17}
	& FDR   & 0.00   & 0.00   & 0.00   & 0.00   & 0.00   & 0.00    & 1.00 & 1.00  & 1.00 & 0.00   & 0.00 & 0.00 & 0.00 & 0.00 & 0.00 \\ \cline{2-17}
	& FPR   & 0.00   & 0.00   & 0.00   & 0.00   & 0.00   & 0.00    & 1.00 & 1.00  & 1.00 & 0.00   & 0.00 & 0.00 & 0.00 & 0.00 & 0.00 \\ \cline{2-17}
	& TPR   & 1.00   & 1.00   & 1.00   & 0.12   & 0.12   & 0.12    & 1.00 & 1.00  & 1.00 & 0.38   & 0.38 & 0.38 & 0.00 & 0.00 & 0.00 \\ \cline{2-17}
	& VD   & 435 & 405 & 430 & -307 & -308 & -282 & -2 & 1 & -1 & -307 & -308 & -7 & -344 & -345 & -318 \\ 
	\bottomrule
	\end{tabular}
\end{table}
The results are reported in Tables~\ref{tab:large-scale-ar}--\ref{tab:large-scale-lunarlander}. We summarize our findings as follows. First, the proposed SEEK algorithm outperforms the other methods, with its FDRs close to zero and TPRs close to one in almost all cases. In addition, it can be seen that the resulting policies based on the selected state variables achieve the highest cumulative rewards, and are consistently better than the estimated optimal policies based on all state variables across all settings. Particularly, in the CP environment, the policies learned by all states have cumulative rewards range from 10 to 15 (see Table~\ref{tab:cartpole-drl-value-large} in the Supplementary Materials), which are eight to ten times lower than those of the trained policies based on SEEK. This demonstrates the importance of variable selection in RL. Second, the Reward-only method omits significant state variables as they do not directly impact the rewards. This failure substantially worsens the learned policy in the last two environments. Third, regardless of the implementation (either with the knockoff filter or vanilla machine learning methods), the one-step method fails with high FDR and poorly-estimated policies in the first, third and last environments. Fourth, SFS has low power in the Mixed and LL environments. As noted previously, this is because SFS requires linear function approximation and cannot handle complex nonlinear systems. These results align with our theoretical findings, demonstrating the advantage of the proposed SEEK algorithm. 

\vspace*{-0.8em}
\section{Analysis of the MIMIC-III Dataset}\label{sec:mimic3}
\vspace*{-0.4em}
The MIMIC-III v1.4 dataset contains critical care data for over 40,000 patients from Beth Israel Deaconess Medical Center, covering the period from 2001 to 2012 \citep{johnson2016mimic}. As commented in the introduction, each patient's state contains 47 variables, including their baseline information, laboratory values, vital signs, and intake/output records. Following 
the analyses in \citet{raghu2017continuous} and \citet{ zhou2023optimizing}, we discretize the dosage level of 
intravenous (IV) fluid and maximum vasopressor (VP) into five each. This yields 25 actions, each representing a combination of the two treatments' dosage levels. 
The reward is set to be the inverse of the SOFA score, which assesses the severity of patient organ failure \citep{lambden2019sofa}. 
Similar to the simulation study, we compare SEEK against One-step, Reward-only, SFS, and VS-RF, as shown in Table \ref{tab:mimic3}: 
\begin{itemize}[leftmargin=*]
    \item \textbf{SEEK} substantially reduces the state dimension from 47 to 3. Among the three selected variables, two of them are related to urine output: urine output in the last 4 hours and total urine output. Both are identified as crucial by existing literature and clinician recommendations \citep{komorowski2018artificial}.
    \item \textbf{Reward-only} successfully selects urine output in the last 4 hours but fails to include total urine output. This is expected, as the method is designed to identify variables that affect immediate rewards rather than long-term outcomes.
    \item \textbf{VS-RF} and \textbf{One-step} fail to eliminate irrelevant variables, selecting nearly all variables. This suggests that in this dataset, many variables, while not affecting immediate rewards or important states, influence themselves temporally over time. 
    \item \textbf{SFS} also selects 3 variables, but they are entirely different from those chosen by SEEK. 
\end{itemize}
 
To illustrate the usefulness of variable selection, we apply IQL to the data subset based on the selected state variables to learn the optimal policy; see Section \ref{sec:details-IQL} of the Supplementary Materials for its implementation details. We next evaluate these estimated policies using fitted Q-evaluation \citep{le2019batch} and report their estimated cumulative reward in Table~\ref{tab:mimic3}. We further implement two deep RL algorithms on the original dataset without variable selection, one directly applying IQL and the other (denoted by AE) 
utilizing an auto-encoder to extract low-dimensional state representations before applying IQL. The results show that the estimated optimal policy based on SEEK achieves the highest cumulative reward. In conclusion, our analysis reveals the following findings: (i) The comparison against IQL on the whole dataset reveals the usefulness of variable selection in boosting the performance of existing deep RL algorithms; (ii) The comparison against other competing variable selection methods emphasizes the importance of using the most relevant variables for policy learning; (iii) The comparison against AE demonstrates the superiority of SEEK over unsupervised neural-network-based state representation learning methods.


\bibliographystyle{agsm}
\spacingset{1.50}
\bibliography{references}

\newpage
\newcommand\independent{\protect\mathpalette{\protect\independenT}{\perp}}
\def\independenT#1#2{\mathrel{\rlap{$#1#2$}\mkern2mu{#1#2}}}

\newcommand{\fb}{{\bf b}}
\newcommand{\fs}{{\bf s}}
\newcommand{\fR}{{\bf R}}
\newcommand{\Mean}{{\mathbb{E}}}
\newcommand{\Cov}{{\mbox{cov}}}
\newcommand{\Corr}{{\mbox{corr}}}
\newcommand{\diag}{{\mbox{diag}}}
\newcommand{\prob}{{\mbox{Pr}}}

\def\cvgn{\mathop{\longrightarrow}_{n\to+\infty}}
\def\ds1{{\mathrm{1 \hspace{-2.6pt} I}}}
\def\dsB{\mathbb {B}}
\def\dsC{\mathbb {C}}
\def\dsE{\mathbb {E}}
\def\dsN{\mathbb {N}}
\def\dsP{\mathbb {P}}
\def\dsQ{\mathbb {Q}}
\def\dsR{\mathbb {R}}
\def\dsS{\mathbb {S}}
\def\dsV{\mathbb {V}}
\def\dsZ{\mathbb {Z}}
\def\calA{{\cal A}}
\def\calB{{\cal B}}
\def\calBX{{{\cal B}_\calX}}
\def\calBY{{{\cal B}_\calY}}
\def\calBR{{\cal B}}
\def\calBZ{{{\cal B}_\calZ}}
\def\calC{{\cal C}}
\def\tcalC{{\tilde{\cal C}}}
\def\calCm{{{\cal C}_{max}}}
\def\calD{{\cal D}}
\def\calE{{\cal E}}
\def\calF{{\cal F}}
\def\calFXY{{\cal F}}
\def\calG{{\cal G}}
\def\calH{{\cal H}}
\def\calI{{\cal I}}
\def\calK{{\cal K}}
\def\calL{{\cal L}}
\def\calM{{\cal M}}
\def\calN{{\cal N}}
\def\calP{{\cal P}}
\def\calQ{{\cal Q}}
\def\calR{{\cal R}}
\def\tcalR{{\tilde{\cal R}}}
\def\calS{{\cal S}}
\def\calT{{\cal T}}
\def\calU{{\cal U}}
\def\calV{{\cal V}}
\def\calW{{\cal W}}
\def\calX{{\cal X}}
\def\calY{{\cal Y}}
\def\calZ{{\cal Z}}

\def\Dpi{\calD^\pi}
\def\Dpin{ \hat{\calD}^\pi_n}
\def\barotimes{\,\bar{\otimes}\,}
\def\Var{\text{Var}}
\def\EE{\mathbb{E}}
\def\w{\boldsymbol{w}}
\def\e{\boldsymbol{e}}
\def\f{\boldsymbol{f}}
\def\vpsi{\boldsymbol{\psi}}
\def\XX{\textbf{X}}
\def\ZZ{\textbf{Z}}
\def\rwrd{\textbf{R}}
\def\vphi{\boldsymbol{\phi}}
\def\R{\mathbb{R}}
\def\Regret{\operatorname{Regret}}
\def\Rn{\mathbf{R}_n}
\def\Rem{\operatorname{Rem}}
\def\Yn{\mathbf{Y}_n}
\def\I{\,\mathcal{I}}
\def\ones{\mathbf{1}}
\def\zeros{\mathbf{0}}
\def\B{\mathcal{B}}
\def\argmin{\operatorname{argmin}}
\def\argminb{\operatorname*{argmin}}
\def\argmax{\operatorname{argmax}}
\def\N{\mathcal{N}}
\def\argmaxb{\operatorname*{argmax}}
\def\wcvg{\Rightarrow}
\renewcommand{\liminf}{\varliminf}
\renewcommand{\limsup}{\varlimsup}
\def\op{\operatorname{op}}
\def\X{\mathcal{X}}
\def\Y{\mathcal{Y}}
\def\jiao{\cap}
\def\bing{\cup}
\def\C{\mathcal{C}}
\def\linfty{l^{\infty}}
\def\Q{\mathcal{Q}}
\def\given{\, | \,}
\def\Given{\, \Big| \,}
\def\GG{\mathbb{G}}
\def\Gn{\mathbb{G}_n}
\def\V{\mathcal{V}}
\def\transpose{\top}
\def\H{\mathcal{H}}
\def\Morth{\mathcal{M}^{\perp}}
\def\opt{\text{opt}}
\def\Span{\text{Span}}
\def\Vn{\hat{\mathcal{V}}_n}
\def\Mn{\hat{M}_N}

\begin{center}
    {\LARGE{\bf Supplementary Materials to ``Sequential Knockoffs for Variable Selection in Reinforcement Learning''}}
\end{center}
\bigskip
\bigskip
\bigskip

The Supplementary Materials are organized as follows. In Section~\ref{apdx:a}, we provide detailed descriptions of our numerical experiments. Section~\ref{apdx:b} outlines the implementation details of the SEEK algorithm. Technical proofs are presented in Section~\ref{apdx:proof}, while additional extensions and discussions are demonstrated in Section~\ref{apdx:extension}. Throughout the paper, uppercase and lowercase letters denote random variables and their realizations, respectively, while boldface letters represent vectors or matrices.

\renewcommand\thetable{A\arabic{table}}
\renewcommand\thefigure{A\arabic{figure}}

\section{Additional Numerical Details}\label{apdx:a}
This section begins by examining the performance of an algorithmic variant introduced at the end of Section~\ref{sec:seek}. Following this, Section~\ref{sec:exp-detail} offers additional numerical configurations and results that are directly related to Section~\ref{sec:exp}. In Section~\ref{sec:addsimuresults}, we proceed with simulation results on small-scale datasets, exploring how the selection of $K$ influences the identification of significant states. Subsequently, Section~\ref{sec:details-IQL} details the implementation of IQL applied to policy learning using the MIMIC-III dataset. Lastly, in Section~\ref{sec:ohio}, we present an additional data analysis conducted on the OhioT1DM dataset.

\subsection{Balancing Type-I and Type-II errors}\label{sec:simulation-balancing}

This section examines an algorithmic variant discussed at the conclusion of Section~\ref{sec:seek}. This algorithmic variant aims at enhancing the performance of the estimated optimal policy by managing type-I and type-II errors. This variant, named ``SEEK-Alpha,'' implements a data-adaptive approach to selecting the parameter $\alpha$. Specifically, SEEK-Alpha picks the best $\alpha$ from the set $\{0.15, 0.3, 0.5, 0.65, 0.8\}$ such that it maximizes the estimated cumulative reward computed by FQE.

We conducted a simulation study to compare SEEK-Alpha against the vanilla SEEK algorithm under the CP environment. Figure~\ref{fig:select-alpha} illustrates that SEEK-Alpha generally achieves a higher TPR than SEEK, albeit with a slightly higher FDR that remains manageable. Importantly, SEEK-Alpha accurately identifies more significant states, leading to policies that generally outperform those derived by SEEK (as shown in the right panel of Figure~\ref{fig:select-alpha}). This empirical evidence supports that balancing type-I and type-II errors with SEEK-Alpha enhances policy learning in practice.

\begin{figure}[H]
	\centering
	\includegraphics[width=0.9\linewidth]{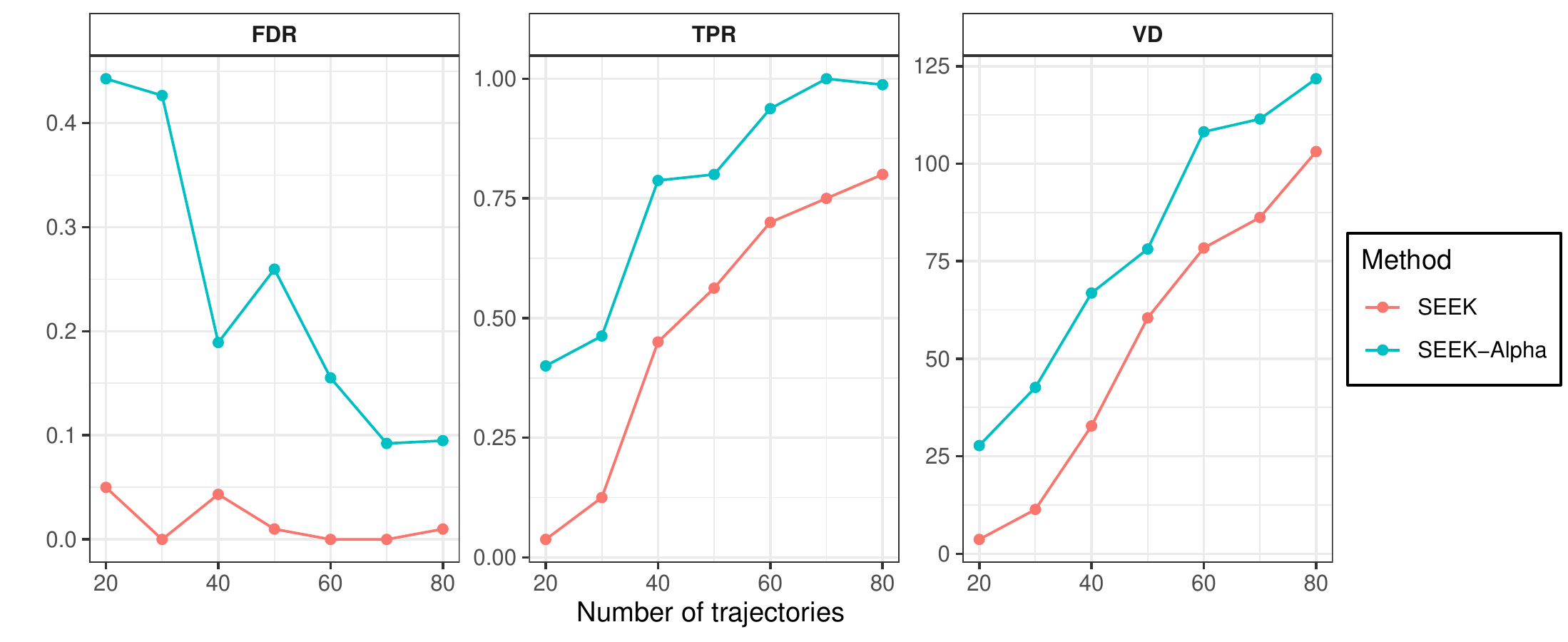}
	\vspace*{-10pt}
	\caption{The FDR, TPR, VD of SEEK and SEEK-Alpha. We use the CP environment and manually inject 96 white noises into the state to create a 100-dimensional state system.}
	\vspace*{-0.6em}
	\label{fig:select-alpha}
\end{figure}

\subsection{Simulation: Data Generating Processes, Tuning-parameters and Additional Results}\label{sec:exp-detail}
The AR environment in Section \ref{sec:exp} is generated as follows. Each state is generated according to: $S_{t+1,j}=0.9 S_{t,j}A_t+0.09S_{t,j}(1-A_t)+N(0,1)$ for any $t$ and $j$. The reward is given by $R_t=A_t (S_{t,1}+S_{t,2}) + N(0,1)$ for any $t$. 

In our implementation, we apply LASSO and random forest to compute the $W$-statistics and summarize their hyper-parameters information in Table \ref{tab:hyper-knockoff}. 
\begin{table}[H]
	\caption{Hyper-parameters information in Knockoff procedure.}\label{tab:hyper-knockoff}
	\centering
	\begin{tabular}{l|c|c}
		\toprule  
		Method & Hyper-parameter     & Setting   \\
		\midrule 
		\multirow{2}{*}{LASSO}
		& Penalty term $\lambda$ selection criterion    &  Bayesian information criterion    \\
		& Penalty term $\lambda$ selection range     & $[\exp(-20),\exp(-8)]$ \\
		\midrule
		\multirow{4}{*}{RF}
		& Number of trees in the forest & $NT/K/30$    \\
		& Maximum depth of the tree    & 4       \\
		& Number of variables for the best split    & $\sqrt{p}$  \\
		& Others   & Default values in \textsf{scikit-learn}\footnote{\url{https://github.com/scikit-learn/scikit-learn}}  \\
		\bottomrule
	\end{tabular}
\end{table}
\begin{table}[H]
	\caption{Hyper-parameters information for multilayer perceptron in policy learning.}\label{tab:hyper-policy-learning}
	\centering
	\begin{tabular}{l|c|c}
		\toprule  
		Method & Hyper-parameter     & Setting   \\
		\midrule
		\multirow{3}{*}{FQI} 
		& Number of hidden layers & 1       \\
		& Number of hidden units    & 128       \\
		& Others   & Default values in \textsf{scikit-learn}  \\
		\midrule
		\multirow{3}{*}{IQL} 
		& Number of hidden layers   & 2       \\
		& Number of hidden units    & 128       \\
		& Others   & Default values in \textsf{d3rlpy}\footnote{https://github.com/d3rlpy}  \\
		\bottomrule
	\end{tabular}
\end{table}
\begin{table}[H]
	\centering
	\caption{The cumulative reward of policies learned with all states in large-scale datasets.}\label{tab:cartpole-drl-value-large}
	\begin{tabular}{l|l|lccc}
		\toprule
		Environment & $N=p$ & 600 & 800 & 1000 \\
		\midrule
		\multirow{2}{*}{CartPole-v0}
		& IID & 13.14 & 14.21 & 14.78 \\
		& AR  & 11.71 & 12.62 & 12.36 \\
		\midrule
		\multirow{2}{*}{LunarLander-v2}
		& IID & -77.48 & -81.84 & -94.13 \\
		& AR  & -456.04 & -455.35 & -481.88 \\
		\bottomrule
	\end{tabular}
\end{table}

\subsection{Additional Simulation Results}\label{sec:addsimuresults}

\subsubsection{Small-scale Datasets}
Under small-scale settings, we choose $N$ from $\{10,20,40\}$ and $d=20$ for the AR environment and choose $N$ from $\{50,100,200\}$ and $d=20$ for the Mixed environment. As for the CP environment, $N$ is chosen from $\{100, 200\}$, and $d$ takes values in $\{50, 100, 150, 200\}$. The remaining settings are the same as that in large-scale datasets. The results under the small-scale on AR and Mixed environments are can be seen from Tables \ref{tab:ar-compare} and \ref{tab:mixed-compare}, respectively. And Figures \ref{fig:carpole_rf} demonstrate the results on the CP environments. From these results, we can see that SEEK still has the best overall variable selection performance in the small-scale datasets. 

\begin{table}[H]
	\linespread{1.25}\selectfont
	\setlength{\tabcolsep}{4.8pt}
	\caption{Performances of SEEK and benchmark methods in the AR environment, aggregated over 100 simulation replications. $T$ is fixed at 150. The cumulative rewards of policies learned by all states are 1.79, 1.80 and 1.80 when $N$ is $10, 20$ and $40$, respectively.}
	\label{tab:ar-compare}
	\centering
	\scriptsize
	\begin{tabular}{l|ccc|ccc|ccc|ccc|ccc}
	\toprule
	& \multicolumn{3}{c}{SEEK} & \multicolumn{3}{c}{Reward-only} & \multicolumn{3}{c}{One-step} &  \multicolumn{3}{c}{SFS} & \multicolumn{3}{c}{VS-LASSO} \\ 
	\hline
	$N$ & 10 & 20 & 40 & 10 & 20 & 40 & 10 & 20 & 40 & 10 & 20 & 40 & 10 & 20 & 40 \\ \hline
	$K$ & 19.60 & 24.62 & 29.28 & 19.60 & 24.62 & 29.28 & 19.60 & 24.62 & 29.28 & / & / & / & 19.60 & 24.62 & 29.28 \\ \hline
	mFDR & 0.00 & 0.00 & 0.00 & 0.00 & 0.00 & 0.00 & 0.60  & 0.60  & 0.60 & 0.02 & 0.01 & 0.02 & 0.60 & 0.60 & 0.60 \\ \hline
	FDR  & 0.00 & 0.00 & 0.00 & 0.00 & 0.00 & 0.00 & 0.90  & 0.90  & 0.90 & 0.07 & 0.03 & 0.06 & 0.90 & 0.90 & 0.90 \\ \hline
	FPR  & 0.00 & 0.00 & 0.00 & 0.00 & 0.00 & 0.00 & 1.00  & 1.00  & 1.00 & 0.02 & 0.01 & 0.02 & 1.00 & 1.00 & 1.00 \\ \hline
	TPR  & 0.96 & 1.00 & 1.00 & 0.91 & 1.00 & 1.00 & 1.00  & 1.00  & 1.00 & 1.00 & 1.00 & 1.00 & 1.00 & 1.00 & 1.00 \\ \hline
	VD  & 0.19 & 0.27 & 0.28 & 0.12 & 0.27 & 0.28 & -0.30 & -0.04 & 0.06 & 0.23 & 0.27 & 0.28 & -0.30 & -0.03 & 0.10 \\
	\bottomrule
	\end{tabular}
\end{table}
\begin{table}[H]
	\linespread{1.25}\selectfont
	\caption{Performances of SEEK and benchmark methods in the Mixed environment, aggregated over 100 independent replications. $T$ is fixed at 150. The cumulative rewards of policies learned by the all states are 2.19 when $N$ is $50, 100$ and $200$.}
	\label{tab:mixed-compare}
	\centering
	\scriptsize
	\begin{tabular}{l|ccc|ccc|ccc|ccc|ccc}
	\toprule
	& \multicolumn{3}{c}{SEEK} & \multicolumn{3}{c}{Reward-only} & \multicolumn{3}{c}{One-step} & \multicolumn{3}{c}{SFS} & \multicolumn{3}{c}{VS-LASSO} \\ \hline
	$N$ & 50 & 100 & 200 & 50 & 100 & 200 & 50 & 100 & 200 & 50 & 100 & 200 & 50 & 100 & 200 \\ 
	\hline
	$K$   & 3.22 & 4.16 & 4.56 & 3.22 & 4.16 & 4.56 & 3.22 & 4.16 & 4.56 & / & / & / & 3.22 & 4.16 & 4.56 \\ \hline
	mFDR  & 0.36  & 0.45 & 0.45 & 0.04  & 0.01 & 0.01 & 0.02  & 0.03 & 0.03 & 0.03  & 0.02 & 0.01 & 0.00 & 0.00 & 0.00 \\ \hline
	FDR   & 0.97  & 0.90 & 0.90 & 0.12  & 0.04 & 0.05 & 0.11  & 0.16 & 0.18 & 0.14  & 0.09 & 0.04 & 0.00 & 0.00 & 0.00 \\ \hline
	FPR   & 0.33  & 0.50 & 0.50 & 0.04  & 0.01 & 0.01 & 0.02  & 0.03 & 0.02 & 0.02  & 0.02 & 0.01 & 0.00 & 0.00 & 0.00 \\ \hline
	TPR   & 0.06  & 0.49 & 0.50 & 0.71  & 0.98 & 1.00 & 0.51  & 0.50 & 0.53 & 0.51  & 0.53 & 0.54 & 0.50 & 0.50 & 0.50 \\ \hline
	Value & -0.04 & 0.00 & 0.00 & -0.02 & 0.00 & 0.00 & -0.02 & 0.00 & 0.00 & -0.02 & 0.00 & 0.00 & -0.02 & 0.00 & 0.00 \\ 
	\bottomrule
	\end{tabular}
\end{table}
\begin{figure}[H]
	\vspace*{-10pt}
    \centering
    \includegraphics[width=1.0\linewidth]{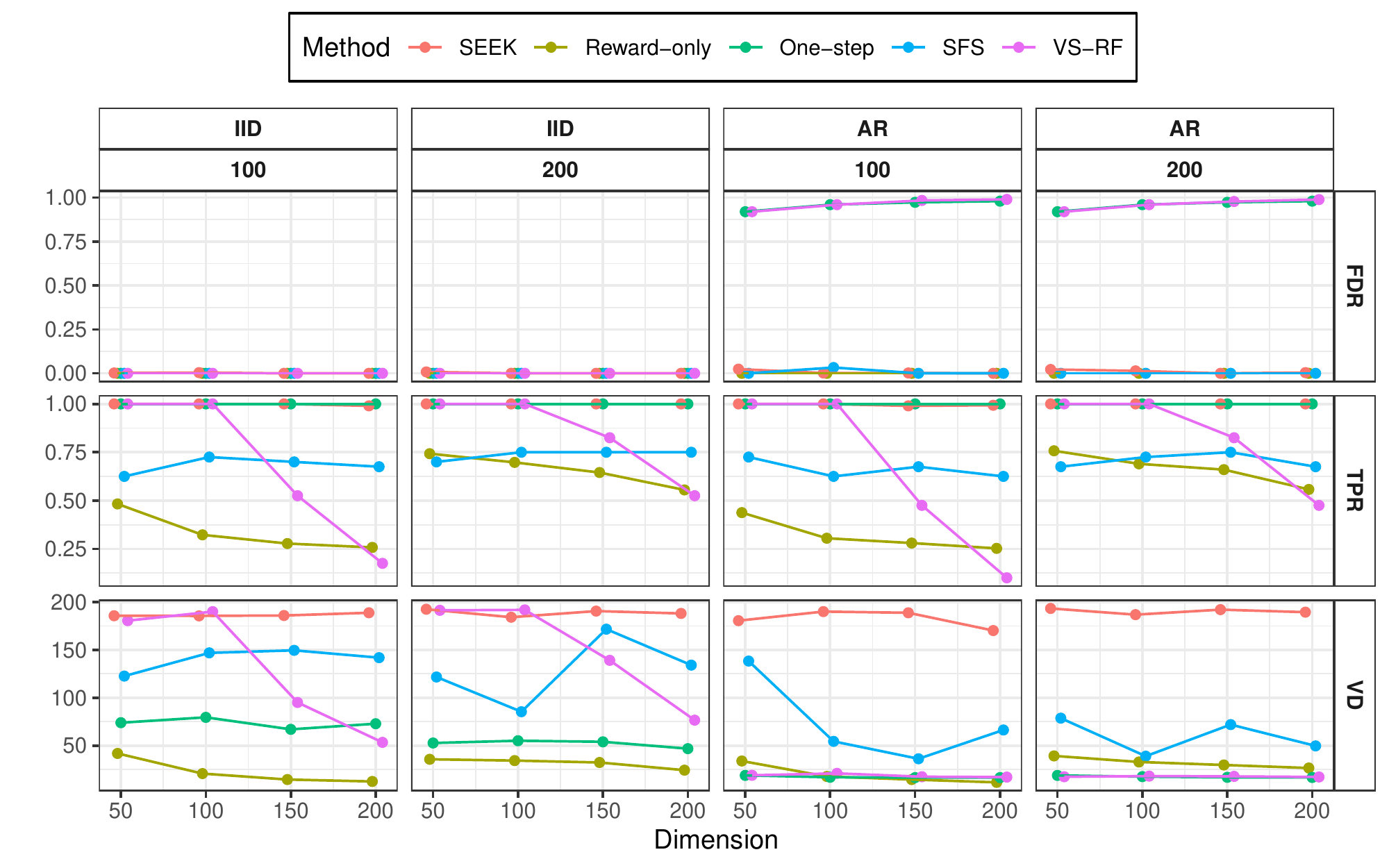}
	\vspace*{-20pt}
    \caption{Performance of methods in the CartPole-v0 environment across 100 simulation runs. Each row represents an evaluation criterion. Columns 1 and 3 show results for $N=100$, while other columns depict results for $N=200$. Columns 1-2 correspond to settings with independent null states, and columns 3-4 correspond to settings with the AR(1)-structure null states. The cumulative rewards of policies learned by the all states are summarized in Table~\ref{tab:cartpole-drl-value}.}
    \label{fig:carpole_rf}
\end{figure}
\begin{table}[H]
	\centering
	\caption{The cumulative rewards of policies learned by the all states of datasets generated from the CartPole-v0 environment.}\label{tab:cartpole-drl-value}
	\vspace*{-0.6em}
	\scriptsize
	\begin{tabular}{ll|cccc}
		\toprule
		$n$ & Cor & 50 & 100 & 150 & 200 \\
		\midrule
		\multirow{2}{*}{100} 
		& IID & 21.98 & 13.77 & 13.23 & 13.61 \\
		& AR & 17.18 & 13.81 & 13.50 & 13.87 \\
		\midrule
		\multirow{2}{*}{200} 
		& IID & 41.58 & 17.41 & 14.40 & 14.05 \\
		& AR & 20.02 & 15.91 & 14.39 & 14.16 \\
		\bottomrule
	\end{tabular}
\end{table}

\subsubsection{Selection $K$}
To investigate the proposed $K$-selection algorithm in Section~\ref{sec:bestK}, we conduct another analysis that applies SEEK to the first two environments with $K$ fixed to $5,10,20,40$, and we compare the results against those obtained based on SEEK with adaptively selected $K$. The numerical results are summarized in Table~\ref{tab:ar-mixed-bestK}. It can be seen from Table \ref{tab:ar-mixed-bestK} that mFDR, FDR, and TPR depend on the specification of $K$. In particular, when $K$ is moderately large (e.g., 10 or 20),  the TPR is reduced to half whereas the mFDR, FDR, and FPR are exactly equal to zero.  On the contrary, the proposed $K$-selection algorithm tends to select a small value of $K$ in the Mixed environment, achieving a better balance between (m)FDR/FPR and TPR. 


\begin{table}[H]
	\linespread{1.25}\selectfont
	\setlength{\tabcolsep}{4.2pt}
	\caption{Performances of SEEK for different $K$ in the AR environment, aggregated over 100 simulation runs. $T$ is fixed at 150. The cumulative reward of policies learned by all states in the AR environment are listed in the caption of Table~\ref{tab:ar-compare} while that in the Mixed environment are listed in the caption of Table~\ref{tab:mixed-compare}.}
	\label{tab:ar-mixed-bestK}
	\centering
	\scriptsize
	\begin{tabular}{l|c|ccc|ccc|ccc|ccc|ccc}
	\toprule
	& & \multicolumn{3}{c}{Selected $K$} & \multicolumn{3}{c}{$K=5$} & \multicolumn{3}{c}{$K=10$} & \multicolumn{3}{c}{$K=20$} & \multicolumn{3}{c}{$K=40$} \\ 
	\midrule
	\multirow{7}{*}{AR}
	& $N$ & 10 & 20 & 40 & 10 & 20 & 40 & 10 & 20 & 40 & 10 & 20 & 40 & 10 & 20 & 40 \\ \cline{2-17}
	& $K$ & 19.60 & 24.62 & 29.28 & / & / & / & / & / & / & / & / & / & / & / & / \\ \cline{2-17}
	& mFDR & 0.00 & 0.00 & 0.00 & 0.01 & 0.01 & 0.01 & 0.00 & 0.00 & 0.00 & 0.00 & 0.00 & 0.00 & 0.03 & 0.00 & 0.00 \\ \cline{2-17}
	& FDR & 0.00 & 0.00 & 0.00 & 0.03 & 0.04 & 0.03 & 0.01 & 0.02 & 0.00 & 0.00 & 0.00 & 0.00 & 0.12 & 0.00 & 0.00 \\ \cline{2-17}
	& FPR & 0.00 & 0.00 & 0.00 & 0.00 & 0.01 & 0.00 & 0.00 & 0.00 & 0.00 & 0.00 & 0.00 & 0.00 & 0.02 & 0.00 & 0.00 \\ \cline{2-17}
	& TPR & 0.96 & 1.00 & 1.00 & 1.00 & 1.00 & 1.00 & 1.00 & 1.00 & 1.00 & 0.89 & 1.00 & 1.00 & 1.00 & 0.98 & 1.00 \\ \cline{2-17}
	& VD & 0.19 & 0.27 & 0.28 & 0.24 & 0.26 & 0.28 & 0.24 & 0.26 & 0.28 & 0.08 & 0.27 & 0.28 & 0.22 & 0.24 & 0.28 \\ 
	\midrule
	\multirow{7}{*}{Mixed}
	& $N$ & 50 & 100 & 200 & 50 & 100 & 200 & 50 & 100 & 200 & 50 & 100 & 200 & 50 & 100 & 200 \\ \cline{2-17}
	& $K$ & 3.22 & 4.16 & 4.56 & / & / & / & / & / & / & / & / & / & / & / & / \\ \cline{2-17}
	& mFDR & 0.04  & 0.01 & 0.01 & 0.01  & 0.00 & 0.00 & 0.00  & 0.00 & 0.00 & 0.00 & 0.00 & 0.00 & 0.00 & 0.00 & 0.00 \\ \cline{2-17}
	& FDR  & 0.12  & 0.04 & 0.05 & 0.03  & 0.02 & 0.01 & 0.00  & 0.00 & 0.00 & 0.00 & 0.00 & 0.00 & 0.00 & 0.00 & 0.00 \\ \cline{2-17}
	& FPR  & 0.04  & 0.01 & 0.01 & 0.00  & 0.00 & 0.00 & 0.00  & 0.00 & 0.00 & 0.00 & 0.00 & 0.00 & 0.00 & 0.00 & 0.00 \\ \cline{2-17}
	& TPR  & 0.71  & 0.98 & 1.00 & 0.50  & 0.95 & 1.00 & 0.50  & 0.50 & 0.98 & 0.50 & 0.50 & 0.50 & 0.50 & 0.50 & 0.50 \\ \cline{2-17}
	& VD  & -0.02 & 0.00 & 0.00 & -0.02 & 0.00 & 0.00 & -0.02 & 0.00 & 0.00 & -0.02 & 0.00 & 0.00 & -0.02 & 0.00 & 0.00 \\
	\bottomrule
	\end{tabular}
\end{table}

In addition, we conduct a sensitivity analysis to test the robustness of four SEEK methods (SEEK-LASSO, SEEK-RF, SEEK-LASSO+, and SEEK-RF+, where SEEK-X+ refers to the SEEK-X algorithm with knockoffs+ for variable selection) to different choices of $K$ and $q$. We focus on the CP environment with $N=100$, $p=100$, and AR(1) noise, set $K = k_0\log(NT)$ with $k_0\in \{1,1.5,2\}$, choose $q\in(0,0.3]$ for the FDR control, and plot of the corresponding FDRs and TPRs in Figures \ref{fig:sa1} and \ref{fig:sa2} respectively. It can be seen that our method controls the FDR for any $K$ and $q$. In addition, the performance is not overly sensitive to the choice of $K$. 

\begin{figure}[H]
\centering
\begin{subfigure}{0.32\linewidth}
\centering
\includegraphics[width=1\linewidth]{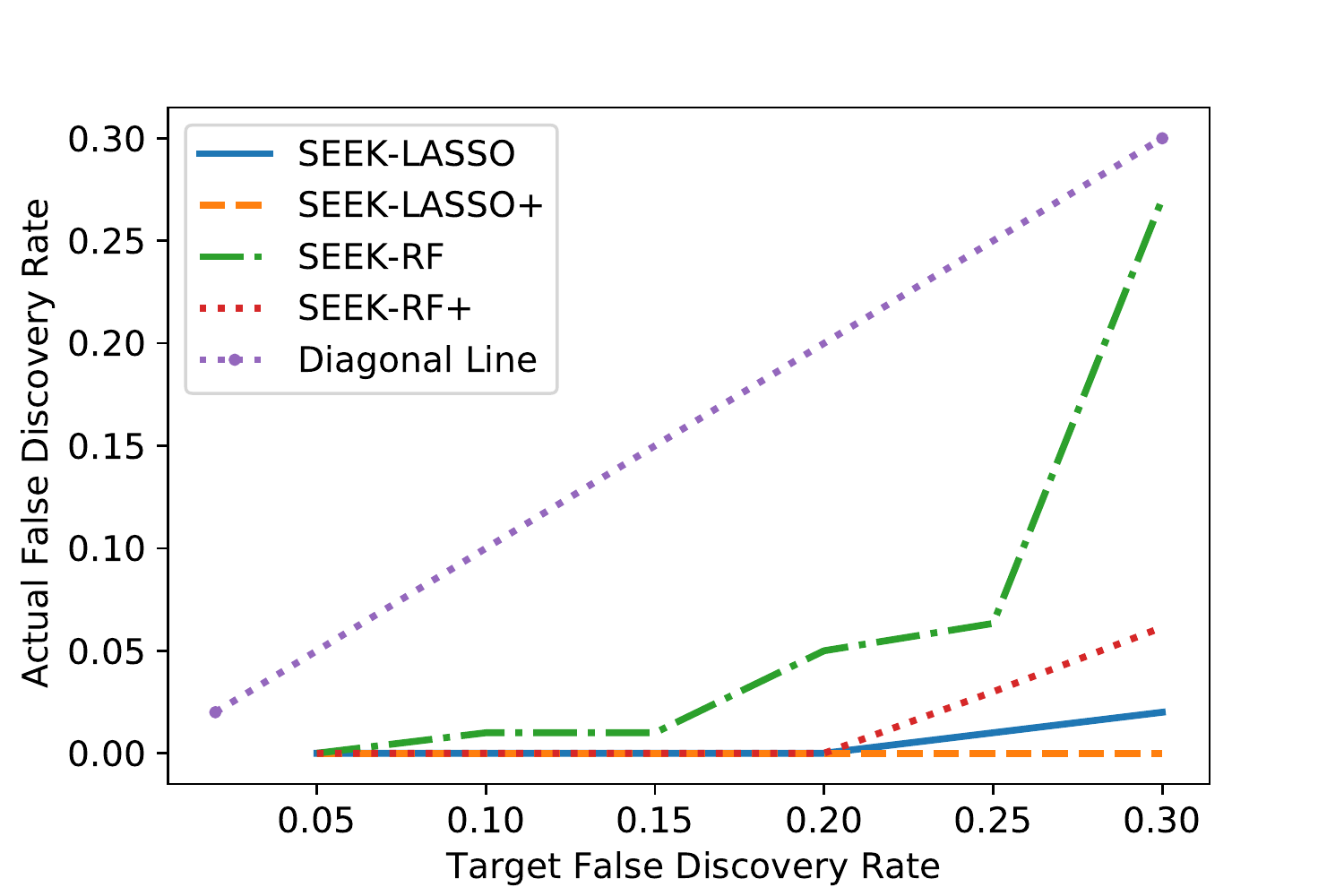} 
\caption{$K =\log(NT)$}
\end{subfigure}%
\begin{subfigure}{0.32\linewidth}
\centering
\includegraphics[width=1\linewidth]{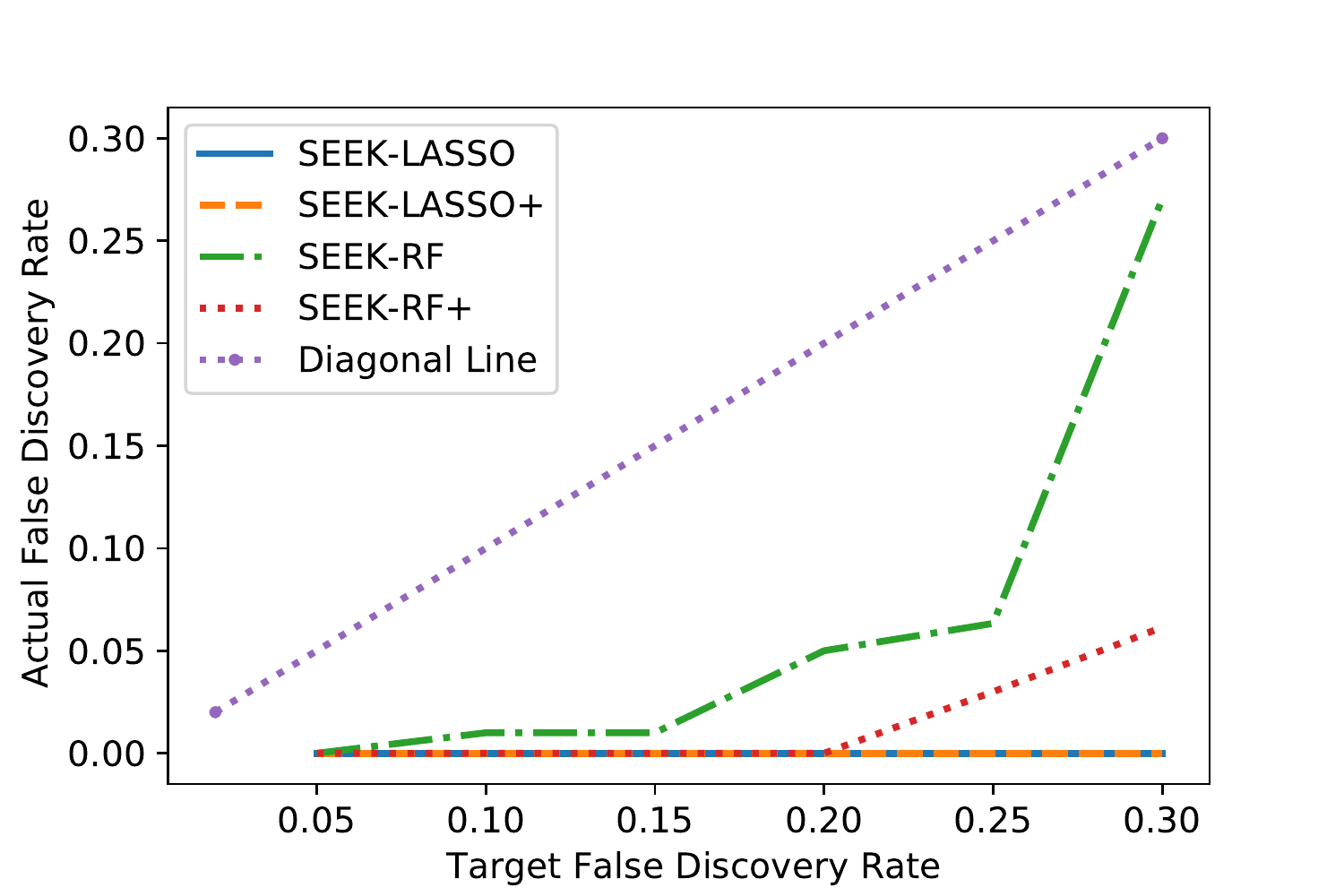} 
\caption{$K = 1.5\log(NT)$}
\end{subfigure}%
\begin{subfigure}{0.32\linewidth}
\centering
\includegraphics[width=1\linewidth]{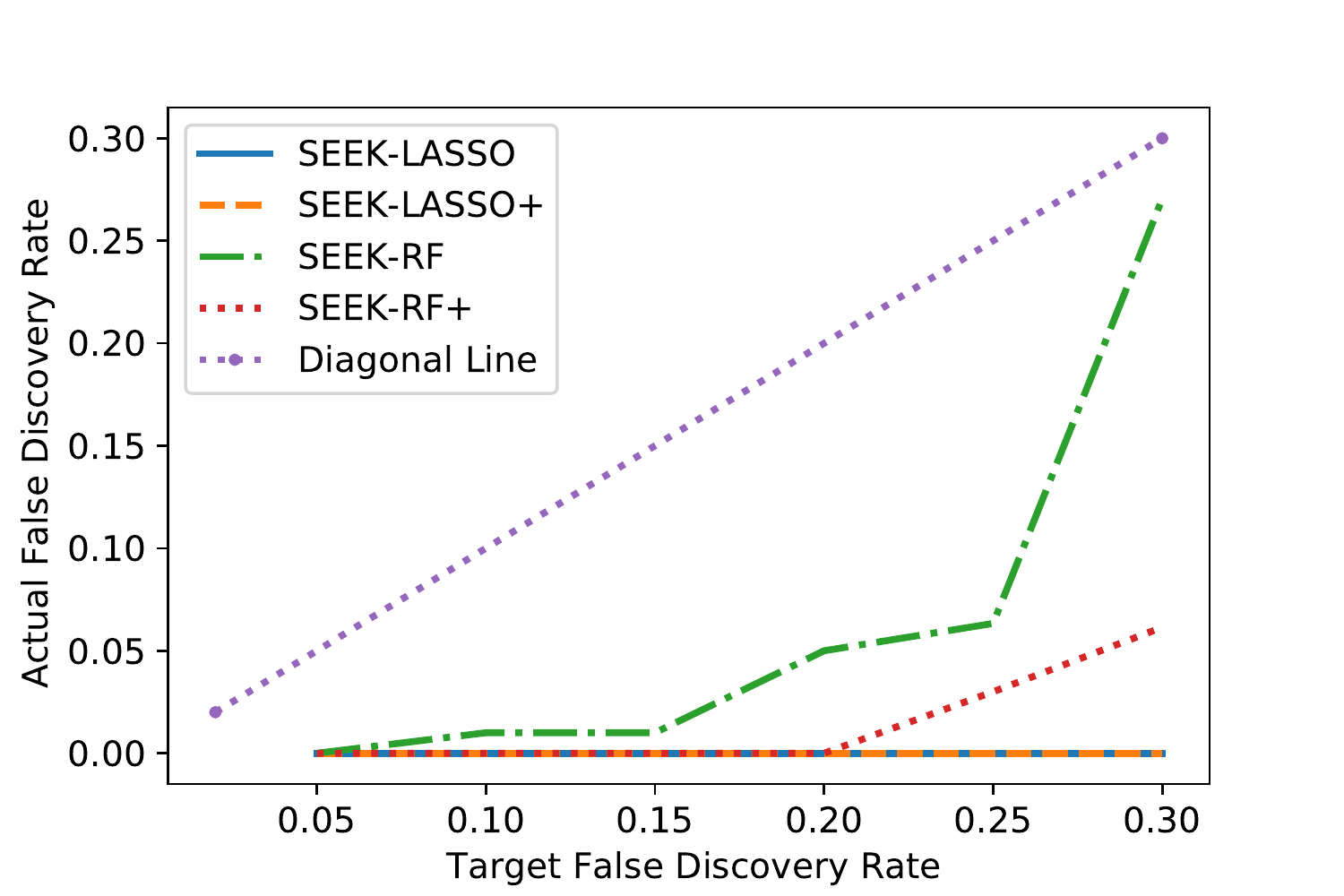} 
\caption{$K = 2\log(NT)$}
\end{subfigure}%
\caption{Results of the FDR using the proposed SEEK methods under the CP environment with $N=100$, $p=100$, $\alpha = 0.5$, and AR(1) noises. All the results are aggregated over 20 runs.}
\label{fig:sa1}
\end{figure}
\begin{figure}[H]
\centering
\begin{subfigure}{0.32\linewidth}
\centering
\includegraphics[width=1\linewidth]{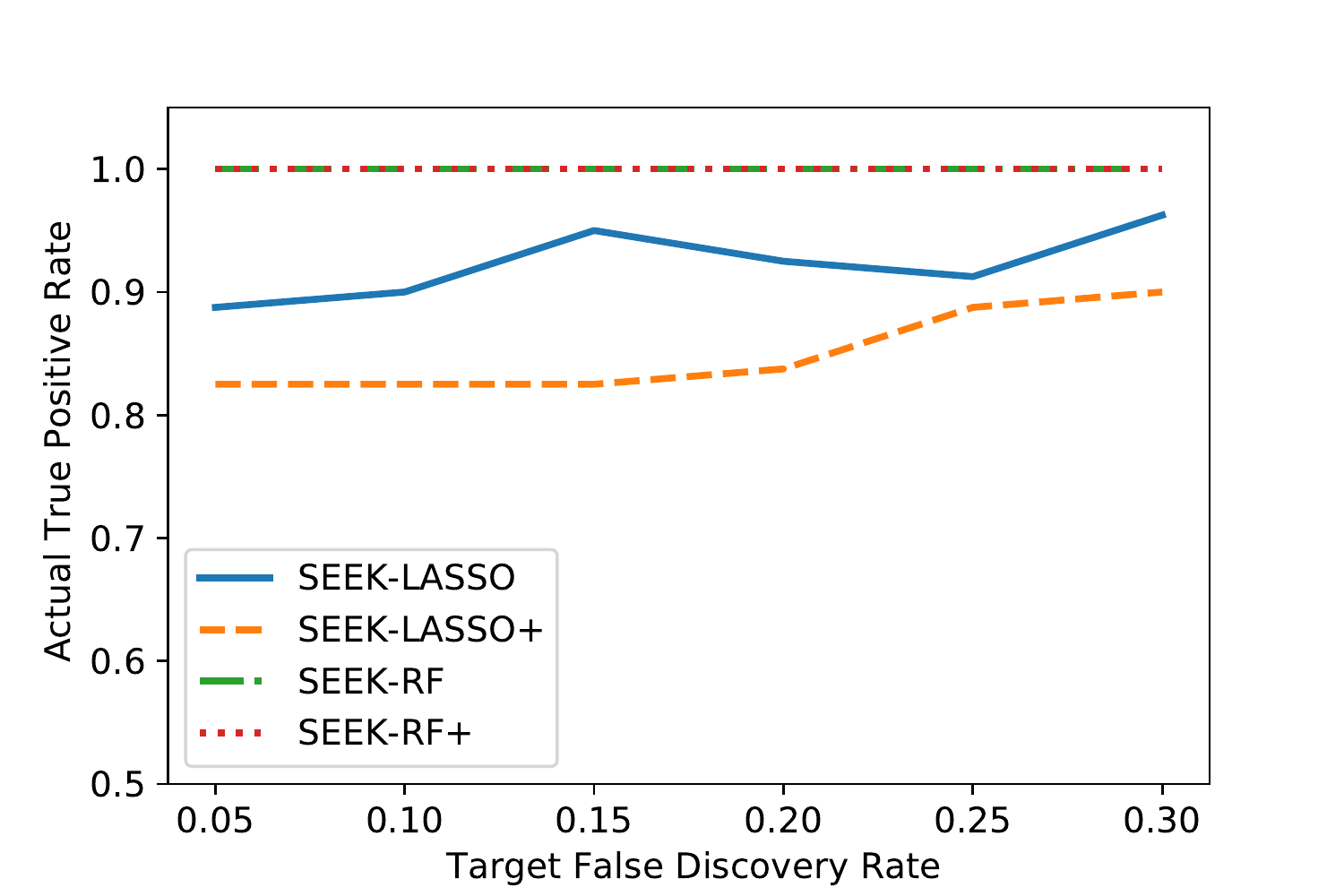} 
\caption{$K = \log(NT)$}
\end{subfigure}%
\begin{subfigure}{0.32\linewidth}
\centering
\includegraphics[width=1\linewidth]{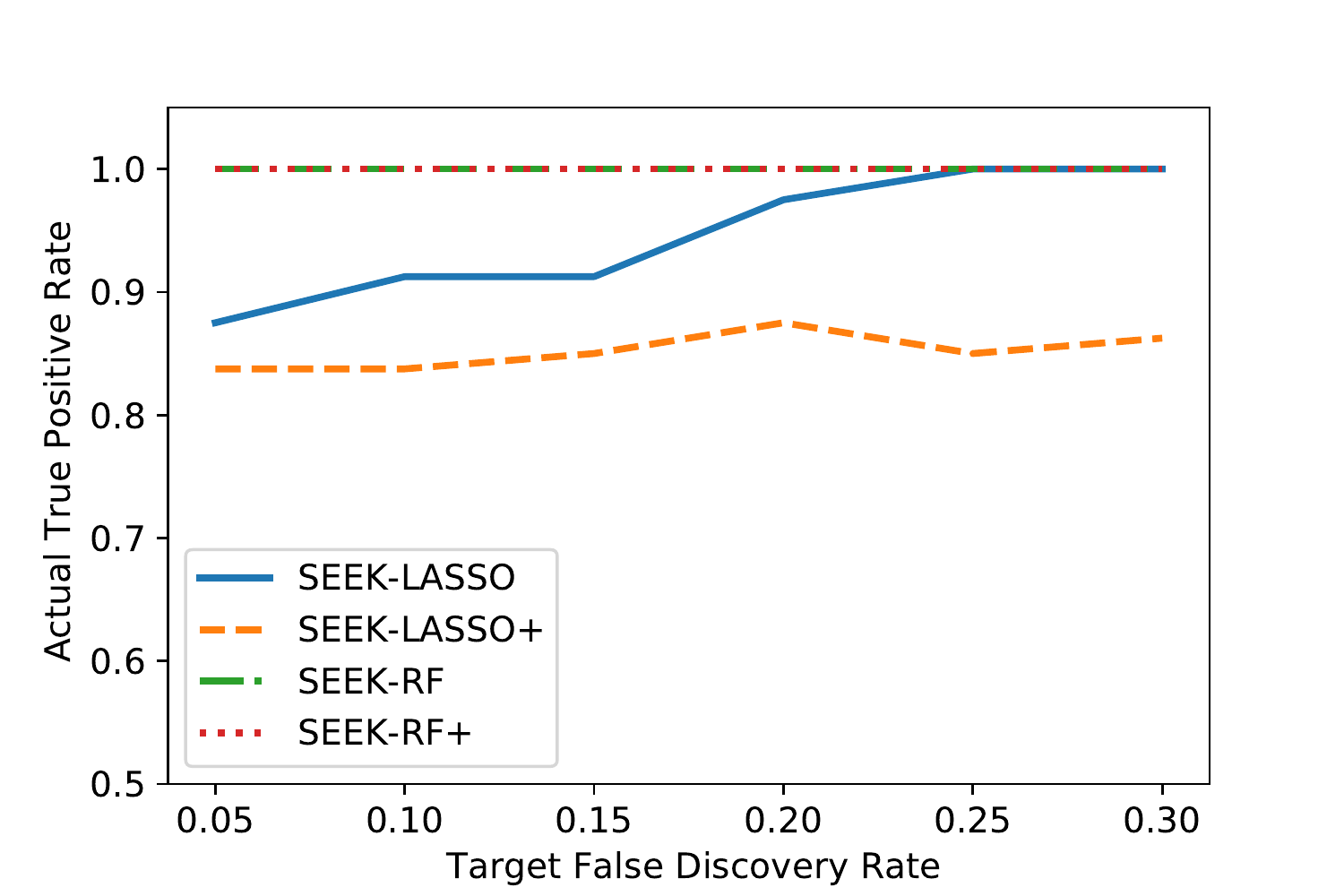} 
\caption{$K = 1.5\log(NT)$}
\end{subfigure}%
\begin{subfigure}{0.32\linewidth}
\centering
\includegraphics[width=1\linewidth]{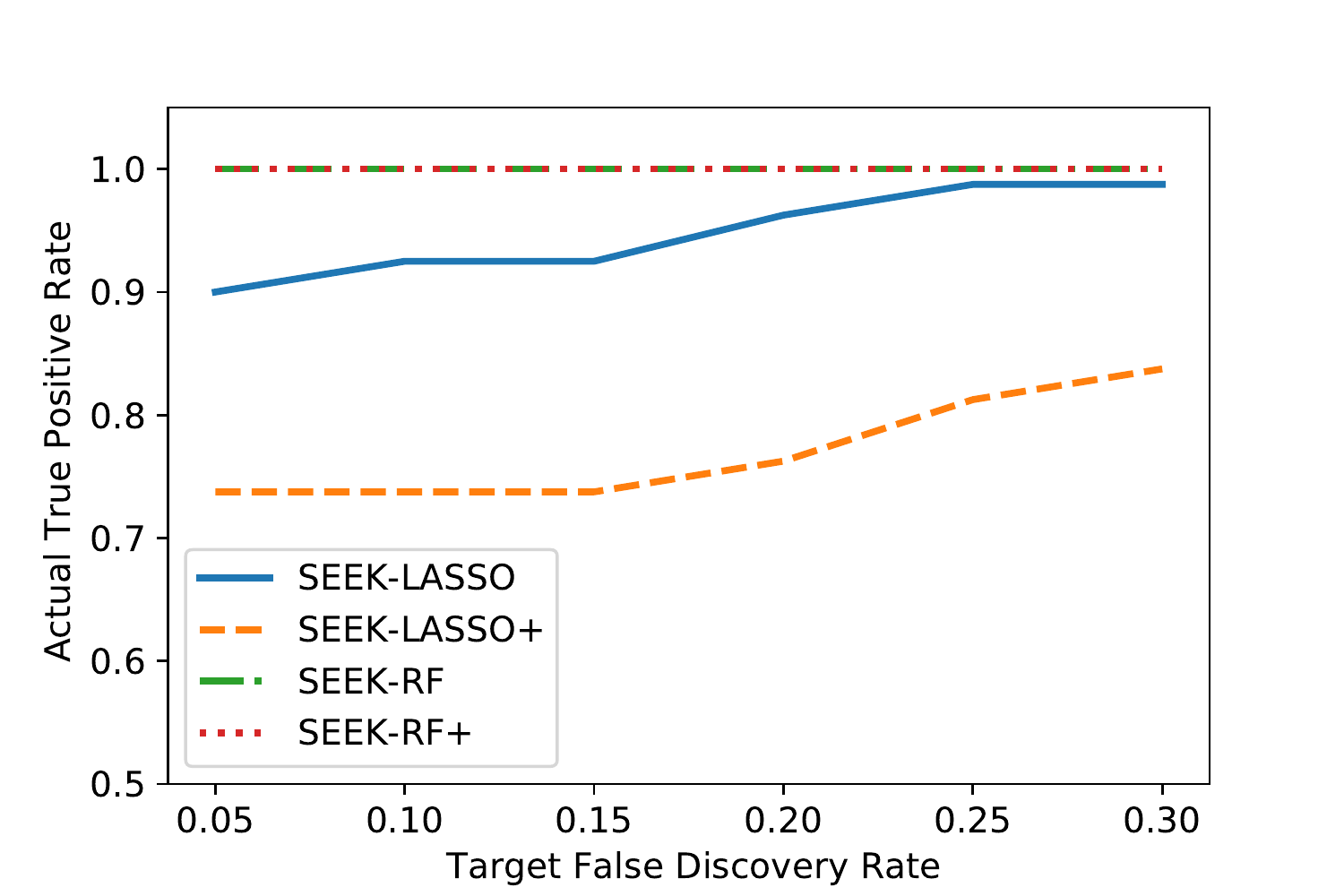} 
\caption{$K = 2\log(NT)$}
\end{subfigure}%
\caption{Results of the TPR using the proposed SEEK methods under the CP environment with $N=100$, $p=100$, $\alpha = 0.5$, and AR(1) noises. All the results are aggregated over 20 runs.}
\label{fig:sa2}
\end{figure}

\subsection{MIMIC-III Data Analysis: Implementation Details of IQL}\label{sec:details-IQL}

IQL is a state-off-the-art offline RL algorithm that is good at tackling both state and action (after one-hot encoding) with a large dimension. We adopt an offline deep reinforcement learning library, the \textsf{d3rlpy}\footnote{\url{https://github.com/takuseno/d3rlpy}} library, which provides an implementation for IQL. The IQL uses an actor-critic framework, where the neural network architecture for actor and critic is set as the default in \textsf{d3rlpy}, and the optimizer for training parameters is also set as the default values. Specifically, both actor and critic networks are multilayer perceptron with two hidden layers where each layer has 256 hidden units. Two key tuning parameters of IQL, i.e., the expectile value and the inverse temperature, are fixed as 0.7 and 3.0. We update actor-critic networks 100 thousand steps, each update with 256 samples. 



\subsection{Analysis of the OhioT1DM Dataset}\label{sec:ohio}
In this section, we apply SEEK to the OhioT1DM dataset \citep{marling2020ohiot1dm} which contains data from 6  patients with type-I diabetes. For each patient, their glucose levels and self-reported meals are continuously measured over 8 weeks. These variables can be used to construct data-driven decision rules to determine whether a patient needs insulin injections to improve their health \citep{luckett2019estimating,shi2020statistical,zhou2022estimating}. 

To analyze this data, we divide the eight weeks into one-hour intervals and compute the average glucose level $G_t$ and the average carbohydrate estimate for the meal $M_t$ over each one-hour interval $(t-1,t]$. Previous studies found these variables do not satisfy the Markov assumption \citep[see e.g.,][]{shi2020does}. This motivates us to construct the state $S_t$ by including both the current observations $G_t$, $M_t$ as well as the past measurements within 5 hours, i.e., $G_{t-i}$ and $M_{t-i}$ for $i=1,2,3,4,5$. We also manually include 10 null variables (denoted by $N_1,\cdots,N_{10}$) in the state, each following an AR(1) model, given by $N_{t+1,j}=0.5 N_{t,j}+\mathcal{N}(0,1)$ for $j \in \{1, \ldots, 10\}$.  The action $A_t$ is binary, indicating whether the patient injected the insulin or not in the last hour. 

We next apply SEEK, Reward-only, One-step, VS-LASSO, and SFS to 4 out of 6 data trajectories for variable selection, iterate this procedure for all $\binom{6}{4} =15$ combinations, compute the percentage of each state variable being selected, and report the results in Figure~\ref{fig:real}. It can be seen that (i) the one-step method tends to select all states and thus it has a high FDP; (ii) the reward-only merely has a 30\% chance of identifying the most relevant variable while missing the remaining relevant variables; (iii) VS-LASSO has a high FDP with each of the null variables being selected at least $50\%$ of the time; (iv) SFS fails to identify the glucose levels and carbohydrate estimates as significant variables, suffering from a low TPR; (v) SEEK selects no null variable and has a high success rate in selecting glucose levels and carbohydrate estimates. 

\begin{figure}[H]
	\centering
	\includegraphics[width=\linewidth]{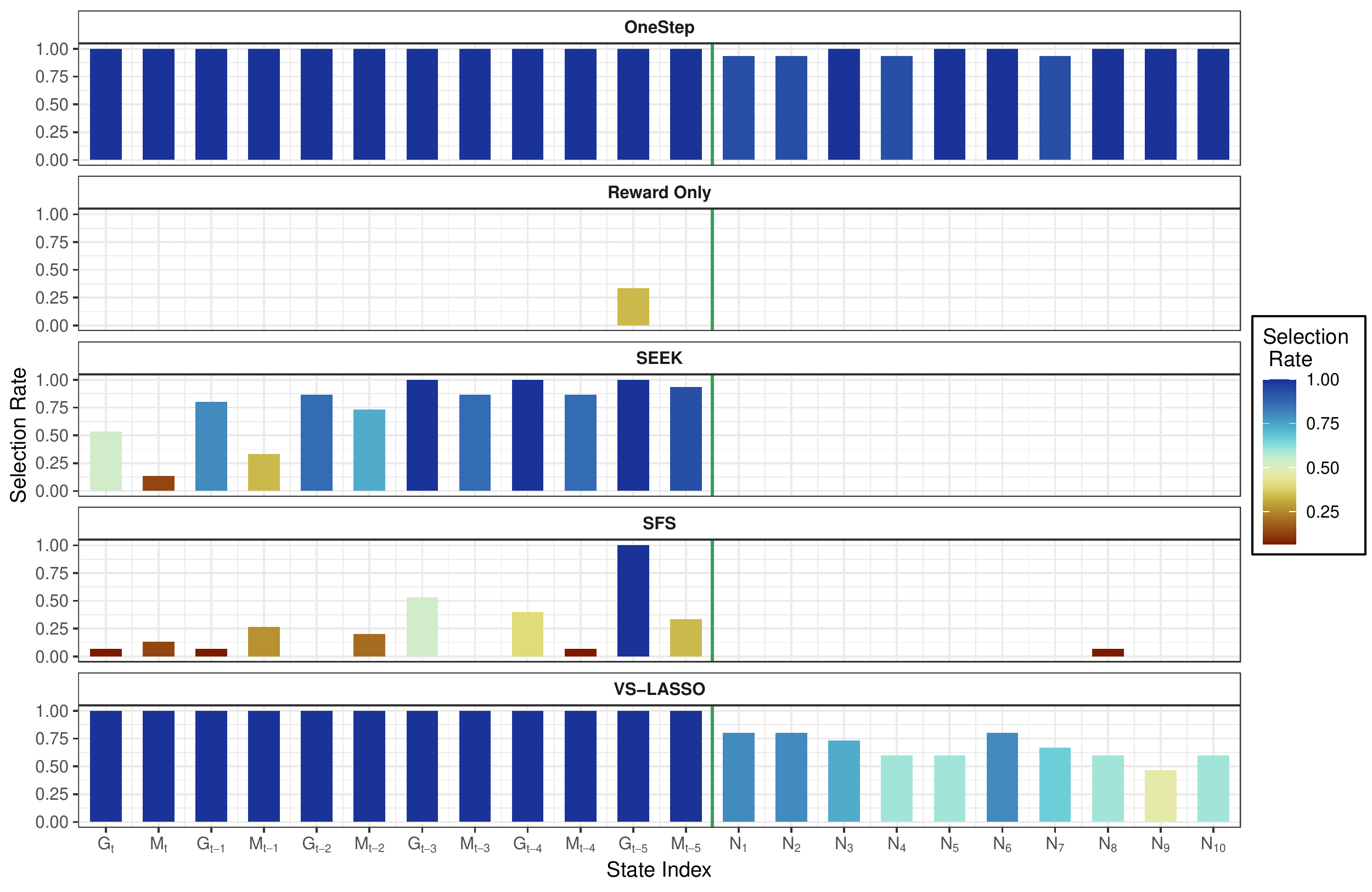}
	\caption{Selection rates of One-step, Reward-only, SEEK, SFS, and VS-LASSO on the Ohio datasets.}
	\label{fig:real}
\end{figure}

\section{Details of Algorithms}\label{apdx:b}

In this section, we provide more implementation details about the proposed algorithm. We first present our algorithm for estimating the $\beta$-mixing coefficients and determining to the optimal number of data splits $K$ in Section \ref{sec:bestK}. We next detail the use of the second-order machine or Gaussian sampling to construct knockoff states in  Section \ref{sec:secondorder}. Next, we outline the process of constructing feature importance using generic machine learning algorithms in Section~\ref{sec:genericML}. Lately, we present the pesudo code of benchmarked methods in Section~\ref{sec:algo-rewardonly-onestep}.

\subsection{Determining the Number of Splits}\label{sec:bestK}
The performance of the proposed algorithm relies on the number of sub-datasets $K$ in the data splitting step. In this section, we first present a practical algorithm to adaptively select $K$ in Section~\ref{sec:bestK-algo}. We next establish the consistency of the algorithm in Section~\ref{sec:bestK-consistency}. Finally, we conduct simulation studies to investigate its consistency and sensitivity in Section~\ref{sec:bestK-simulation}.

\subsubsection{Algorithmic Procedure}\label{sec:bestK-algo}
A closer look at the proof for Theorem \ref{thm1} reveals that for a given $K$, 
the FDR is bounded above by
$q + (NT/K)\beta(K),$ where $\beta(K)$ denotes the $\beta$-mixing coefficient, measuring the temporal dependence of the MDP  
at lag $K$. See Definition \ref{def:beta-mix} for details. Thus, the choice of $K$ represents the following trade-off. On the one hand, when the system is $\beta$-mixing, we have $\lim_{K\to\infty} \beta(K)=0$. In theory, $K$ should diverge to infinity to avoid inflation of type-I error. On the other hand, $K$ should not be too large in order to guarantee that each data subset $\calD_k$ has sufficiently many observations. To balance this trade-off, we propose to set $K$ to $K^*=\argmin \{k\ge 1: k^{-1} (NT)\beta(k)\le  \delta\}$ and estimate it via
\begin{align}\label{k_critn}
    \widehat{K} = \argmin \left\{k\ge 1: \frac{NT}{k}\widehat{\beta}(k) \leq \delta \right\},
\end{align}
where $\widehat{\beta}(\bullet)$ denotes the estimated mixing coefficient and $\delta$ denotes a pre-specified upper bound for the inflation of the type-I error (e.g., 0.01 or 0.05). 

It remains to estimate the $\beta$-mixing coefficient. Existing solutions in the time series literature are based on histograms
 \citep{mcdonald2015estimating} and thus suffer from the curse of dimensionality and perform poorly in problems with moderate dimension. Furthermore, it is
 not trivial to extend these methods to produce an accurate estimator 
 of $K^*$ without imposing additional structure on the problem.    
  To see this, suppose that $\beta(k)$ decays exponentially fast with respect to the lag $k$. Then $\beta(K^*)$ is typically of the order 
$\delta (NT)^{-1}$ up to some logarithmic factor. Nonetheless, without additional assumptions, its estimation error is at least of the order $(N T)^{-1/2}\gg \delta (NT)^{-1}$. 
To address both challenges, we develop a novel three-step algorithm detailed below. 

\textbf{Step 1. } 
Construct initial estimators of $\beta(1)$, $\cdots$, $\beta(K_0)$ for a specified integer value of $K_0$ using a generic density estimator. Assuming both the behavior policy $\pi_b$ and the MDP are stationary, it follows that 
\begin{eqnarray}\label{eqn:betak}
\begin{split}
    \beta(k)
    &=\frac{1}{2}\sum_{a,a'} \int_{s,s'} |f(s) \pi_b (a|s)f(s')\pi_b(a'|s')-f_k(s,s')\pi_b'(a|s,s')\pi_b(a'|s')|\\
    &=\frac{1}{2}\sum_{a} \int_{s,s'} |f(s) \pi_b (a|s)f(s')-f_k(s,s')\pi_b' (a|s,s')|,
\end{split} 
\end{eqnarray}
where $f$ denotes the marginal state density function, $f_k$ denotes the joint distribution function of $(\mathbf{S}_t, \mathbf{S}_{t+k})$, and $\pi_b'$ is the conditional probability of the action given both the current state and the state after lag $k$. To simplify the calculation, we assume $\pi_b$ and $\pi_b'$ are approximately the same so that $\beta(k)$ can be approximated by $\int_{s,s'} |f(s) f(s')-f_k(s,s')|/2$. Next, we estimate $f$, 
and $f_k$, and then plug in these estimators to estimate $\beta(k)$. In practice, we use a 
Gaussian mixture model to approximate $f$ based on tuples $(\mathbf{S}_t)$, 
and similarly to approximate $f_k$ based on tuples $(\mathbf{S}_t,\mathbf{S}_{t+k})$. 

In addition, we apply importance sampling to numerically calculate the integral in \eqref{eqn:betak}. A naive method is to uniformly sample $(S_i,S_i')$ over the product space $\mathcal{S}\times \mathcal{S}$ and estimate $\beta(k)$ by
    \begin{align}\label{eqn:naiveest}
        \frac{|\mathcal{S}|^2}{2M} \sum_{i=1}^M |\widehat{f}(S_i) \widehat{f}(S_i')-\widehat{f}_k(S_i,S_i')|,
    \end{align}
where $\widehat{f}$ and $\widehat{f}_k$ denote the corresponding estimators for $f$ and $f_k$, respectively. However, this method requires the state space to be bounded. In addition, the resulting estimator may suffer from a large variance when the area of the state space is large. 
    
In our implementation, we propose to sample $(S_i,S_i')$ according to a mixture distribution $[\widehat{f}(S_i) \widehat{f}(S_i')+\widehat{f}_k(S_i,S_i')]$. This yields the following estimator 
\begin{align*}
    \frac{1}{M}\sum_{i=1}^M\frac{|\widehat{f}(S_i) \widehat{f}(S_i')-\widehat{f}_k(S_i,S_i')|}{|\widehat{f}(S_i) \widehat{f}(S_i')+\widehat{f}_k(S_i,S_i')|}.
\end{align*}
Different from \eqref{eqn:naiveest}, the above importance sampling ratio is strictly smaller than $1$, reducing the variance of the resulting estimator. When $\widehat{f}$ and $\widehat{f}_k$ are consistent, it can be shown that its variance is upper bounded by $\beta(k)/M$ asymptotically.

We denote the resulting estimators by $\widetilde{\beta}(1)$, $\cdots$, $\widetilde{\beta}(K_0)$. 


\textbf{Step 2. } 
The second step is to impose a parametric structure on the mixing coefficients to refine the initial estimators and to estimate $\beta(k)$ for $k\ge K_0$. We assume $\beta(k)=a_0 \exp(-b_0 k)$ for some $a_0,b_0>0$. To estimate these model parameters, we assume the first-step estimators satisfy
\begin{eqnarray}\label{eqn:initialestmodel}
    \widetilde{\beta}(k)=\eta_0 + a_0\exp(-b_0k) + \epsilon_k,
\end{eqnarray}
for some $\eta_0 \in\mathbb{R}$ and mean-zero random errors $\{\epsilon_k\}_{k=1}^{K_0}$. Here, $\eta_0$ and $\epsilon_k$ measure the bias and variance of the initial estimator. Under the model \eqref{eqn:initialestmodel}, the estimators $(\widehat{\eta}, \widehat{a}, \widehat{b})$ are computed by minimizing the following nonlinear least square loss:
\begin{align}\label{eqn:estK}
	\mathop{\textup{argmin}}_{a, b \in (0, \infty), \eta\in \mathbb{R}} \frac{1}{K_0}\sum_{k=1}^{K_0} \left(\widetilde{\beta}(k) - \eta - a \exp(-b k) \right)^2,
\end{align}
for some constant integer $K_0>3$. Given the convex nature of the objective function and the fact that it involves only three parameters, the global optimum can be efficiently computed. In our implementation, we solve this optimization using the well-established Broyden-Fletcher-Goldfarb-Shanno (BFGS) algorithm.

It is worth mentioning that we remove the bias term when constructing the final estimator to ensure that $\widehat{\beta}(k)$ indeed decays exponentially fast with respect to $k$. Nonetheless, in our numerical experiments, we find that the inclusion of $\eta_0$ in \eqref{eqn:initialestmodel} is essential to ensure the consistency of $\widehat{\beta}(k)$ due to the finite sample bias of $\widetilde{\beta}(k)$. 

\textbf{Step 3}. We set $\widehat{K}$ to the smallest integer $k$ such that $\widehat{\beta}(k)$ is no larger than $\delta(NT)^{-1} k$. We find that the choice of $\widehat{K}$ is not overly sensitive to the value of $K_0$ used in the first step. 

\subsubsection{Consistency of the Estimated $K$}\label{sec:bestK-consistency} 
We first present conditions for establishing the property. The first condition says the model used in Step 2 is correctly specified.
\begin{condition}\label{cond:initialest}
	Suppose the initial estimators $\widetilde{\beta}(1), \ldots, \widetilde{\beta}(K_0)$ satisfy:
	\begin{equation}\label{eq:init-beta-estimator}
		\widetilde{\beta}(k) =  a_0 \exp(-b_0 k)+\eta_0 + \epsilon_k,
	\end{equation}
	for some $a_0,b_0 \in (0, \infty)$ and $\eta_0\in \mathbb{R}$, and $\{\varepsilon_k\}_k$ are zero-mean noises whose variances decay to zero at a rate of $(NT)^{-\varepsilon}$ for any $\varepsilon>0$.
\end{condition}
Condition \ref{cond:initialest} is mild. To elaborate this condition, we discuss the three terms on the right-hand-side (RHS) of \eqref{eq:init-beta-estimator} one by one:
\begin{itemize}
	\item The first term $a_0 \exp(-b_0 k)$ represents the main effect, which equals the oracle value $\beta(k)$. This term decays to zero exponentially fast under the exponential $\beta$-mixing assumption (see Condition 1 in Section 5.1 of the main paper);
	\item The second term measures the shared bias among all the initial estimators. It originates from the estimation of some common nuisance functions, such as $f$, $\pi_b$, and $\pi_b'$, in estimating the mixing coefficients;
	\item The last term denotes the mean-zero residual, capturing the variability of the initial estimators. 
\end{itemize}
This condition is not restrictive at all. First, it only requires the variances of the initial estimators to decay to zero; it does not impose any requirement for the bias $\eta_0$. Second, the variances can decay to zero at any arbitrary rate, thus enabling a wide range of nonparametric estimators.  

Additionally, we require that $\beta(K^*)$ must be strictly smaller than $\delta K^*/(NT)$, and  $\beta(K^*-1)$ must be strictly larger than $\delta K^*/(NT)$.
\begin{condition}\label{cond:Kstar}
	$NT\beta(K^*)/K^*$ is strictly smaller than $\delta$ and $NT\beta(K^*-1)/(K^*-1)$ is strictly larger than $\delta$. Specifically, $\delta-NT\beta(K^*)/K^*$ and $NT\beta(K^*-1)/(K^*-1)-\delta$ are bounded away from zero. 
\end{condition}

Given Conditions~\ref{cond:initialest} and \ref{cond:Kstar}, the theorem below establishes the consistency of $\widehat{K}$, followed by the corresponding proof.
\begin{theorem}\label{thm:Kconsistency}
	Under Conditions \ref{cond:initialest} and \ref{cond:Kstar}, we have $\widehat{K}=K^*$ with probability approaching~$1$. 
\end{theorem}
\begin{proof}
	We begin by noting the convexity of the least square objective function. Since $K_0>3$ and is a finite integer, its Hessian matrix is positive definite. Consequently, it is straightforward to show that both estimators $\widehat{a}$ and $\widehat{b}$ converge to their oracle values at a rate of $O_p((NT)^{-\varepsilon})$, under Condition \ref{cond:initialest}. 

	Next, we notice that under exponential $\beta$-mixing, both $\widehat{\beta}(k)/k$ and $\beta(k)/k$ are monotonically decreasing as functions of $k$. To establish the consistency of $\widehat{K}$, it suffices to show $\widehat{\beta}(K^*)/K^*\le \delta/(NT)$ and $\widehat{\beta}(K^*-1)/(K^*-1)> \delta/(NT)$. Under Condition \ref{cond:Kstar}, it suffices to show 
	\begin{eqnarray}\label{eqn:converge}
		\max\Big|\frac{\widehat{\beta}(K^*)}{K^*}-\frac{\beta(K^*)}{K^*},\frac{\widehat{\beta}(K^*-1)}{K^*-1}-\frac{\beta(K^*-1)}{K^*-1}\Big|=o_p\Big(\frac{1}{NT}\Big).
	\end{eqnarray}
	Equation \eqref{eqn:converge} poses a non-standard condition, as it requires the estimator to decay at a much faster rate than $O((NT)^{-1})$. Nevertheless, as we will demonstrate, this rate is achievable, by noting that both $\beta(K^*-1)$ and $\beta(K^*)$ are of the same order of magnitude as $O((NT)^{-1})$.

	For a given integer $k$, we have by Taylor expansion that
	\begin{eqnarray}\label{eqn:taylor1}
		\frac{\exp(-\widehat{b}k)-\exp(-b_0 k)}{k}=\sum_{t=1}^{\infty} \frac{(-k)^t\exp(-b_0 k)}{t! k} (\widehat{b}-b_0)^t.
	\end{eqnarray}
	Since $\beta(K^*-1)/(K^*-1) > \delta/(NT)$, $K^*$ is upper bounded by $O(\log (NT))$. This together with the convergence rate of $\widehat{b}$ implies that, when $k$ is set to either $K^*-1$ or $K^*$, \eqref{eqn:taylor1} is much smaller than
	\begin{eqnarray*}
		\sum_{t=1}^{\infty} \frac{\exp(-b_0 k)}{t! k},
	\end{eqnarray*}
	with probability approaching $1$. Additionally, due to Condition~\ref{cond:Kstar}, we have $\beta(K^*)/K^*=O((NT)^{-1})$; and thus, we obtain that \eqref{eqn:taylor1} is $o_p((NT)^{-1})$. Finally, notice that
	\begin{eqnarray*}
		\frac{\widehat{\beta}(k)}{k}-\frac{\beta(k)}{k}=\widehat{a} \frac{\exp(-\widehat{b}k)-\exp(-b_0 k)}{k}+\frac{(\widehat{a}-a_0)\exp(-b_0 k)}{k}.
	\end{eqnarray*}
	When setting $k$ to either $K^*-1$ or $K^*$, the first term on the RHS is $o_p((NT)^{-1})$. Meanwhile, the second term is $o_p((NT)^{-1})$ as well, since $\beta(K^*)/K^*=O((NT)^{-1})$. This proves \eqref{eqn:converge}, which in turn completes the proof of Theorem \ref{thm:Kconsistency}. 
\end{proof}

\subsubsection{Simulations}\label{sec:bestK-simulation}
To analytically calculate the ground truth for $K^*$, we consider a simplified where states, actions, and rewards are all binary. We notice that even for slightly more complicated linear auto-regressive models, calculating $K^*$ becomes much more challenging \citep{mcdonald2015estimating}. Consider the following time-homogeneous MDP with the transition probability given by:
\begin{align*}
	&
	\mathbb{P}\left((S_t, A_t, R_t) = (s', a', r') \given (S_{t-1}, A_{t-1}, R_{t-1}) = (s, a, r)\right) 
	\\
	=& 
	\begin{cases}
		0, \;& \textup{ if } s=s', a=a' \textup{ and } r = r'
		\\
		\displaystyle \frac{1}{7}, \;& \textup{ otherwise },
	\end{cases}
\end{align*}
for any $t$. 
Through detailed calculations, we can show that the beta-mixing coefficient under this configuration is:
\begin{align*}
	\beta(q) = \frac{7}{8} \times 7^{-q}.
\end{align*}
This expression enables us to derive $K^*$, which additionally depends on $NT$ and $\delta$. 

We adopt the procedure in Section 4.2 to estimate $K^*$. 
The initial estimators $\widehat{\beta}(1)$, $\ldots$, $\widehat{\beta}(K_0)$ 
were constructed using \citet{mcdonald2015estimating}. 
We visualize both the relative bias and the relative mean squared error (MSE) of $\widehat{K}$ in Figure~\ref{fig:k-selection-consistency}. It can seen that, for sufficiently large $NT$, both metrics decay to zero, which implies the consistency of $\widehat{K}$. Moreover, the consistency holds across all values of $\delta$. 

\begin{figure}[t]
	\centering
	\includegraphics[width=0.8\linewidth]{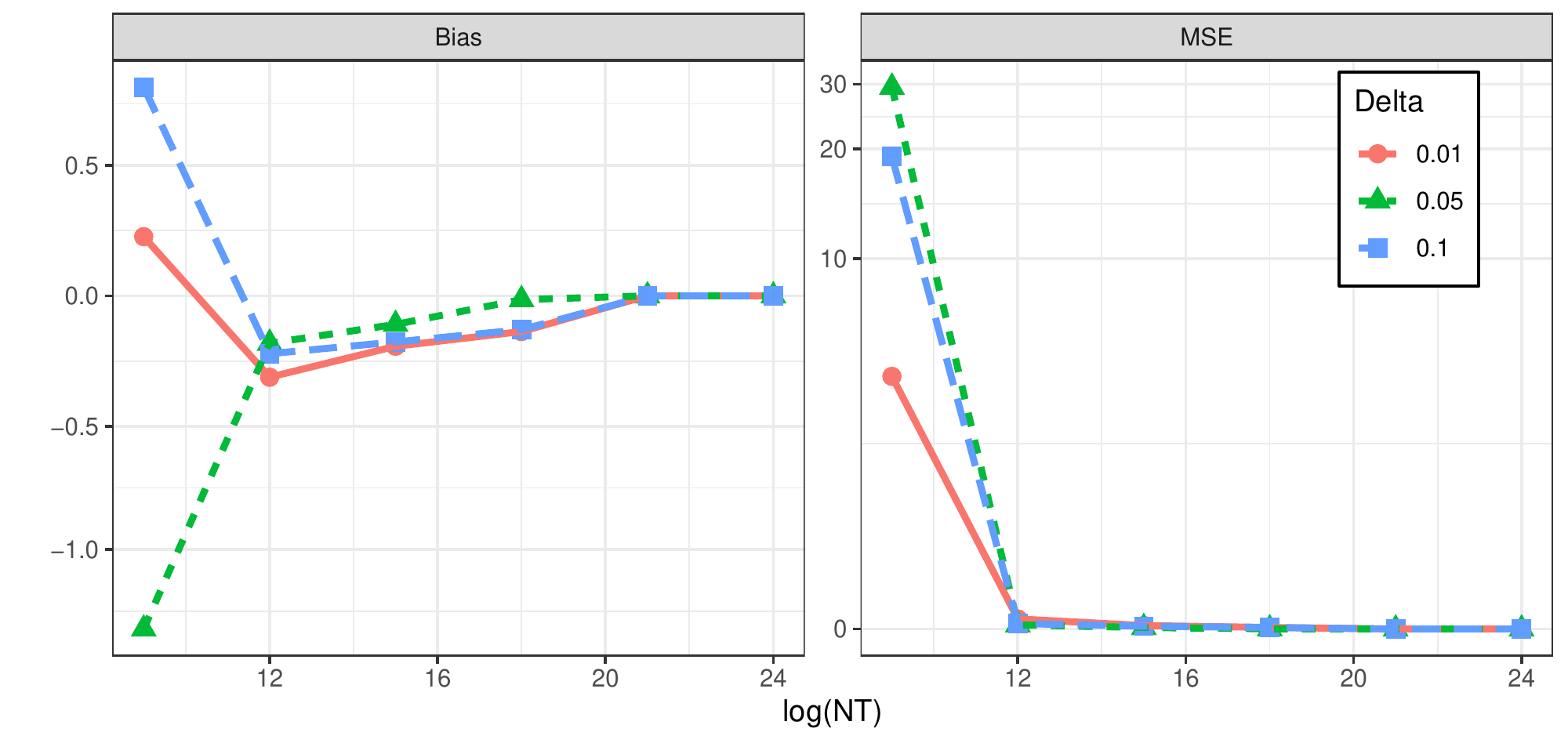}
	\caption{The $x$-axis denotes the value of $\log_2 (NT)$. The relative bias computes $(\widehat{K} - K)/K$ whereas the relative MSE calculates $(\widehat{K} - K)^2/K^2$ over 20 replications.}\label{fig:k-selection-consistency}
\end{figure}
\textit{\textbf{Sensitivity analysis}}. We further conduct a sensitivity analysis by varying (i) $K_0$, (ii) the number of Gaussian components, and (iii) the parametric form of the mixing parameter. Specifically for (iii), we consider a scenario where $\beta(k)$ decays polynomially fast with $k$, i.e., $\beta(k) = a_0 \times b_0^{-k}$, and estimate $a_0$, $b_0$ using a procedure similar to that under the exponential $\beta$-mixing. 

\begin{figure}[t] 
    \centering
    \includegraphics[width=\linewidth]{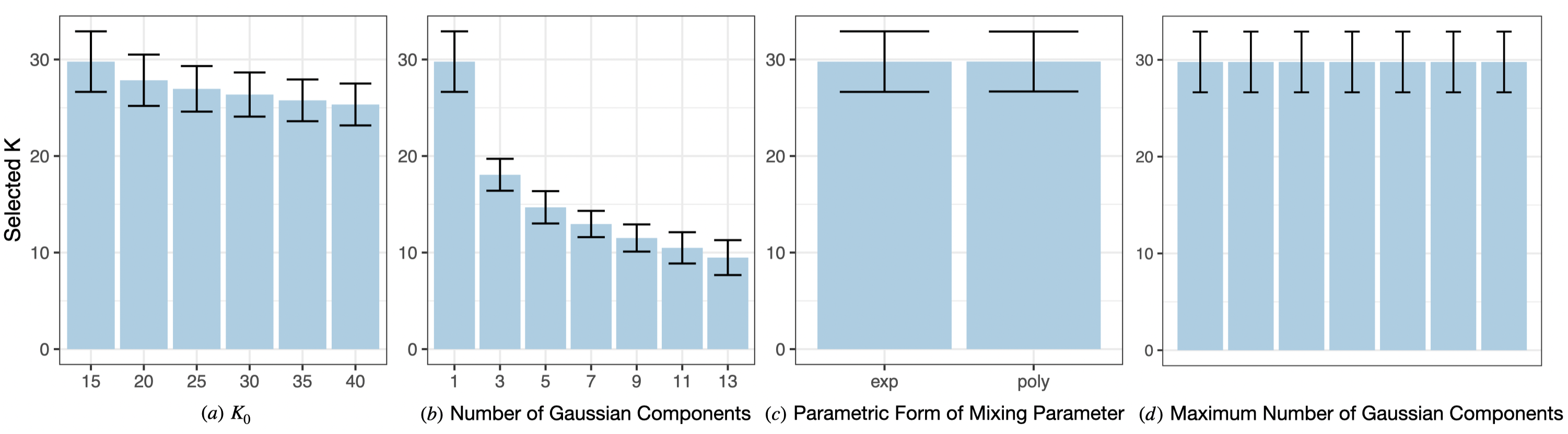}
    \caption{The bar plot for the selected $K$ where the error bar indicates the 95\% confidence interval. The experiments are conducted on the AR environment with 100 replications. }
    \label{fig:select-K-sensitive}
\end{figure}

The results are visualized in Figure~\ref{fig:select-K-sensitive}. 
From Figure \ref{fig:select-K-sensitive}(a), 
it is evident that increasing $K_0$ from 15 (the default value in our empirical studies) to 40 does not significantly affect the selected $K$. Furthermore, Figure \ref{fig:select-K-sensitive}(c) shows that the results based on the two parametric forms are almost indistinguishable. 
This confirms that the selection of $K$ is not sensitive to either the change in $K_0$ or the parametric form of the mixing parameter. 

On the other hand, it can be seen from Figure \ref{fig:select-K-sensitive}(b) that the number of Gaussian components, has a visible impact on the selection of $K$, with $K$ monotonically decreasing as this number increases. To address this sensitivity, we suggest using the Bayesian information criterion to select the number of Gaussian components \citep{mclachlan2014number}. The results presented in Figure~\ref{fig:select-K-sensitive}(d) show that the selected $K$ remains stable when the maximum number of Gaussian components varies from 1 to 13, and aligns with the number of splits selected in our numerical studies.

\subsection{Second-order Machine for Knockoff Construction}\label{sec:secondorder}
\label{apdx:b.3}
For each $a \in \mathcal{A}$, recall that $\calD_k^{(a)}$ denotes the data subset $\left\{(\mathbf{S}, A, Y)\in\calD_k:A=a\right\}$. 
According to Line 2 of Algorithm \ref{alg:2}, on each $\calD_k^{(a)}$, we aim to construct a knockoff state $\widetilde{\svar}$ such that
\begin{align}\label{eqn:satisfyexchangeable}
	\widetilde{\bm{S}}\indep R|\bm{S},A \,\,\hbox{and}\,\, (\svar, ~ \widetilde{\svar})_{\text{swap}(B)}|A \overset{d}{=} (\svar, ~ \widetilde{\svar})|A.
\end{align}
To satisfy 
\eqref{eqn:satisfyexchangeable}, 
we assume each $\svar|A=a$ follows a multivariate normal distribution with mean $\bm{\mu}^{(a)}$ and covariance matrix $\bm{\Sigma}^{(a)}$. For now, assumes $\widetilde{\svar}$ follows the following multivariate normal distribution, 
\begin{align}\label{eqn:gaussiansampling}
	\widetilde{\svar}|\svar,A=a \overset{d}{=}
	\mathcal{N}(\bm{\nu}^{(a)}, \bm{V}^{(a)}),
\end{align}
where 
\begin{align*}
	\bm{\nu}^{(a)} &= \svar^{(a)} -  (\bm{\Sigma}^{(a)})^{-1} \text{diag}(\mathbf{d}^{(a)})(\svar^{(a)}-\bm{\mu}^{(a)}),\\
	\bm{V}^{(a)} &= 2\text{diag}(\mathbf{d}^{(a)}) - \text{diag}(\mathbf{d}^{(a)}) (\bm{\Sigma}^{(a)})^{-1} \text{diag}(\mathbf{d}^{(a)}),
\end{align*}
for some $\mathbf{d}^{(a)}\in \mathbb{R}^d$. Notice that the conditional Gaussian distribution specified in \eqref{eqn:gaussiansampling} depends solely on $\svar$ and $A$ only, and is independent of the reward. This guarantees the conditional independence between the knockoff state and the reward in \eqref{eqn:satisfyexchangeable}. 

Meanwhile, with some calculations, it is straightforward to show that such a construction leads to the following joint distribution:
	\begin{align*}
		(\mathbf{S}, \widetilde{\mathbf{S}})|A=a \sim
		\mathcal{N}((\bm{\mu}^{(a)}, \bm{\mu}^{(a)}), \mathbf{U}^{(a)}),
		\mathbf{U}^{(a)}=
		\begin{bmatrix}
			\bm{\Sigma}^{(a)} & \bm{\Sigma}^{(a)} - \text{diag}(\mathbf{d}^{(a)})\\
			\bm{\Sigma}^{(a)} - \text{diag}(\mathbf{d}^{(a)})& \bm{\Sigma}^{(a)}
		\end{bmatrix}.
	\end{align*}
By the symmetry of $\mathbf{U}^{(a)}$, switching any entries of $(S_{j_1},S_{j_2},\cdots,S_{j_{\kappa}})$ and $(\widetilde{S}_{j_1},\widetilde{S}_{j_2},\cdots,
\widetilde{S}_{j_{\kappa}})$ for any $\kappa$ and $j_1,\cdots,j_\kappa$ results in the same covariance matrix. Since the $\mathbf{S}$ and $\widetilde{\mathbf{S}}$ share the same mean, the exchangeability assumption (refer to Equation  \eqref{eqn:satisfyexchangeable}) is satisfied. This motivates us to employ the Gaussian sampling in \eqref{eqn:gaussiansampling} to construct the knockoff variables.

It remains to determine the diagonal elements $\mathbf{d}^{(a)}$ in $\bm{V}^{(a)}$. On one hand, these elements shall be chosen such that $\mathbf{U}^{(a)}$ is positive semidefinite 
 with all entries being nonnegative. On the other hand, each entry in $\mathbf{d}^{(a)}$ should be as large as possible to reduce the correlation between $\svar$ and $\widetilde{\svar}$ conditional on $A=a$. To balance this trade-off, one could adopt the approximate semidefinite program 
	developed in \citet[][]{candes2018panning} to determine the optimal $\mathbf{d}^{(a)}$.

Finally, we remark that the construction of the conditional Gaussian distribution in \eqref{eqn:gaussiansampling} ensures that the first two moments of $\widetilde{\mathbf{S}}^{(a)}$ match those of $\mathbf{S}^{(a)}$.
Thus, we refer to this method as a second-order (sampling) machine. In practice, it works well even without the 
normality assumption.

\subsection{Feature Importance Statistics}\label{sec:genericML}
\subsubsection{LASSO}\label{sec:importance-lasso}
In this section, we detail the construction of feature importance statistics via LASSO. For each $k$, we apply LASSO to the augmented data subset $\widetilde{\calD}_k^{(a)}$ (see Step 3 of Algorithm \ref{alg:2}) with both the observed states and their knockoff variables serving as the ``predictors'', and the rewards or the next states serving as the ``response'' (i.e., LASSO for each response variable). Let $\widehat{\bm{b}}^{(a)}\in \mathbb{R}^{2d}$ denote the regression coefficient vector. We next assign $Z_j^{(a)}$ and $\widetilde{Z}_j^{(a)}$ to the absolute values of the $j$th and $(j+d)$th elements of $\widehat{\bm{b}}^{(a)}$. Finally, we set $Z_j=\max_a Z_j^{(a)}$ and $\widetilde{Z}_j=\max_a \widetilde{Z}_j^{(a)}$.

\subsubsection{Generic Machine Learning Methods}\label{sec:importance-ML}
In this section, we focus on the setting where a generic machine learning method is applied to $\widetilde{\calD}_k^{(a)}$ to calculate the variable importance (VI) score for each state as well as its knockoff variable. 
When the random forest algorithm is applied, it will simultaneously produce an importance measure for each variable. Alternatively, deep neural networks are applicable to produce these measures as well \citep{lu2018deeppink}. To handle multiple actions, we calculate the VI based on each data subset  $\widetilde{\calD}_k^{(a)}$ and obtain the maximum value over all actions. Let $\textrm{VI}^{(a,k)}_{j,i}$ denote the $j$th variable's VI with the $i$th state being the outcome. Note that when $i=d+1$, the outcome is the reward. Then we can define the $W$-statistic 
    $W_j^{(k)} = \max\limits_{a\in\mathcal{A}}|VI_{j,i}^{(a,k)}| - \max\limits_{a\in\mathcal{A}}|VI_{j+d,i}^{(a,k)}|$. 


To analyze the power property, similar to Section \ref{sec:poweranalysis}, we construct a $(d+1)\times (d+1)$ matrix $\mathbf{V}^{(k)}$ whose $(j,i)$th entry equals $\max\limits_{a \in \aspace} |\textrm{VI}_{j,i}^{(a,k)}|$ for $j\le d$ and $0$ for $j=d+1$. This generates a directed graph $\mathcal{G}(\mathbf{V}^{(k)})$. 


\begin{condition} [Separation] \label{ass:sep}
There exists a function $\gamma(n,d)$ 
such that for any $k\in[K]$, with probability approaching $1$, we have 
\begin{enumerate}
    \item $\max\limits_{a\in\mathcal{A}}|\textrm{VI}_{j,i}^{(a,k)}| 
    < \gamma(n,d)/2$
    uniformly for any $j\in[2d]\backslash G_\textrm{M}$ and any $i\in [d+1]$; 
    \item
    the subgraph of $\mathcal{G}(\mathbf{V}^{(k)})$, by removing edges with weights smaller than $\gamma(n,p)$, identifies $G_{\textrm{M}}$.
\end{enumerate}
\end{condition}

\begin{remark}
    The separation assumption is 
    mild in 
    the sense that we do not require $\textrm{VI}_{j,i}^{(a,k)}$'s to converge to some population-level variable importance measures as the sample size increases. The first part of Condition \ref{ass:sep} essentially upper bounds the importance of the insignificant variables whereas the second part corresponds to the signal strength condition. Similar to Condition \ref{ass:signal2}, it only requires the existence of one path from the reward to each minimal sufficient state after removing edges with weak or moderate signals. 
\end{remark}

\begin{theorem}\label{thm:power-ml}
Suppose Condition \ref{ass:sep} holds. Then with probability approaching $1$, $\widehat{G}$ produced by Algorithm \ref{alg:2} contains $G_M$. 
\end{theorem}
The proof of Theorem \ref{thm:power-ml} is similar to that of Theorem \ref{thm:type-II} and is hence omitted for brevity.

\subsection{Reward-only and One-step Methods: Pesudo Code}\label{sec:algo-rewardonly-onestep}
This section provides pseudocode of the reward-only method (see Algorithm~\ref{alg:1}) and the one-step method (see Algorithm~\ref{alg:2}) that were set as benchmarked methods in Sections~\ref{sec:exp} and~\ref{sec:mimic3}. From Algorithms~\ref{alg:reward-only} and~\ref{alg:one-step}, the implementations of reward-only and one-step methods ensure fair comparison for variable selection as we have applied the same data splitting step, knockoff filter, and majority voting step as used in the SEEK algorithm.

\renewcommand\thealgorithm{A\arabic{algorithm}}
\renewcommand{\algorithmicrequire}{\textbf{Input:}}
\renewcommand{\algorithmicensure}{\textbf{Output:}}
\begin{algorithm}[H]
\caption{Reward-only method}\label{alg:reward-only}
\begin{algorithmic}[1]
    \Require Batch data $\calD$, number of data subsets $K$, a target FDR level $q \in (0, 0.5)$, and a threshold $\alpha\in(0, 1)$ for the majority vote.
    \State Split $\calD$ into $K$ non-overlapping sets $\calD_1, \ldots, \calD_K$ that have the same cardinality. 
    \For {$k=1,2,\ldots,K$}
    \State Apply Algorithm~\ref{alg:2} to all $(\mathbf{S}, A, R)$ tuples in $\calD_k$ with the target FDR level $q$, and denote the selected index set as $\widehat{G}_{k}$. 
    \EndFor
    \Ensure $\widehat{G} := \left\{j \in \{1,\ldots,d\}: \sum\limits_{1\leq k\leq K}\mathbb{I}(j \in \widehat{G}_k) \geq \alpha K \right\}.$
\end{algorithmic}
\end{algorithm}
\begin{algorithm}[H]
\caption{One-step method}\label{alg:one-step}
\begin{algorithmic}[1]
    \Require Batch data $\calD$, number of data subsets $K$, a target FDR level $q \in (0, 0.5)$, and a threshold $\alpha\in(0, 1)$ for the majority vote.
    \State Split $\calD$ into $K$ non-overlapping sets $\calD_1, \ldots, \calD_K$ that have the same cardinality.
    
    \For {$k=1,2,\ldots,K$}
    \For {$j=1, 2, \ldots, d$}
    \State Apply Algorithm \ref{alg:2} to all $(\mathbf{S}, A, S_j)$ tuples in $\calD_k$ with the target FDR level $q$. Denote the selected index set by $\widehat{G}_{k, j}$.
    \EndFor
    \State Apply Algorithm \ref{alg:2} to all $(\mathbf{S}, A, R)$ tuples in $\calD_k$ with the target FDR level $q$.  Denote the selected index set by $\widehat{G}_{k, d+1}$.
    \State $\widehat{G}_k \leftarrow \bigcup_{j=1}^{d+1} \widehat{G}_{k, j}$
    \EndFor
    \Ensure $\widehat{G} := \left\{j \in \{1,\ldots,d\}: \sum\limits_{1\leq k\leq K}\mathbb{I}(j \in \widehat{G}_k) \geq \alpha K \right\}.$
\end{algorithmic}
\end{algorithm}

\section{Proofs}\label{apdx:proof}

\subsection{Propositions and Their Proofs}\label{sec:prop-proof}
This section is organized as follows. We begin by providing proofs for Propositions \ref{prop1} through \ref{prop4} in Sections \ref{sec:proofprop1} to \ref{sec:proofprop34}. Next, Sections \ref{sec:policy-learning} and \ref{sec:policy-evaluation} introduce two additional propositions that illustrate the benefits of variable selection for policy learning and evaluation. And finally, Section~\ref{sec:action-independence} shows the minimal sufficient state is independent to the action value. 

\subsubsection{Proof of Proposition \ref{prop1}}\label{sec:proofprop1}
We aim to show that for any policy $\pi$ that depends on the state $\svar$ only through $\svar_{G}$, its Q-function $Q^{\pi}$ is solely a function of $\svar_{G}$ as well. Assuming this is established, consider an arbitrary policy $\pi_0$ that depends on the state $\svar$ only through $\svar_{G}$. When we apply policy iteration to compute the optimal policy, the initial iteration evaluates its Q-function $Q^{\pi_0}$, which is solely a function of $\svar_{G}$. Subsequently, applying policy improvement yields its greedy policy, denoted by $\pi_1(a\given s) \coloneqq \arg\max\limits_{a} Q^{\pi_0}(s, a)$, which remains a function of $\svar_{G}$ as well. Repeating this procedure, we use induction to demonstrate that at each iteration, both the Q-function obtained via policy evaluation and the greedy policy obtained via policy improvement are functions of $\svar_{G}$ only. As the number of iterations increases to infinity, the sequence of greedy policies $\pi_1, \pi_2, \ldots, \pi_t, \ldots$ converges to the optimal policy \citep{puterman2014markov}, which also depends exclusively on $\svar_{G}$. 

It remains to show $Q^{\pi}$ is solely a function of $\svar_{G}$ for any $\svar_G$-dependent policy $\pi$. By definition, it suffices to show  $\E^{\pi}(R_t | \mathbf{S}_0, A_0)$ depends on $\mathbf{S}_0$ only through $\mathbf{S}_{0,G}$ for any $t$. Since the reward function is conditionally independent of the history given the current state-action pair, we obtain \useshortskip
\begin{eqnarray}\label{eqn:prop1}
    \E^{\pi}(R_t | \mathbf{S}_0, A_0)=\E^{\pi}[\rfun(\mathbf{S}_t,A_t) | \mathbf{S}_0, A_0]=\sum_a \E^{\pi}[\rfun(\mathbf{S}_t,a)\pi(a|\mathbf{S}_{t}) | \mathbf{S}_0, A_0].
\end{eqnarray}\vspace{-4em}

\noindent According to the definition of the sufficient state, $\rfun(\mathbf{S}_t,a)$ can be represented as a function of $\mathbf{S}_{t,G}$ and $a$. Since $\pi$ is $\svar_G$-dependent, $\rfun(\mathbf{S}_t,a)\pi(a|\mathbf{S}_{t})$ is a function of $\mathbf{S}_{t,G}$ only. 

It suffices to show $\mathbf{S}_{t,G}\indep \mathbf{S}_0|\mathbf{S}_{0,G}, A_0$. Such a conditional independence property can be proven by induction. Specifically, when $t=1$, it is automatically satisfied according to the definition of the minimal sufficient state. Suppose the assertion holds for $t=k$, it follows by the stationarity of the MDP that $\mathbf{S}_{k+1,G}\indep \mathbf{S}_1|\mathbf{S}_{1,G}, A_1$. Next, together with Markov property which implies that $\mathbf{S}_{k+1,G}\indep \mathbf{S}_0,A_0|\mathbf{S}_1,A_1$, we obtain $\mathbf{S}_{k+1,G}\indep \mathbf{S}_0,A_0|\mathbf{S}_{1,G},A_1$ by the contraction rule of conditional independence\footnote{\url{https://en.wikipedia.org/wiki/Conditional_independence}}. Under $\pi$, $A_1$ becomes conditionally independent of $\mathbf{S}_0,A_0$ given $\mathbf{S}_{1,G}$ as well. Using contraction again yields the conditional independence between $\mathbf{S}_{k+1,G}$ and $\mathbf{S}_0,A_0$ given $\mathbf{S}_{1,G}$. Using the weak union rule, we obtain  $\mathbf{S}_{k+1,G}\indep \mathbf{S}_0|\mathbf{S}_{1,G}, A_0,\mathbf{S}_{0,G}$. 
Finally, this together with $\mathbf{S}_{1,G}\indep \mathbf{S}_0|\mathbf{S}_{0,G}, A_0$ and the contraction rule proves the conditional independence between $\mathbf{S}_{k+1,G}$ and $\mathbf{S}_0$ given $\mathbf{S}_{1,G}$ and $A_0$. By induction, $\mathbf{S}_{t,G}\indep \mathbf{S}_0|\mathbf{S}_{0,G}, A_0$ for any $t\ge 1$. The proof is hence completed.

\subsubsection{Proof of Proposition~\ref{prop2}}
To show the reduced process remains an MDP, it suffices to show (i) the reward function in the reduced process is stationary and depends on the history only through the current state-action pair; (ii) the future state in the reduced MDP satisfies the Markov assumption, being conditionally independent of the history given the current state-action pair; (iii) the Markov transition function in the reduced process remains stationary. Notice that (i) is immediate to see as the reward function in the original process depends on the state only through the sufficient state. To prove (ii), it suffices to show 
\begin{eqnarray*}
	\mathbf{S}_{t+1,G}\indep 
	\{\mathbf{S}_{j,G},A_j,R_j\}_{0 \leq j < t}
	|\mathbf{S}_{t,G},A_t.
\end{eqnarray*}
To prove the conditional independence, we first notice that, it follows from the Markov property of the full model $\mathcal{M}$ that
\begin{align*}
	\mathbf{S}_{t+1,G}\indep 
	\{\mathbf{S}_{j,G},\mathbf{S}_{j,G^c},A_j,R_j\}_{0 \leq j < t}
	|\mathbf{S}_{t,G},\mathbf{S}_{t,G^c},A_t.
\end{align*}
Then by the decomposition rule of conditional independence\footnote{\url{https://en.wikipedia.org/wiki/Conditional_independence}},
\begin{align*}
	\mathbf{S}_{t+1,G}\indep 
	\{\mathbf{S}_{j,G},A_j,R_j\}_{0 \leq j < t}
	|\mathbf{S}_{t,G},\mathbf{S}_{t,G^c},A_t.
\end{align*}
On the other hand, by the definition of sufficient state, we have that
\begin{align*}
	\mathbf{S}_{t+1,G}\indep 
	\mathbf{S}_{t,G^c}|\mathbf{S}_{t,G},A_t.
\end{align*}
Finally, by combining the above two observations with the contraction rule of conditional independence, we have that
\begin{align*}
	\mathbf{S}_{t+1,G}\indep 
	\{\mathbf{S}_{j,G},A_j,R_j\}_{0 \leq j < t},\mathbf{S}_{t,G^c}
	|\mathbf{S}_{t,G},A_t.
\end{align*}
Again applying the decomposition rule yields the desired  Markov property. 

Finally, due to the time-homogeneity of the Markov transition function in the original MDP $\mathcal{M}$, the reduced process also possesses a time-homogeneous transition function. Thus, Proposition~\ref{prop2} hold.

\subsubsection{Proofs of Propositions \ref{prop3} and \ref{prop4}}\label{sec:proofprop34} 
\textbf{Part 1 of the proof of Proposition \ref{prop3} (existence)}. Denote the set of all index sets for sufficient state as $\mathbb{G}_\text{sf}$. First we know from the definition, $[d]\in\mathbb{G}_\text{sf}$, meaning that $\mathbb{G}_\text{sf}$ is nonempty. In addition, for any $G_1,G_2\in\mathbb{G}_\text{sf}$, we can define an ordering $\leq_\text{sf}$ by $G_1\leq_\text{sf}G_2$ iff $|G_1|\leq|G_2|$. Then we define the minimal sufficient state by
\begin{align*}
	G^\ast = \mathop{\textup{argmin}}\limits_{G\in\mathbb{G}_\text{sf}}|G|.
\end{align*}
It always exists by definition. If we further assume the strict positivity of the transition kernel,$G^\ast$ is unique, as shown below. 

\textbf{Part 2 of the proof of Proposition \ref{prop3} (uniqueness) and proof of Proposition \ref{prop4}}. We first discuss the 3 methods one by one to prove Proposition \ref{prop4}.

\textit{\textbf{Reward-only approach}}: 
The toy example described in Section \ref{sec:minimal} demonstrates the insufficiency of the state selected by the reward-only approach. Specifically, in this example, the reward function is determined exclusively by the first state variable, whereas the transition function of the first two state variables depends only on themselves. This establishes that that $\mathbf{S}_{G_{\textrm{M}}}=\mathbf{S}_{\{1,2\}}$ is a sufficient state. 
Furthermore, 
we notice for any index set $G$, if $1\notin G$, then Markov property for the conditional mean function of the reward will be violated. If $2\notin G$, 
the Markov property for the next state will be violated. Then for any $G\in\mathbb{G}_\text{sf}$, we must have $1,2\in G$. This leads to our conclusion that $\mathbf{S}_{G_{\textrm{M}}}$ is the unique minimal sufficient state in this scenario. 
As for the reward-only approach, it only selects the first state variable and is thus insufficient. 

\textit{\textbf{One-step approach}}: Any such method would select a subset $G$ such that
\begin{align*}
	(R_t,\mathbf{S}_{t+1,G})\indep \mathbf{S}_{t, G^c}|(\mathbf{S}_{t,G},A_t).
\end{align*}
It follows from the decomposition rule of conditional independence that
\begin{align*}
	\mathbb{E}(R_t \given \mathbf{S}_{t},A_t) = \mathbb{E}(R_t \given \mathbf{S}_{t, G},A_t)  \textup{ and }
	\mathbf{S}_{t+1,G}
	&\indep \mathbf{S}_{t, G^c}|(\mathbf{S}_{t,G},A_t).
\end{align*}
As a result, $\mathbf{S}_{G}$ is a sufficient state.

However, the selected subset is not guaranteed to be minimally sufficient. Again, consider the toy example in Section \ref{sec:minimal}. 
Any one-step method would select $G_{\textrm{M}}\cup G_{\text{AR}}$, which includes some redundant variables $G_{\text{AR}}$.

\textit{\textbf{Iterative approach}}: For a given MDP $\mathcal{M}$, the iterative method proceeds as follows. First, all state variables directly contributing to reward function are selected, whose index set is denoted by $G_R$ (the uniqueness of $G_R$ will be proven later), such that for any $t$,
\begin{align}\label{eqn:sigreward}
	\mathbb{E}(R_t \given \mathbf{S}_t,A_t) = \mathbb{E}(R_t \given \mathbf{S}_{t,G_R},A_t).
\end{align}
Next, in the second step we select those states (among $\{G_R^c\}$) that contribute to the state transition of of $\mathbf{S}_{G_R}$, denoted as $G_1$, such that
\begin{align}\label{eqn:sigfiststage}
	\mathbf{S}_{t+1,G_R}
	\indep \mathbf{S}_{t, G_1^c}|(\mathbf{S}_{t,G_1},A_t),
\end{align}
and further, by the weak union rule, we have
\begin{align*}
	\mathbf{S}_{t+1,G_R}
	\indep \mathbf{S}_{t, (G_R\cup G_1)^c}|(\mathbf{S}_{t,G_R},\mathbf{S}_{t,G_1},A_t).
\end{align*}
Similarly, in the third step we select those among $\{G_R^c\cap G_1^c\}$ 
that contribute to the state transition function of $(\mathbf{S}_{G_R},\mathbf{S}_{G_1})$, denoted as $G_2$, in the sense that
\begin{align*}
	(\mathbf{S}_{t+1,G_R},\mathbf{S}_{t+1,G_1})
	\indep \mathbf{S}_{t, (G_R\cup G_1\cup G_2)^c}|(\mathbf{S}_{t,G_R},\mathbf{S}_{t,G_1},\mathbf{S}_{t,G_2},A_t).
\end{align*}
We 
repeat the iterations until no more state can be selected. 
The number of iterations is always finite, and is upper bounded by $d$. 
Let $m$ denote the total number of iterations minus 1. The selected subset is given by
$G_\text{iter} = (\cup_{1\leq j \leq m}G_j) \cup G_R$. 
According to the selection procedure, we have for $1\leq j \leq m-1$ that
\begin{align*}
	(\mathbf{S}_{t+1,G_R},\mathbf{S}_{t+1,G_1},\ldots,\mathbf{S}_{t+1,G_j})
	\indep \mathbf{S}_{t, (G_R\cup G_1\cup\ldots\cup G_{j+1})^c}|(\mathbf{S}_{t,G_R},\mathbf{S}_{t,G_1},\ldots,\mathbf{S}_{t,G_{j+1}},A_t).
\end{align*}
Since the procedure stops after $m+1$ iterations, 
it implies that the following termination condition is met
\begin{align*}
	\mathbf{S}_{t+1,G_\text{iter}}
	\indep \mathbf{S}_{t, G_\text{iter}^c}|(\mathbf{S}_{t,G_\text{iter}},A_t).
\end{align*}
In addition, by \eqref{eqn:sigreward} and the weak union rule, 
\begin{align*}
	\mathbb{E}(R_t \given \mathbf{S}_{t},A_t) = \mathbb{E}(R_t \given \mathbf{S}_{t, G_\textup{iter}},A_t)  \textup{ and }
	\mathbf{S}_{t+1,G_{\textup{iter}}}
	&\indep \mathbf{S}_{t, G_{\textup{iter}}^c}|(\mathbf{S}_{t,G_{\textup{iter}}},A_t).
\end{align*}
Therefore, $\mathbf{S}_{G_\text{iter}}$ is a sufficient state. We emphasize that the sufficiency relies on the consistency of the selection algorithm. This condition is imposed in the statement of Proposition~\ref{prop4}.

We next prove that when the transition probability is strictly positive, each subset $G_R$, $G_1,\cdots,G_{\textrm{M}}$ is uniquely defined. This implies that the iterative method will output a unique subset $\mathbf{S}_{G_\text{iter}}$. For convenience, we focus on the case when $\mathcal{S}$ is finite. Meanwhile, the proof is applicable to the continuous state space setting as well, by replacing the transition probability mass function with the probability density function. We emphasize that without the positivity assumption, the selected subset may not be unique or minimally sufficient. 

We begin by considering the second iteration where $\mathbf{S}_{t+1, G_R}$ is the target and prove the uniqueness of $G_1$ by contraction. 
For any index set $\Lambda\subset[d]$ that is different from $G_1$, such that $\mathbf{S}_\Lambda$ satisfies the conditional independence assumption in \eqref{eqn:sigreward}:
\begin{eqnarray}\label{eqn:sigreward0}
\mathbf{S}_{t+1, G_R}\indep \mathbf{S}_{t, \Lambda^c}|(\mathbf{S}_{t,\Lambda},A_t).
\end{eqnarray}
We aim to prove $G_1\subseteq \Lambda$. 

We prove this by contradiction. Suppose $G_1 - \Lambda \neq \emptyset$. Let $\mathcal{A}_t(s)$ denote the support of $A_t$ conditional on $\mathbf{S}_t=\mathbf{s}$. Since the transition function is strictly positive, for any $s\in \mathcal{S}, t\ge 1$ and any $a\in \mathcal{A}_t(s)$, the probability $\mathbb{P}(\mathbf{S}_t=\mathbf{s}, A_t=a)$ is strictly positive as well; 
and thus, we can apply the intersection rule of conditional independence \citep{drton2009lectures,pearl2009causality} on \eqref{eqn:sigfiststage} and \eqref{eqn:sigreward0} and obtain:
\begin{eqnarray}\label{eqn:sigreward2}
    \mathbf{S}_{t+1, G_R} \indep \mathbf{S}_{t, \Lambda^c \cup G_R^c}|(\mathbf{S}_{t,\Lambda\cap G_R},A_t).
\end{eqnarray}
By the assumption $G_1-\Lambda \neq \emptyset$, there exists some $j_0\notin \Lambda \cap G_1$ such that $j_0 \in G_1$, and then $j_0\in\Lambda^c \cup G_1^c$.
This together with \eqref{eqn:sigreward2} and the weak union rule yields that
\begin{eqnarray*}
    \mathbf{S}_{t+1, G_R}\indep S_{t,j_0}|(\mathbf{S}_{t,\{j_0\}^c},A_t).
\end{eqnarray*}
However, it contradicts the definition of $G_1$, which implies $j \notin G_1$ if and only if 
\begin{eqnarray}\label{eqn:sigrewardj}
    \mathbf{S}_{t+1,G_R} \indep S_{t,j}|(\mathbf{S}_{t, \{j\}^c},A_t).
\end{eqnarray} 
As such, we must have $G_1 \subseteq \Lambda$. Consequently, $G_1$ is the intersection of all the sets $\Lambda$ satisfying \eqref{eqn:sigreward0} and is hence uniquely defined. Similarly, we can show that, for $j \in \{2, \ldots, M\}$, the subset $G_j$ is unique. Following a similar procedure but replacing conditional independence with conditional mean independence, we can obtain the same conclusion for $G_R$, i.e., $G_R$ is the unique set that satisfies $\mathbb{E}(R_t \given \mathbf{S}_t, A_t) = \mathbb{E}(R_t \given \mathbf{S}_{t, G_R}, A_t)$. The uniqueness of $G_{\text{iter}}$ is thus proven.

Finally, we show that $\mathbf{S}_{G_\text{iter}}$ is minimally sufficient. For any sufficient state $\mathbf{S}_G$, we first claim that $G$ contains $G_1$. Otherwise, there exists some $j_0\notin G\cap G_1$ that satisfies $j_0\in G_1$ and \eqref{eqn:sigrewardj}. However, this is impossible based on the above arguments. Similarly, we can iteratively show that $G$ must contain $G_R, G_2, G_3, \ldots, G_M$. The proof is hence completed.

\subsubsection{Propositions Regarding Policy Learning}\label{sec:policy-learning}
We consider learning the optimal policy via fitted Q iteration (FQI). To simplify the analysis, we use a linear function class $\mathcal{Q}=\{\bs^\top \bm{w}_a: \|\bm{w}_a\|_2 \leq B, \bm{w}_a\in\mathbb{R}^d,\forall a \}$ to fit the Q-function at each iteration. Let $\mathcal{B}$ denote the Bellman optimality operator, i.e., for any $Q(\bs,a)$, $\mathcal{B} Q(\bs,a)=r(\bs,a)+\gamma \mathbb{E} [\max\limits_{a'\in \mathcal{A}} Q(\mathbf{S}_{t+1},a') \given \mathbf{S}_t=\bs,A_t=a]$. We analyze the regret of the resulting estimated optimal policy in the following proposition. 
\begin{proposition}\label{prop:regret} 
	Support $\mathcal{Q}$ is Bellman complete, i.e., $\B \mathcal{Q}=\mathcal{Q}$, the state process is stationary and the minimal eigenvalue of $\Mean [\pi_b(a|\mathbf{S}_t)(\mathbf{S}_t \mathbf{S}_t^\top)]$ is larger than $\lambda_0$ for any $a \in \mathcal{A}$, then the regret of the estimated optimal policy at the $K$th iteration is upper bounded by
	\begin{equation}\label{eq:regret}
			c_1 \max\left\{ \frac{d^{2c_2+1} B^{2c_2}}{(1-\gamma)^{2c_2+1} \lambda_0^{2c_2}} n^{-c_2} , \frac{d}{(1-\gamma)^2}\exp\left(-c_3 \frac{(1-\gamma)\lambda_0}{dB}n\right) \right\},
	\end{equation}
	when
	\begin{equation}\label{eq:fqi-iteration-number}
		K \geq \frac{\log(\lambda_0^{2m} n / (72d)^2)}{2\log(1/ \gamma)},
	\end{equation}
	where $n$ denotes the total number of observations ($NT$), $m$ denotes the number of actions, and $c_1,c_2,c_3$ denotes some positive constants. 
\end{proposition}
Proposition \ref{prop:regret} is derived from Corollary 3 of \citet{hu2024fast}. 
Equation~\eqref{eq:regret} reveals that, the upper bound of the regret increases as a polynomial function of the state dimension $d$. 
Consequently, using all state variables as opposed to the proposed minimal sufficient state would incur a larger regret for learning the optimal policy. 

Additionally, another benefit of utilizing a sufficient state is highlighted by \eqref{eq:fqi-iteration-number} --- it reduces the number of iterations required to attain this regret bound. In other words, using the minimal sufficient state can result in fewer iterations in FQI, thereby enhancing the computational efficiency.

\subsubsection{Proposition Regarding Policy Evaluation}\label{sec:policy-evaluation}
We consider off-policy evaluation (OPE) where the goal is to estimate the effect of a given target policy using the offline data generated from a different behavior policy. In particular, we focus on estimating the $Q$-function under the target policy. The following proposition, taken from Theorem 4.2 of \citet{chen2022well}, shows that the minimax optimal $L_2$-convergence rate of estimating $Q$-function of a target policy in Sobolev space is $(NT)^{-2p/d}$, where $p$ is the smoothness constant associated with the space. This formally demonstrates that  using all state variables as opposed to the proposed minimal sufficient state would incur a larger estimation error for policy evaluation. 
\begin{proposition}\label{thm: lower bound}
	Suppose (i) the offline data consists of a sequence of independent state-action-reward-next-state tuples; (ii) the target policy is absolutely continuous with respect to the behavior policy; (iii) the discounted visitation probability under the target policy is uniformly bounded away from zero and infinity; (iv) the conditional variance of the temporal different error is bounded away from zero. 
Then we have
	\begin{equation}\label{eqn: Q error}
	\liminf_{NT \rightarrow \infty}\inf_{\widehat Q} \sup_{Q \in \Lambda_2(p, L)}\mathbb{P}^Q\left(\|\widehat Q - Q \|_{L^2}\geq c_4 (NT)^{-p/(2p+d)}\right) \geq c_5 >0,\nonumber
	\end{equation}
	for some positive constants $c_4,c_5$, where $\Lambda_2(p, L)$ is the Sobolev space of smoothness p with radius $L$, $\mathbb{P}^Q$ refers to the offline data distribution with $Q$-function $Q$, and $\|\bullet\|_{L^2}$ is the $L_2$-norm of measurable functions.
\end{proposition}

\subsubsection{Action-independence Proposition}\label{sec:action-independence}
The following proposition shows that the minimal sufficient state of $\mathcal{M}$ is the union of the minimal sufficient state under a Markov decision process with a fixed action space. Therefore, the proposed minimal sufficient state of $\mathcal{M}$ is inherently independent of any particular action. 
\begin{proposition}\label{prop:action-independence}
	For each $a \in \mathcal{A}$, suppose the minimal sufficient state of $\mathcal{M}_a = (\mathcal{S}, \{a\}, \mathcal{R}, \mathcal{P}, \gamma)$ is $G_a$, i.e., $G_a$ satisfies
	\begin{align*}
		\Mean(R_t \given \mathbf{S}_t,A_t=a)=\Mean(R_t \given \mathbf{S}_{t,G_a},A_t=a),\,\,\,\, \mathbf{S}_{t+1,G_a}\independent \mathbf{S}_{t} \mid A_t=a, \mathbf{S}_{t,G_a},
	\end{align*}
	where for any two different $a, a' \in \mathcal{A}$, $G_a$ may not necessarily equal to $G_{a'}$. Then, the minimal sufficient state of $\mathcal{M}$ is $\bigcup\limits_{a \in \mathcal{A}} G_{a}$. 
\end{proposition}
\begin{proof}
	In our proof, we denote $\bigcup\limits_{a \in \mathcal{A}} G_{a}$ as $G$. And it is direct to see that:
	\begin{align*}
		\Mean(R_t \given \mathbf{S}_t,A_t)=\Mean(R_t \given \mathbf{S}_{t,G},A_t) \; \textup{ or } \;
		\mathbf{S}_{t+1,G}\independent \mathbf{S}_{t} \mid A_t, \mathbf{S}_{t,G};
	\end{align*}
	consequently, $G$ is a sufficient state of $\mathcal{M}$.

	Next, we prove $G$ is the minimal sufficient state of $\mathcal{M}$ by showing that, for any $j \in G$, $G \setminus \{j\}$ does not be a sufficient state. Specifically, for any $j \in G$, it must belong to $G_a$ for some $a \in \mathcal{A}$. As $G_a$ is the minimal sufficient state of $\mathcal{M}_a$, then we have either 
	\begin{align*}
		\Mean(R_t \given \mathbf{S}_t,A_t=a)=\Mean(R_t \given \mathbf{S}_{t,G_a \setminus \{j\}},A_t=a) \; \textup{ or } \;
		\mathbf{S}_{t+1,G_1}\independent \mathbf{S}_{t} \mid A_t=a, \mathbf{S}_{t,G_a \setminus \{j\}},
	\end{align*}
	does not hold. Hence, we can conclude that either $\Mean(R_t \given \mathbf{S}_t,A_t) = \Mean(R_t \given \mathbf{S}_{t,G \setminus \{j\}},A_t)$ or $\mathbf{S}_{t+1,G_1}\independent \mathbf{S}_{t} \mid A_t, \mathbf{S}_{t,G \setminus \{j\}}$ does not hold; in other words, $G \setminus \{j\}$ is not a sufficient state. Therefore, $G$ itself must be minimal sufficient state.
\end{proof}

From Proposition~\ref{prop:action-independence}, this union is necessary for optimal policy making in general, as either action may be optimal depending on the state. In cases where one or a few actions consistently excel others, it may be practical to consider only the minimal sufficient states associated with these actions, resulting in a smaller subset than the proposed minimal sufficient state. In that case, we can apply a two-step procedure that first eliminate those ``harmful'' actions that consistently yield lower values, and apply the proposed procedure to the resulting MDP with a reduced action space. 

\subsection{Proof of Theorem \ref{thm1}}\label{apdx:2}

This section is divided into three parts. We first show the constructed $W_j$-statistic satisfies the flip-sign property \citep{candes2018panning} in Section \ref{sec:flipsignproperty}. This finding forms the basis for the proof of Theorem \ref{thm1}. We next illustrate how this property enables Algorithm \ref{alg:2} to achieve asymptotic FDR control if the selected $\widehat{G}_k$ were computed from independent data in Section \ref{sec:proofthm1inddata}. Finally, we replace this independent assumption with exponential $\beta$-mixing (Condition \ref{ass:mixing}) to prove Theorem \ref{thm1} in Section \ref{sec:proofthm1depdata}.
\subsubsection{Flip-sign Property}\label{sec:flipsignproperty}
We first introduce the $\text{swap}(B)$ operator, applied to the augmented data matrix $[\mathbf{s}_k ~ \tilde{\mathbf{s}}_k]$ containing both the state matrix $\mathbf{s}_k\in \mathbb{R}^{n\times d}$ and its corresponding knockoff matrix $\tilde{\mathbf{s}}_k\in \mathbb{R}^{n\times d}$ in $\widetilde{\mathcal{D}}_k=\{\widetilde{\mathcal{D}}_k^{(a)}\}_{a\in \mathcal{A}}$.

\begin{definition}[Swapped data]
	\label{def: c.1}
	For any $1 \leq a \leq m$, the dataset $[\mathbf{s}_k ~ \tilde{\mathbf{s}}_k ]_{\text{swap}(B)}$ is obtained by swapping the $j$th columns of $\mathbf{s}_k$ and $\tilde{\mathbf{s}}_k$ for all $j\in B$.
\end{definition}

In the following lemma, we show the feature importance statistics we compute from $\calD_k$ satisfies the flip-sign property. Importantly, this lemma holds regardless of whether the observations are independent or not.
\begin{lemma}[Flip-sign property of feature importance statistics]\label{lemma:c.3}
For any subset $B\subset \{1,\ldots,d\}$ and any $k\in \{1, \ldots, K\}$, the feature importance statistics compute on the dataset $\widetilde{\calD}_k$ satisfies \useshortskip
	\begin{align*}
		W_j([ \mathbf{s}_k ~ \tilde{\mathbf{s}}_k ]_{\text{swap}(B)},
		\mathbf{a}_k, \mathbf{y}_k)
		=
		W_j([ \mathbf{s}_k ~ \tilde{\mathbf{s}}_k ],
		\mathbf{a}_k, \mathbf{y}_k) \times
		\begin{cases}
			-1, & \text{if}\ j\in B, \\
			+1, & \text{otherwise},
		\end{cases}
	\end{align*}
 where $\bm{a}_k$ and $\bm{y}_k$ are $n$-dimensional action and response vectors in $\calD_k$, respectively.
\end{lemma}
\begin{proof}[\textbf{\textup{Proof of Lemma \ref{lemma:c.3}}}.]
Recall that each $W_j$-statistic is computed as follows: 
\begin{enumerate}
    \item The data is divided into different subsets $\{\widetilde{\mathcal{D}}_k^{(a)}\}_{a \in \aspace}$, grouped by the action, where each $\widetilde{\mathcal{D}}_k^{(a)}$ contains the state and knockoff matrices, along with the response vector associated with action $a$: $\mathbf{s}_k^{(a)}$, $\widetilde{\mathbf{s}}_k^{(a)}$ and $\mathbf{y}_k^{(a)}$.
    \item We next apply a nonparametric regression method such as LASSO or random forest to each $\widetilde{\mathcal{D}}_k^{(a)}$ to compute the feature importance for each $j$th state variable and its knockoff counterpart, denoted by $Z_j^{(a)}$ and $\widetilde{Z}_j^{(a)}$, respectively.
    \item Finally, we aggregate these importance measures across different actions, leading to $Z_j=\max\limits_{a \in \aspace} Z_j^{(a)}$ and $\widetilde{Z}_j=\max\limits_{a \in \aspace} \widetilde{Z}_j^{(a)}$, and set $W_j=Z_j-\widetilde{Z}_j$.
\end{enumerate}
Importantly, either LASSO or random forest applied to the second step satisfies an exchangeability property. That is, if we were to permute the columns of the predictor matrix according to a permutation $\Pi: (1,\cdots,d)\to (\Pi(1),\cdots,\Pi(d))$, its output feature importance $Z_{\Pi(j)}^{(a)}$ ($\widetilde{Z}_{\Pi(j)}^{(a)}$)
would equal $Z_j^{(a)}$ ($\widetilde{Z}_j^{(a)}$) obtained prior to permutation. Based on this property, it is clear that:
\begin{itemize}
    \item For $j\notin B$, $Z_j^{(a)}$ and $\widetilde{Z}_j^{(a)}$ remain unchanged after swapping, and consequently, so do $Z_j$ and $\widetilde{Z}_j$. Thus, the $W_j$-statistic remains unchanged. 
    \item For $j\in B$, after swapping, $Z_j^{(a)}$ and $\widetilde{Z}_j^{(a)}$ switch roles for each $a$. By definition, $W_j$ becomes $-W_j$. 
\end{itemize}
This completes the proof of Lemma \ref{lemma:c.3}. 
\end{proof}

\subsubsection{Proof for Independent Data}\label{sec:proofthm1inddata}
We illustrate how the flip-sign property leads to asymptotic FDR control in this section. With independent data, the idea is to adapt Theorem 1 of \citet{barber2020robust} to our setting, which yields the following lemma. 
\begin{lemma}\label{thm:indep-fdr}
	Under Conditions \ref{ass:external-dataset} and \ref{condnull}, $\widehat{G}_k$ returned by Algorithm \ref{alg:2} satisfies \vspace{-1em}
	\begin{align*}
		\textrm{mFDR}(\widehat{G}_k)\le q \exp(\epsilon) + \mathbb{P}\Big(\max_{j \in \mathcal{H}_0}\widehat{\textup{KL}}_j > \epsilon\Big)
	\end{align*}\vspace{-3.5em}
 
	\noindent for any $\epsilon>0$. Additionally, $\widehat{G}_k$ outputted by model-X knockoffs satisfies \vspace{-1em}
	\begin{align*}
		\textrm{FDR}(\widehat{G}_k)\le q \exp(\epsilon) + \mathbb{P}\Big(\max_{j \in \mathcal{H}_0}\widehat{\textup{KL}}_j > \epsilon\Big).
	\end{align*}\vspace{-3.5em}
\end{lemma}
The proof is largely parallel to that of Theorem 1 in \citet{barber2020robust}, so a full proof is omitted here to save space. Instead, we emphasize two critical aspects:
\begin{itemize}
    \item \textbf{Role of the flip-sign property}: this property guarantees that for each null variables $j$, $W_j$ is almost equally likely to be positive or negative. More specifically, it guarantees the validity of Lemma 2 of \citet{barber2020robust}, which states that for any $\epsilon>0$, \vspace{-1em}
    \begin{eqnarray}\label{eqn:lemma2}
        \mathbb{P}(W_j>0,\widehat{\textrm{KL}}_j\le \epsilon| |W_j|, \mathbf{W}_{-j})\le \exp(\epsilon) \mathbb{P}(W_j<0| |W_j|, \mathbf{W}_{-j}),
    \end{eqnarray}\vspace{-3.5em}

    \noindent where $\mathbf{W}_{-j}$ denotes the set of $W$-statistics excluding $W_j$. When the covariate distribution is known, the KL-divergence equal zero,  leading to equal probabilities of $W_j>0$ and $W_j<0$ given $|W_j|$ and $\mathbf{W}_{-j}$. When the distribution is unknown, the two probabilities remain similar, with their difference governed by the KL-divergence. 
    \item \textbf{Necessity of the action-wise approach}: One key difference from the regression framework, as highlight in Section \ref{sec:seek} is that our knockoff variables are constructed in an action-wise manner to guarantee the exchangeability condition $(\mathbf{S},\widetilde{\mathbf{S}})_{\textrm{swap}(B)}|A\stackrel{d}{=}(\mathbf{S},\widetilde{\mathbf{S}})|A$ to hold conditional on the action, rather than marginally, facilitating the proof of \eqref{eqn:lemma2} in our setting. Specifically, such a conditional exchangeability, coupled with the independent data assumption, allows us to show that \vspace{-1em}
    \begin{eqnarray*}
        \frac{\mathbb{P}(W_j>0,\widehat{\textrm{KL}}_{j}\le \epsilon|\text{All}~(S_j^{(0)},S_j^{(1)},\mathbf{S}_{-j},\widetilde{\mathbf{S}}_{-j},A,Y)~\text{tuples in}~\widetilde{\mathcal{D}}^{(k)})}{\mathbb{P}(W_j<0|\text{All}~(S_j^{(0)},S_j^{(1)},\mathbf{S}_{-j},\widetilde{\mathbf{S}}_{-j},A,Y)~\text{tuples in}~\widetilde{\mathcal{D}}^{(k)})}\le \exp(\epsilon),
    \end{eqnarray*}\vspace{-3.5em}
    
    \noindent where $S_j^{(0)}$ and $S_j^{(1)}$ are analogs to $X_j^{(0)}$ and $X_j^{(1)}$ in Equation (22) of \citet{barber2020robust}. Importantly, both the numerator and denominator include conditioning on the action. This ensures that the minimal sufficient state contains only those variables that directly impact the reward and relevant future states. Without conditioning on actions, SEEK is not guaranteed to effectively exclude variables that solely exert indirect effects the reward or relevant future state through actions.
\end{itemize}


\subsubsection{Proof for Dependent Data}\label{sec:proofthm1depdata}
In this section, we replace the independence assumption with Assumption~\ref{ass:mixing} to handle dependent observations. First, we introduce some notations.
\begin{definition}[Total variation] For a given measure space $(\Omega, \Sigma, \mu)$, where $\Omega$ is the underlying space, $\Sigma$ is a $\sigma$-algebra over $\Omega$, and $\mu$ is a corresponding signed measure, the total variation of $\mu$ is $||\mu||_{TV} := \sup\limits_{A\in\Sigma}|\mu(A) - \mu(A^c)|.$
\end{definition}

\begin{definition}[$\beta$-coefficient] 
	Given two random vectors $X, Y$ defined on a space $\Omega$, the $\beta$-coefficient measuring their dependency is defined as $\beta(X, Y) := \frac{1}{2} ||P_{X,Y} - P_X P_Y||_{TV}$
	where $P_X,P_Y,$ and $P_{X,Y}$ are probability measures induced respectively by $X,Y$ and $(X,Y)$. 
\end{definition}

Next we review an important result introduced in Lemma 2.1 of \citet{berbee1987convergence}, which is the key to handle dependent data in our proof. 
\begin{lemma}\label{lemma:c.5}
	On a given space $\Omega$ with probability measure $\mathbb{P}$ and real random variables $X_1,\ldots,X_n$, define a sequence of coefficients
	\begin{equation}\label{eq:beta-i}
		\beta_i := \beta(X_i, (X_{i+1},\ldots,X_n))
	\end{equation}
	for $i=1,\ldots,n-1$. Then we can construct a new sequence of random variables on $\Omega$, denoted as $Z_1,\ldots,Z_n$ such that
	\begin{enumerate}
		\item $Z_1,\ldots,Z_n$ are independent;
		\item $Z_i \overset{d}{=}  X_i$;
		\item $\mathbb{P}(Z_i \neq X_i \text{ for some } i\in[n]) \leq \beta_1+\cdots+\beta_{n-1}$.
	\end{enumerate}
\end{lemma}
In particular, the detailed steps on the construction of $Z_1,\ldots,Z_n$ can be found in \citet{schwarz1980finitely}.

We next introduce the formal definition of exponential $\beta$-mixing. 
\begin{definition}[Exponential $\beta$-mixing, \citet{bradley2005basic}]\label{def:beta-mix}
	For a sequence of random variables $\{X_t\}$, define its $\beta$-mixing coefficient as
	\begin{equation*}
		\beta(i) := \sup\limits_{m\in\mathbb{Z}} 
		\beta(\mathcal{F}_{-\infty}^{m}, \mathcal{F}_{i+m}^{\infty})
	\end{equation*}
	where $\mathcal{F}_{J}^{L}$ is the $\sigma$-field generated by $\{X_J,\ldots,X_L\}$. A stationary Markov chain $\{X_t\}_{t\geq 0}$ is said to be exponentially $\beta$-mixing if its $\beta$-mixing coefficient $\beta(i)$ satisfies
	\begin{equation*}
		\beta(i) = O(\rho^{i})
	\end{equation*}
	for some constant $0<\rho<1$.
\end{definition}

Next, we prove Theorem~\ref{thm1}. 
\begin{proof}[\textup{\textbf{Proof of Theorem~\ref{thm1}}}]
	Consider a data subset $\calD_k\in\mathbb{R}^{(NT/K)\times(2d+2)}$ with
\begin{align*}
	\calD_k = [\mathbf{s}_k, \mathbf{a}_k, \mathbf{y}_k] = 
	\begin{bmatrix}
		s_{k-1} & a_{k-1} & y_{k-1} \\
		s_{k-1+K} & a_{k-1+K} & y_{k-1+K} \\
		\vdots & \vdots & \vdots \\
		s_{k-1+jK} & a_{k-1+jK} & y_{k-1+jK} \\
		\vdots & \vdots & \vdots \\
		s_{k-1+(NT-K)} & a_{k-1+(NT-K)} & y_{k-1+(NT-K)} \\
	\end{bmatrix}.
\end{align*}
By \citet{schwarz1980finitely} and Lemma~\ref{lemma:c.5}, we can construct $[\mathbf{z}, \mathbf{b}]\in \mathbb{R}^{\frac{NT}{K}\times {(p+1)}}$, i.e., a collection of $\frac{NT}{K}$ \textit{i.i.d.} samples $[\mathbf{z}_1, b_1],\ldots,[\mathbf{z}_\frac{NT}{K}, b_\frac{NT}{K}]$ satisfying
\begin{equation}\label{eq:1}
	\begin{split}
		& \mathbb{P} ([\mathbf{s}_{k-1+(i-1)K},a_{k-1+(i-1)K}] \neq [\mathbf{z}_i,b_i] \text{ for some } i \in \{1,\ldots,K^{-1}NT\}) 
		\\
		\leq 
		& \beta_1 + \ldots + \beta_{\frac{NT}{K}-1}
	\end{split}
\end{equation}
where $\beta_i$ is defined by Equation~\eqref{eq:beta-i} and has expression:
\begin{align*}
	\beta_i := \beta([\mathbf{s}_{k-1+(i-1)K},a_{k-1+(i-1)K}], ([\mathbf{s}_{k-1+iK},a_{k-1+iK}],\ldots,[\mathbf{s}_{k-1+(NT-K)},a_{k-1+(NT-K)}]))
\end{align*}
for $i=1,\ldots,\frac{NT}{K} - 1$. Notice that the gap between two adjacent indices is at least $K$.
By Assumption \ref{ass:mixing} and 
the construction of $\calD_k$ with $K = k_0\log(NT)$, 
\begin{align}\label{eq:beta_i}
	\beta_i = O(\rho^{K}) = O(\rho^{k_0\log(NT)}).
\end{align}
Notice that if any two consecutive observations in $\calD_k$ are from two different trajectories, then they are independent, and the above inequality remains valid. 
It follows from Equation~\eqref{eq:1} that
\begin{align}\label{eqn:someimportantinequality}
\begin{split}
	& \mathbb{P} ([\mathbf{s}_{k-1+(i-1)K},a_{k-1+(i-1)K}] 
	\neq [\mathbf{z}_i,b_i] \text{ for some } i \in \{1,\ldots,\frac{NT}{k_0\log(NT)}\})\\
	 =&\frac{NT}{K}O\left(\rho^{k_0\log(NT)}\right)
	= O\left(\frac{NT}{K}(NT)^{k_0\log(\rho)}\right)=O\left(\frac{1}{K} \left(\frac{1}{NT}\right)^{k_0\log(\frac{1}{\rho}) - 1}\right).
\end{split}
\end{align}
\begin{remark}\label{remarkbridge}
    This probability bound serves as a critical bridge, connecting our results from independent data to those with MDP data. Specifically, if certain properties are established with high probability for independent data, we can extend these results to MDP data, at the cost of a reduced probability by $O\Big\{K^{-1}(NT)^{k_0\log(1/\rho)-1}\Big\}$. The same applies to metrics such as FDR/mFDR, as we illustrate below. 
\end{remark}
Combining \eqref{eqn:someimportantinequality} together with Lemma \ref{thm:indep-fdr} enables us to control the FDR/mFDR asymptotically. Specifically,
\begin{equation}\label{eq:one-step-fdr}
	\begin{split}
		\mathbb{E}\left[\frac{\displaystyle|\widehat{G}_k\cap \mathcal{H}_0|}{\displaystyle|\widehat{G}_k| \vee 1}\right]
		=& \mathbb{E}\left[\frac{\displaystyle|\widehat{G}_k\cap \mathcal{H}_0|}{\displaystyle|\widehat{G}_k| \vee 1} \mathbf{1}_{\{[\mathbf{s}_{k-1+(i-1)K},a_{k-1+(i-1)K}] = [\mathbf{z}_i,b_i] \text{ for all } i\}} \right]\\
		&+ \mathbb{E}\left[\frac{\displaystyle|\widehat{G}_k\cap \mathcal{H}_0|}{\displaystyle|\widehat{G}_k| \vee 1} \mathbf{1}_{\{[\mathbf{s}_{k-1+(i-1)K},a_{k-1+(i-1)K}] \neq [\mathbf{z}_i,b_i] \text{ for some } i\}} \right]\\
		\leq & 
		q\exp(\epsilon) + \mathbb{P}\Big(\max_{j\in \mathcal{H}_0}\widehat{\textrm{KL}}_j>\epsilon\Big)
		\\
		&+ 
		\mathbb{P} ([\mathbf{s}_{k-1+(i-1)K},a_{k-1+(i-1)K}] \neq [\mathbf{z}_i,b_i] \text{ for some } i \in \{1,\ldots,K^{-1}NT \})
		\\
		= & 
		q\exp(\epsilon) + \mathbb{P}\Big(\max_{j\in \mathcal{H}_0}\widehat{\textrm{KL}}_j>\epsilon\Big) + O(K^{-1}(NT)^{1-k_0\log(\rho^{-1})}),
	\end{split}
\end{equation}
where the first inequality is due to Lemma \ref{thm:indep-fdr} and the fact that  $\frac{\displaystyle|\widehat{G}_k\cap \mathcal{H}_0|}{\displaystyle|\widehat{G}_k| \vee 1} \leq 1$ whereas the last inequality comes from Equation~\eqref{eqn:someimportantinequality}.
\end{proof}

\subsection{Proof of Theorem \ref{thm2}}\label{sec:proofthm2} 
\subsubsection{Proof for Independent Data}\label{sec:indedataproof}
Similar to the proof of Theorem \ref{thm1}, we first assume the observations are independent to illustrate the main idea. We begin by revisiting some notations. Recall that $\widehat{G}_{k,1}$ denotes the selected state index during the first iteration when applying Algorithm \ref{alg:2} to all $(\mathbf{S}, A, R)$ triplets in $\mathcal{D}_k$. For subsequent iterations ($j>1$), we apply knockoffs to all $(\mathbf{S}, A, \mathbf{S}_{\widehat{G}_{k,j-1}}')$ triplets in $\mathcal{D}_k$, and set $\widehat{G}_{k,j}$ to the union of the selected index at this iteration and those from previous iterations $\widehat{G}_{k,j-1}$. Finally, we compute $\widehat{G}$ by applying majority voting based on these collected outputs across all $k$.

To highlight the dependence between the response at each iteration and the selected state index from the previous iteration, we redefine each $\widehat{G}_{k,j}$ as a function of the state index at the previous iteration, expressed as $\widehat{G}_{k,j}(\widehat{G}_{k,j-1})$. We aim to bound the probability of discovering null states:
\begin{equation*}
    \mathbb{P}\left( \left\{\widehat{G}_k \cap \mathcal{H}_0 \neq \emptyset \right\} \right),
\end{equation*}
which equals the following probability due to the recursive nature of our algorithm,
\begin{equation*}
    \mathbb{P}\Big( \left\{\widehat{G}_{k,1} \cap \mathcal{H}_0 \neq \emptyset \right\} \bigcup \left\{\widehat{G}_{k,2}(\widehat{G}_{k,1}) \cap \mathcal{H}_0 \neq \emptyset \right\} \bigcup\cdots \bigcup \left\{\widehat{G}_{k,j}(\widehat{G}_{k,j-1}) \cap \mathcal{H}_0\neq \emptyset\right\} \bigcup\cdots \Big),
\end{equation*}
and can be upper bounded by
\begin{eqnarray}\label{eq:couple-events}
        \mathbb{P}\Big( \left\{\widehat{G}_{k,1} \cap \mathcal{H}_0 \neq \emptyset \right\}\Big)+
        \mathbb{P}\Big( \left\{\widehat{G}_{k,1} \cap \mathcal{H}_0 =\emptyset \right\} \bigcap \Big\{\bigcup_{j>1}\left\{\widehat{G}_{k,j}(\widehat{G}_{k,j-1}) \cap \mathcal{H}_0 \neq \emptyset \right\}\Big\} \Big).
\end{eqnarray}
Here, the second term in \eqref{eq:couple-events} is very difficult to analyze, due to the dependence between $\widehat{G}_{k,j}$ and $\widehat{G}_{k,j-1}$. To decouple such dependence, 
we notice that, the event in the second term in \eqref{eq:couple-events} occurs only when there exists some $G\cap \mathcal{H}_0=\emptyset$ such that knockoffs selects a null variable in its output $\widehat{G}$ when using $\mathbf{S}'_{G}$ as the response. This observation enables us to upper bound the second probability in \eqref{eq:couple-events} by
\begin{eqnarray}\label{eq:couple-events2}
    \mathbb{P}\Big(\bigcup_{G:G\cap \mathcal{H}_0=\emptyset}  \left\{\widehat{G}_k(G) \cap \mathcal{H}_0 \neq \emptyset \right\}\Big).
\end{eqnarray}
However, it remains challenging to upper bound \eqref{eq:couple-events2}. This is because the number of nonempty state indices satisfying $G\cap \mathcal{H}_0=\emptyset$ is $2^{d_0}-1$, which increases exponentially with the number of non-null variables $d_0$. Employing a simple Bonferroni’s inequality results in a loose bound proportional to $2^{d_0}-1$. 

To remove such exponential dependence upon $d_0$, we take advantage of our knockoffs procedures which operates elementwise on the response. That is, given the response $\mathbf{S}'_{G}$, we apply Algorithm \ref{alg:2} to each $(\mathbf{S}, A, S_l')$ tuple for every $l\in G$. This elementwise approach allows us to refine our probability bound to:
\begin{eqnarray*}
    \mathbb{P}\Big(\bigcup_{j:j\notin \mathcal{H}_0}  \left\{\widehat{G}_k(\{j\}) \cap \mathcal{H}_0 \neq \emptyset \right\}\Big),
\end{eqnarray*}
effectively reducing the complexity from $2^{d_0-1}$ to $d_0$ and achieving a much tighter bound using Bonferroni's inequality.
Consequently, \eqref{eq:couple-events} can be upper bounded by
\begin{eqnarray}\label{eqn:finalinequality}
    \sum_{j\in \mathcal{H}_0^c\cup \{d+1\}} \mathbb{P}(\widehat{G}_k(\{j\}) \cap \mathcal{H}_0 \neq \emptyset),
\end{eqnarray}
where $\widehat{G}_k(\{d+1\})$ denotes the selected subset at the first iteration for the reward. The following lemma shows that the probability in \eqref{eqn:finalinequality} for each $j$ can be upper bounded by $O(d_0 q)$. Consequently, \eqref{eqn:finalinequality} is of the order $O(d_0^2 q)$. 

Finally, notice that $\widehat{G}$ will not output any null variable if $\sum_{k=1}^K \mathbb{I}\{\widehat{G}_k\cap \mathcal{H}_0 \neq \emptyset\}<\alpha K$. An application of Markov's inequality yields 
\begin{eqnarray}\label{eqn:Markovinequality}
\begin{split}
    \mathbb{P}(\widehat{G}\cap \mathcal{H}_0\neq \emptyset)\le \mathbb{P}(\sum_{k=1}^K \mathbb{I}\{\widehat{G}_k\cap \mathcal{H}_0 \neq \emptyset\}\ge \alpha K)\\\le \frac{ \sum_{k=1}^K \mathbb{P}\{\widehat{G}_k\cap \mathcal{H}_0 \neq \emptyset\}}{\alpha K}=O( \frac{qd_0^2}{\alpha}).
\end{split}
\end{eqnarray}
This completes the proof with independent data. 

\begin{lemma}\label{lemma:onestep-false-inclusion}
Under Conditions \ref{ass:external-dataset} and \ref{condnull},
	\begin{align*}
		\max_{j\in \{d+1\}\cup \mathcal{H}_0^c}\mathbb{P}(\widehat{G}_k(\{j\}) \cap \mathcal{H}_0 \neq \emptyset) =O(d_0 q).
	\end{align*}
\end{lemma}
\begin{proof}[\textup{\textbf{Proof of Lemma~\ref{lemma:onestep-false-inclusion}}}]
	Without the loss of generality, we assume $|\widehat{G}_k(\{j\})\cap \mathcal{H}_0|>0$. Define  
\begin{align*}
	\eta_j:= \frac{\displaystyle|\widehat{G}_k(\{j\})\cap \mathcal{H}_0|}{\displaystyle|\widehat{G}_k(\{j\})| \vee 1}
	&= \left(1 + \frac{\displaystyle  |\widehat{G}_k(\{j\})\backslash\mathcal{H}_0|}{\displaystyle|\widehat{G}_k(\{j\})\cap \mathcal{H}_0|}\right)^{-1}.
\end{align*}
The last equation implies that $\eta_j$ strictly decreases as $|\widehat{G}_k(\{j\})\backslash\mathcal{H}_0|/|\widehat{G}_k(\{j\})\cap \mathcal{H}_0|$ increases. Given that $d_0$ is the dimension of the minimal sufficient state, at least one null variable is selected if and only if $\eta_j\ge (1+d_0)^{-1}$. 
This together with the Markov inequality yields
\begin{align*}
\max_{j\in \{d+1\}\cup \mathcal{H}_0^c}\mathbb{P}(\widehat{G}_k(\{j\}) \cap \mathcal{H}_0 \neq \emptyset) 
&\le \max_{j\in \{d+1\}\cup \mathcal{H}_0^c}\mathbb{P}(\eta_j\ge (1+d_0)^{-1})
\\
&\le \frac{\mathbb{E}(\eta_j)}{(1+d_0)^{-1}}
=\frac{\textrm{FDR}(\widehat{G}_k(\{j\}))}{(1+d_0)^{-1}}.
\end{align*}
By setting $\epsilon$ in Theorem \ref{thm1} to 1, an application of Lemma \ref{thm:indep-fdr} yields that $\textrm{FDR}(\widehat{G}(\{j\}))\le q\exp(1)$, which in turn concludes the proof. 
\end{proof}
\subsubsection{Proof for Dependent Data}\label{sec:thm2depdataproof}
\begin{proof}[\textup{\textbf{Proof of Theorem~\ref{thm2}}}]
    The general proof is very similar to that for independent data. Let $\mathcal{I}$ denoted the event in Equation \eqref{eqn:someimportantinequality}. The probability of $\widehat{G}_k$ discovering null states can be upper bounded by
\begin{eqnarray*}
    \mathbb{P}(\mathcal{I})+\mathbb{P}\left( \left\{\widehat{G}_k \cap \mathcal{H}_0 \neq \emptyset \right\}\cap \mathcal{I}^c \right)
\end{eqnarray*}
The first term can be upper bounded by $O(K^{-1} (NT)^{1-k_0\log (1/\rho)})$, according to Equation \eqref{eqn:someimportantinequality}. Meanwhile, the second term can be upper bounded as if the data were independent, which falls into the order $O(d_0^2 q)$ according to our analysis in Section \ref{sec:indedataproof}. 

Similar to \eqref{eqn:Markovinequality}, after majority voting, the probability of $\widehat{G}$ identifying null states is inflated by a factor of $\alpha^{-1}$, given by
\begin{eqnarray*}
    O\Big(\alpha^{-1}K^{-1} (NT)^{1-k_0\log (1/\rho)}\Big)+O\Big(\frac{d_0^2 q}{\alpha}\Big).
\end{eqnarray*}
This completes the proof of Theorem \ref{thm2}.
\end{proof}

\subsection{Proof of Theorems \ref{thm:power-lasso} and \ref{thm:type-II}}\label{sec:proofthm3and4}
This section is organized as follows. We first introduce some technical conditions in Section \ref{thm:powercond}. We next analyze the TPR of $\widehat{G}_k$ output by Algorithm \ref{alg:2} to prove Theorem \ref{thm:power-lasso} in Section \ref{sec:poweroneiteration}. Finally, we study the type-II error of $\widehat{G}$ output by Algorithm \ref{alg:1} to prove Theorem \ref{thm:type-II} in Section \ref{sec:powersequential}. 

\subsubsection{Technical Conditions}\label{thm:powercond}
In this section, we impose some other regularity conditions in addition to Conditions \ref{ass:signal} and \ref{ass:signal2}. Let $\widetilde{\svar}^*$ denote the ``oracle'' knockoff vector that works as if the marginal state distribution were known in advance. To assess the quality of our constructed knockoff variables, we define $\{\Delta_{ij}^{(a)}\}_{i,j=1}^d$ as the difference in the second moment between $\widetilde{\svar}^*$ and our constructed knockoffs $\svar$ given $A=a$. More specifically, for any $1\le i,j\le d$, let $\Delta_{ij}^{(a)}=|\Mean (\widetilde{S}_i^* \widetilde{S}_j^*|A=a)-\Mean (\widetilde{S}_i \widetilde{S}_j|A=a)|$. Similarly, let $\Delta_{i|j}^{(a)}=\Mean |\Mean (\widetilde{S}_i^*|S_j,A=a)-\Mean (\widetilde{S}_i|S_j,A=a)|$, which measures the difference in their conditional expectations given $S_j$ and $A=a$.
\begin{condition}[Knockoff construction errors]\label{ass:knockoffs}
    $d_{*}\max\limits_{i,j,a} \Delta_{ij}^{(a)}=o(1)$ and $d_{*}\max\limits_{i,j,a} \Delta_{i|j}^{(a)}=o(1)$.
\end{condition}
\begin{condition}[Sparsity]\label{ass:sparsity}
    $(d_{*})^2\log(dn)=o(n)$.
\end{condition}
\begin{condition}[Eigenvalues]\label{ass:prec-matrix-estimate}
(i) The minimum eigenvalue of the matrix $$\mathbb{E}\Big[ [\svar^\top, (\widetilde{\svar}^*)^\top]^\top [\svar^\top, (\widetilde{\svar}^*)^\top]\mathbb{I}(A=a)\Big]$$ is bounded away from zero across all values of $a$. (ii) The maximum eigenvalue of the matrix $\mathbb{E}(\svar \svar^\top)$ is bounded away from infinity. 
\end{condition}
\begin{condition}[Tuning]\label{ass:tuning}
    Tuning parameters in LASSO are set to $c\sqrt{n^{-1/2}\log p}$ for some sufficiently large constant $c>0$. 
\end{condition}
\begin{condition}[No ties]\label{ass:noties}
    There are almost surely no ties in the magnitude of nonzero $W_j$s and in the magnitude of nonzero components of Lasso solution.
\end{condition}
As commented in Section \ref{sec:poweranalysis}, assumptions similar to Conditions \ref{ass:sparsity} -- \ref{ass:noties} are frequently imposed in the literature \citep[see e.g.,][]{fan2019rank}. We focus on discussing Condition \ref{ass:knockoffs} below, which requires that the constructed knockoffs should approximate the oracle knockoffs closely in distribution. This condition is likely to hold given the large amount of unlabeled data under Condition \ref{ass:external-dataset}. 

To elaborate, 
assume each state $\svar$ follows a multivariate normal distribution $\mathcal{N}(\bm{\mu}^{(a)}, \bm{\Sigma}^{(a)})$ when conditional on $A=a$. Under the exchangeability assumption \ref{eqn:satisfyexchangeable}, the oracle knockoff state $\widetilde{\svar}^*$ follows the same multivariate normal distribution. We adopt the second-order machine detailed in Section \ref{sec:secondorder} to construct the knockoff variables. In practice, however, $\bm{\mu}^{(a)}$ and $\bm{\Sigma}^{(a)}$ remain unknown. We first employ the external data to compute their estimators  $\widehat{\bm{\mu}}^{(a)}$ and $\widehat{\bm{\Sigma}}^{(a)}$, and then plug them into \eqref{eqn:gaussiansampling} to sample $\widetilde{\svar}$. This together with \eqref{eqn:gaussiansampling} leads to 
\begin{eqnarray}\label{eqn:Deltai|j1}
\begin{split}
    |\Mean (\widetilde{{S}}_j|\bm{S},A=a)-\Mean (\widetilde{S}_j^*|\bm{S},A=a)|\le |\bm{e}_j^\top(\bm{\Sigma}^{(a)})^{-1}\diag(\mathbf{d}^{(a)})(\widehat{\bm{\mu}}^{(a)}-\bm{\mu}^{(a)})|\\
    +|\bm{e}_j^\top[(\bm{\Sigma}^{(a)})^{-1}-(\widehat{\bm{\Sigma}}^{(a)})^{-1}]\diag(\mathbf{d}^{(a)})(\bm{S}-\bm{\mu}^{(a)})|\\
    +|\bm{e}_j^\top[(\bm{\Sigma}^{(a)})^{-1}-(\widehat{\bm{\Sigma}}^{(a)})^{-1}]\diag(\mathbf{d}^{(a)})(\widehat{\bm{\mu}}^{(a)}-\bm{\mu}^{(a)})|,
\end{split}
\end{eqnarray}
where $\bm{e}_j\in \mathbb{R}^d$ is a unit vector with the $j$th element $1$ and all other elements $0$. In addition, we have by Jensen's inequality that
\begin{eqnarray*}
    \max_{ij}\Delta_{j|i}^{(a)}\le \max_j \Mean\Big[|\Mean (\widetilde{{S}}_j|\bm{S},A=a)-\Mean (\widetilde{S}_j^*|\bm{S},A=a)|A=a\Big].
\end{eqnarray*}
This together with \eqref{eqn:Deltai|j1} and Cauchy-Schwarz inquality yields 
\begin{eqnarray}\label{eqn:Deltai|j2}
\begin{split}
    &\max_{ij}\Delta_{j|i}^{(a)}\le \max_j |\bm{e}_j^\top(\bm{\Sigma}^{(a)})^{-1}\diag(\mathbf{d}^{(a)})(\widehat{\bm{\mu}}^{(a)}-\bm{\mu}^{(a)})|\\
    +&\max_j|\bm{e}_j^\top[(\bm{\Sigma}^{(a)})^{-1}-(\widehat{\bm{\Sigma}}^{(a)})^{-1}]\diag(\mathbf{d}^{(a)})(\widehat{\bm{\mu}}^{(a)}-\bm{\mu}^{(a)})|\\
    +&\max_j|\bm{e}_j^\top[(\bm{\Sigma}^{(a)})^{-1}-(\widehat{\bm{\Sigma}}^{(a)})^{-1}]\diag(\mathbf{d}^{(a)})\bm{\Sigma}^{(a)}\diag(\mathbf{d}^{(a)})[(\bm{\Sigma}^{(a)})^{-1}-(\widehat{\bm{\Sigma}}^{(a)})^{-1}]\bm{e}_j|^{\frac{1}{2}}.
\end{split}
\end{eqnarray}
Below, we sketch a few lines to upper bound each term on the RHS of \eqref{eqn:Deltai|j2}: 
\begin{itemize}
    \item \textbf{The first term}: We observe the first term on the RHS of \eqref{eqn:Deltai|j2} is essentially a weighted average of $\widehat{\bm{\mu}}^{(a)}-\bm{\mu}^{(a)}$. Under the eigenvalue assumption in Condition \ref{ass:prec-matrix-estimate}(i), the Euclidean norm of the weight vector is uniformly bounded across all $j$. Together with the exponential $\beta$-mixing condition (Condition \ref{ass:mixing}), we can apply exponential inequalities for sums of weakly dependent variables \citep[see e.g.,][]{delyon2009exponential} to demonstrate that the RHS on the first line is of the order $\sqrt{|\mathcal{U}|^{-1} \log d}$ up to some $\log (|\mathcal{U}|)$ terms, with probability approaching $1$. 
    \item \textbf{The second term}: In line with the analysis presented below Lemma 5 of \citet{barber2020robust}, the second line of \eqref{eqn:Deltai|j2} can be shown to be of the order of magnitude $\sqrt{L |\mathcal{U}^{-1}|\log d}$ up to some $\log (|\mathcal{U}|)$ terms under exponential $\beta$-mixing, where $L$ denotes the maximum number of nonzero entries in any column of the Gaussian precision matrix \citep{barber2020robust}.
    \item \textbf{The third term}: The third term is a high-order reminder term whose order of magnitude is determined by the product of the orders of the first two terms.
\end{itemize}
Combining these results, the order of $\max\limits_{i,j,a}\Delta_{j|i}^{(a)}$ is given by $O(\sqrt{L |\mathcal{U}^{-1}|\log d})+O(\sqrt{L} |\mathcal{U}^{-1}|\log d)$, up to some $\log (|\mathcal{U}|)$ terms. Using similar arguments, one can show that $\max\limits_{i,j,a}\Delta_{ij}^{(a)}$ is of similar order of magnitude. Given that the size of the unlabeled data is much larger than that of the labeled data, the assumptions in \eqref{ass:knockoffs} are likely to be satisfied.


\subsubsection{Power in One Iteration}\label{sec:poweroneiteration}
\begin{proof}[Proof of Theorem \ref{thm:power-lasso}]
\textbf{\textit{Outline}}. We focus on proving that with independent data, the TPR is lower bounded by
\begin{align}\label{eqn:TPRind}
    \mathbb{E}\left(\frac{|\widehat{G}_k\cap G(\{i\})|}{|G(\{i\})|}\right) \geq 1 -O(\kappa_n^{-1})-O\left\{(nd)^{-C}\right\}.
\end{align}
Together with our analysis in Section \ref{sec:proofthm1depdata} under the exponential $\beta$-mixing condition, we can show that the TPR with dependent data will be further reduced by $O(K^{-1} (NT)^{1-k_0 \log (\rho^{-1})})$. The proof of Theorem \ref{thm:power-lasso} is thus completed.

The rest of the proof can be divided into three parts. Let $\bm{\Omega}^{(a)}$ denote $\Mean [\svar^\top, (\widetilde{\svar}^*)^\top]^\top [\svar^\top, (\widetilde{\svar}^*)^\top]\mathbb{I}(A=a)$. Under Condition \ref{ass:prec-matrix-estimate}(i), $\bm{\Omega}^{(a)}$ satisfies the restricted eigenvalue condition \citep[see e.g.,][]{bickel2009simultaneous,zhou2009restricted}. Let
$\widehat{\bm{\Omega}}^{(a)}_k$ denote the empirical estimator for $\bm{\Omega}^{(a)}$, i.e.,
\begin{eqnarray*}
    \widehat{\bm{\Omega}}^{(a)}_k=\frac{1}{n}\sum_{(\svar,\widetilde{\svar},A)\in \widetilde{\mathcal{D}}_k} [\svar^\top, (\widetilde{\svar})^\top]^\top [\svar^\top, (\widetilde{\svar})^\top]^\top\mathbb{I}(A=a).
\end{eqnarray*}
\begin{enumerate}
    \item \textbf{Part I} of the proof aims to upper bound the discrepancy between $\bm{\Omega}^{(a)}$ and $\widehat{\bm{\Omega}}^{(a)}_k$ using random matrix theories \citep{tropp2012user} and Condition \ref{ass:knockoffs}. In the proof, Condition \ref{ass:knockoffs} enables us to establish the restricted eigenvalue condition on  $\widehat{\bm{\Omega}}_k^{(a)}$ in the presence of the knockoff construction error.
    \item \textbf{Part II} establishes the oracle inequality for the LASSO estimator $\widehat{\bm{b}}^{(a)}$ under Condition \ref{ass:tuning}, built upon on the restricted eigenvalue condition on  $\widehat{\bm{\Omega}}_k^{(a)}$. This analysis follows standard techniques in the existing literature \citep[see e.g.,][]{bickel2009simultaneous,zhou2009restricted}.
    \item \textbf{Part III} lower bounds the TPR under Condition \ref{ass:noties}, accommodating the presence of moderate signals in Condition \ref{ass:signal}, using the techniques developed by \citet{fan2019rank}.
\end{enumerate} 
We next detail each part of the proof. We use $C$ to denote a generic constant whose value might change from place to place. 

\textbf{\textit{Part I}}. Let $G$ denote a given subset of $\{1,\cdots,2d\}$ and $|G|$ denote its cardinality. For any $2d\times 2d$ matrix $\bm{\Sigma}$, we use $\bm{\Sigma}_{GG}$ to denote the $|G|\times |G|$ submatrix of $\bm{\Sigma}$, formed by selecting the rows and columns indexed by $G$. Similarly, for any vector $\bm{\psi}\in \mathbb{R}^{2d}$,  $\bm{\psi}_G$ denote the subvector formed by selecting the entries in $G$. We focus on proving 
\begin{eqnarray}\label{eqn:randommatrix}
    \sup_{G: |G|\le d_*}\|\bm{\Omega}_{GG}^{(a)}-\widehat{\bm{\Omega}}_{k,GG}^{(a)}\|_2=O(d_*\Delta_{\max})+O(d_*\sqrt{n^{-1}\log (dn)}),
\end{eqnarray}
with probability at least $1-O((nd)^{-Cd_*})$ for any constant $C>0$, where $\Delta_{\max}=\max\limits_{i,j,a} \Delta_{ij}^{(a)}+\max\limits_{i,j,a}\Delta_{i|j}^{(a)}$.  

Let $G_*\subseteq \{1,\cdots,d\}$ denote the support of the oracle regression coefficient vector $\bm{b}^{(a)}$. It corresponds to the support of the $(d+1)$th column of $\bm{B}^{(a)}$ when the response is the immediate reward, and to the support of the $j$th column of $\bm{B}^{(a)}$ when the response is the $j$th next state. Assume for now, \eqref{eqn:randommatrix} has been proven. Under Conditions \ref{ass:knockoffs} and \ref{ass:sparsity}, the RHS of \eqref{eqn:randommatrix} is $o(1)$. Then, for any vector $\bm{\delta}\in \mathbb{R}^{2d}$ that satisfies $\|\bm{\delta}_{G_*^c}\|_1 \le 3\|\bm{\delta}_{G_*}\|_1$, it follows from the analysis in Equation (B.8) of \citep{zhou2009restricted} that
\begin{eqnarray}\label{eqn:randommatrix1}
    \sqrt{|\bm{\delta}^\top (\bm{\Omega}^{(a)}-\widehat{\bm{\Omega}}_{k}^{(a)}) \bm{\delta}|}=o(\|\bm{\delta}_{G_*}\|_2).
\end{eqnarray}
As commented earlier, Condition \ref{ass:prec-matrix-estimate}(i) implies that the restricted eigenvalue condition holds for $\bm{\Omega}^{(a)}$, i.e., there exists some constant $\bar{c}>0$ such that $\|\bm{\Omega}^{(a)}\bm{\delta}\|_2\ge \bar{c} \|\bm{\delta}_{G_*}\|_2$ for any $\bm{\delta}$ such that $\|\bm{\delta}_{G_*^c}\|_1 \le 3\|\bm{\delta}_{G_*}\|_1$. This together with \eqref{eqn:randommatrix1} implies that the restricted eigenvalue condition holds for $\widehat{\bm{\Omega}}_{k}^{(a)}$ as well. 

It remains to prove \eqref{eqn:randommatrix}. Notice that $\sup\limits_{G: |G|\le d_*}\|\bm{\Omega}_{GG}^{(a)}-\widehat{\bm{\Omega}}_{k,GG}^{(a)}\|_2$ can be decomposed into the sum of 
\begin{equation}\label{eq:bound-prec-matrix-estimate}
    \sup_{G: |G|\le d_*}\|\bm{\Omega}_{GG}^{(a)}-\Mean (\widehat{\bm{\Omega}}_{k,GG}^{(a)})\|_2+\sup_{G: |G|\le d_*}\|\Mean (\widehat{\bm{\Omega}}_{k,GG}^{(a)})-\widehat{\bm{\Omega}}_{k,GG}^{(a)}\|_2.
\end{equation}
We next show that the first term of \eqref{eq:bound-prec-matrix-estimate} is of the order $O(d_*\Delta_{\max})$ whereas the second term is of the order $O(d_*\sqrt{n^{-1}\log n})+O(\sqrt{d_* n^{-1}\log(dn)})$ with probability at least $1-O((nd)^{-Cd_*})$.
\begin{itemize}
    \item \textbf{Upper bound for the first term}: For any symmetric matrix $\bm{\Sigma}$, we have $\|\bm{\Sigma}\|_2\le \sqrt{\|\bm{\Sigma}\|_1\|\bm{\Sigma}\|_{\infty}}=\|\bm{\Sigma}\|_1$. As such, the first term can be upper bounded by 
    \begin{eqnarray*}\nonumber
        d_* \max\limits_{i,j,a} |\Omega_{ij}^{(a)}-\Mean (\widehat{\Omega}_{k,ij}^{(a)})|\le d_{*} \Big[\max\limits_{i,j,a} |\Mean (S_i \widetilde{S}_j|A=a)-\Mean (S_i \widetilde{S}_j^*|A=a)|\\+\max\limits_{i,j,a} |\Mean (\widetilde{S}_i \widetilde{S}_j|A=a)-\Mean (\widetilde{S}_i^* \widetilde{S}_j^*|A=a)|\Big].
    \end{eqnarray*}
    By definition, the quantities within the squared brackets can be upper bounded by $\max\limits_{i,j,a} \Delta_{j|i}^{(a)}+\max\limits_{i,j,a} \Delta_{ij}^{(a)}$. This proves that the first term is of the order $O(d_{*} \Delta_{\max})$. 
    \item \textbf{Upper bound for the second term}: We apply the matrix bernstein inequality \citep[see e.g.,][Theorem 6.2]{tropp2012user} to establish the high probability upper bound for the second term of \eqref{eq:bound-prec-matrix-estimate}. Recall that the state space $\sspace$ is a compact subspace of $\mathbb{R}^d$. Therefore, when constructing knockoff variables, it is reasonable to also restrict these variables to be bounded. In particular, let $M$ denote the upper bound of $\|\svar\|_{\infty}$ and $\|\widetilde{\svar}\|_{\infty}$. For any integer $p>1$, it follows that
    \begin{eqnarray*}
        \|\Mean (\svar_G \svar_G^\top)^p\|_2\le (M |G|)^{p-1} \|\Mean (\svar_G \svar_G^\top)\|_2,
    \end{eqnarray*}
    where the last term on the RHS is uniformly bounded away from infinity for any $G$, by Condition \ref{ass:prec-matrix-estimate}(ii). Similarly, we can show
    \begin{eqnarray*}
        \|\Mean (\widetilde{\svar}_G \widetilde{\svar}_G^\top)^p\|_2\le (M |G|)^{p-1} \|\Mean (\widetilde{\svar}_G \widetilde{\svar}_G^\top)\|_2,
    \end{eqnarray*}
    where the last term on the RHS is again uniformly bounded away from infinity, by \eqref{eqn:randommatrix1}. Using Cauchy-Schwarz inequality, we can establish similar bounds for $\|\Mean ([\svar_G^\top, \widetilde{\svar}_G^\top]^\top [\svar_G^\top, \widetilde{\svar}_G^\top]\mathbb{I}(A=a)-\bm{\Omega}^{(a)})^p\|_2$ and show that the conditions in Theorem 6.2 of \citet{tropp2012user} are satisfied with $R=O(|G|)$ and $\sigma^2=O(|G|)$. This enables us to apply Theorem 6.2 to show that for each $G$ with $|G|=d_*$, $\|\bm{\Omega}^{(a)}_{GG}-\widehat{\bm{\Omega}}^{(a)}_{k,GG}\|_2$ can be upper bounded by $O(d_*\sqrt{n^{-1}\log (nd)})+O(d_*^2n^{-1}\log (nd))$ with probability at least $1-O((nd)^{-C d_*})$ for any sufficiently large $C>0$. Notice that the number of subsets $G\subseteq \{1,\cdots,2d\}$ satisfying $|G|=d_*$ is at most $2d \choose d_*$. Under Condition \ref{ass:sparsity}, an application of the Bonferroni's inequality allows us to obtain the desired uniform error bound across all such $G$.
\end{itemize}
To summarize, the above analysis verifies that the restricted eigenvalue condition holds on $\widehat{\bm{\Omega}}^{(a)}_k$ with probability at least $1-O((nd)^{-Cd_*})$, assuming observations are independent. According to Remark \ref{remarkbridge}, when the independence assumption is replaced with the exponential $\beta$-mixing condition, the arguments presented in Section \ref{sec:proofthm1depdata}, particularly Equation \eqref{eqn:someimportantinequality}, are applicable to validate the same condition for MDP data. However, this adjustment slightly reduces the probability, which is now at least $1-O((nd)^{-Cd_*})-O\Big\{K^{-1}(NT)^{1-k_0\log(1/\rho)}\Big\}$. 


\textbf{\textit{Part II}}. In the main text, we assume that the reward is uniformly bounded and the state space is compact. This implies that all the entries of the error residuals $\{\bm{\varepsilon}_t\}_t$ (see the definition of $\bm{\varepsilon}_t$ in Section \ref{sec:poweranalysis}) are uniformly bounded as well. As outlined in Part I, considering the compactness of the state space, it is reasonable to truncate the knockoff states to the same compact space. Due to the boundedness of the state and the residual, applications of the Hoeffding's inequality \citep{hoeffding1994probability} and the Bonferroni's inequality allow us to show that both
\begin{eqnarray*}
    \max_j \Big|\sum_{(S_j,A,Y)\in \widetilde{\mathcal{D}}_k} S_j \mathbb{I}(A=a)Y\Big|\qquad\hbox{and}\qquad \Big|\sum_{(\widetilde{S}_j,A,Y)\in \widetilde{\mathcal{D}}_k} \widetilde{S}_j \mathbb{I}(A=a)Y\Big|
\end{eqnarray*}
are upper bounded by $O(\sqrt{n^{-1}\log (dn)})$, uniformly for each $j$, with probability at least $1-O((nd)^{-C})$ for any constant $C>0$. Together with Condition~\ref{ass:tuning} and the restricted eigenvalue condition established in Part I, this allows us to apply the proof techniques in Proposition C.3 of \citet{zhou2009restricted} to establish the following oracle inequalities, 
\begin{eqnarray}\label{eqn:oracleinequality}
    \|\widehat{\bm{b}}^{(a)}-\bm{b}^{(a)}\|_2=O(\sqrt{d_*n^{-1}\log (dn)})\,\,\hbox{and}\,\,\|\widehat{\bm{b}}^{(a)}-\bm{b}^{(a)}\|_1=O(d_*\sqrt{n^{-1}\log (dn)}),
\end{eqnarray}
uniformly for any $a\in \aspace$ with probability at least $1-O((nd)^{-C})$, where $\bm{b}^{(a)}$ denotes the ground truth.

Similarly, according to Remark \ref{remarkbridge}, when substituting the independence assumption with exponential $\beta$-mixing, the same oracle inequalities are maintained, at the cost of a slightly reduced probability, which is now $1-O((nd)^{-C})-O\Big\{K^{-1}(NT)^{1-k_0\log(1/\rho)}\Big\}$. 

\textbf{\textit{Part III}}. Let $\widehat{\bm{b}}$ and $\bm{b}$ be defined as $\widehat{\bm{b}}=\max\limits_a |\widehat{\bm{b}}^{(a)}|$ and $\bm{b}=\max\limits_a |\bm{b}^{(a)}|$, respectively, where both the maximum and absolute value operations are applied element-wise. It follows from \eqref{eqn:oracleinequality} that
\begin{eqnarray}\label{eqn:oracleinequality2}
    \|\widehat{\bm{b}}-\bm{b}\|_2=O(\sqrt{d_*n^{-1}\log (dn)})\,\,\hbox{and}\,\,\|\widehat{\bm{b}}-\bm{b}\|_1=O(d_*\sqrt{n^{-1}\log (dn)}).
\end{eqnarray}
These bounds confirm that the oracle inequality also holds for $\widehat{\bm{b}}$. Under Conditions \ref{ass:signal} and \ref{ass:noties}, the proof of \eqref{eqn:TPRind} follows very similarly to the proof of Theorem 3 of \citet{fan2019rank}; hence, we omit further details to avoid redundancy. 

This completes the proof of Theorem \ref{thm:power-lasso}.
\end{proof}

\subsubsection{Power of the Sequential Procedure}\label{sec:powersequential}
\begin{proof}[Proof of Theorem \ref{thm:type-II}]
    We focus on lower bounding the probability that, for each $k \in \{1, \ldots, K\}$, $\widehat{G}_k$ contains the minimal sufficient state $G_M$. Specifically, we aim to show that 
    \begin{eqnarray}\label{eqn:lowerbound}
        \mathbb{P}(G_M \subseteq \widehat{G}_k)\ge 1-O\Big\{(nd)^{-C}\Big\}-O\Big\{K^{-1}(NT)^{1-k_0 \log (\rho^{-1})}\Big\}
    \end{eqnarray}
    Once this has been proven, using similar arguments in \eqref{eqn:Markovinequality}, it can be shown that the probability of $\widehat{G}$ containing $G_M$ is only inflated by a factor of $(1-\alpha)^{-1}$. The proof of Theorem \ref{thm:type-II} can thus be completed. 

    It remains to prove \eqref{eqn:lowerbound}. Similar to the proofs of Theorems \ref{thm1} -- \ref{thm:power-lasso}, it suffices to show that with independent data, the probability can be lower bounded by
    \begin{eqnarray}\label{eqn:lowerbound1}
        \mathbb{P}(G_M \subseteq \widehat{G}_k)\ge 1-O\Big\{(nd)^{-C}\Big\}.
    \end{eqnarray}
    Under Conditions \ref{ass:external-dataset} and \ref{ass:knockoffs} -- \ref{ass:tuning}, the arguments presented in Part I and Part II of the proof of Theorem \ref{thm:power-lasso} remains applicable. Consequently, the oracle inequalities for the LASSO estimator \eqref{eqn:oracleinequality} remain valid with a probability of $1-O((nd)^{-C})$. Additionally, since the constant $C$ can be chosen sufficiently large, it follows from the Bonferroni's inequality that, uniformly across all responses (whether the reward or the $j$th next state for any $1\le j\le d$), the oracle inequalities are satisfied with probability $1-O((nd)^{-C})$. Therefore,  
    \begin{eqnarray*}
        \max_{1\le j\le d+1} \|\widehat{\bm{b}}^{(a)}_j-\bm{b}^{(a)}_j\|_2=O(\sqrt{d_* n^{-1}\log (dn)}),\,\,\max_{1\le j\le d+1} \|\widehat{\bm{b}}^{(a)}_j-\bm{b}^{(a)}_j\|_1=O(d_* \sqrt{n^{-1}\log (dn)}),
    \end{eqnarray*}
    where the first $d$th elements in $\bm{b}_j^{(a)}$ corresponds to the $j$th column of $\bm{B}^{(a)}$, whereas the last $d$the elements form a zero vector,  and $\widehat{\bm{b}}^{(a)}_j$ denotes the corresponding LASSO estimator. Moreover, similar to \eqref{eqn:oracleinequality2}, we can show that
    \begin{eqnarray}\label{eqn:oracleinequality3}
        \max_{1\le j\le d+1} \|\widehat{\bm{b}}_j-\bm{b}_j\|_2=O(\sqrt{d_* n^{-1}\log (dn)}),
    \end{eqnarray}
    where $\widehat{\bm{b}}_j=\max\limits_{a \in \mathcal{A}} |\widehat{\bm{b}}_j^{(a)}|$ and $\bm{b}_j=\max\limits_{a \in \mathcal{A}} |\bm{b}_j^{(a)}|$. Let $G_j^*$ denote the set of indices of those strong signals in the $j$th column of $\bm{B}$ whose coefficients are much larger than $\sqrt{d_* n^{-1}\log (dn)}$. It follows from \eqref{eqn:oracleinequality3} that for any index $g\in G_j^*$, their $W$-statistic $W_g=\widehat{b}_{j,g}-\widehat{b}_{j+d,g}$ is much larger than $\sqrt{d_* n^{-1}\log (dn)}$, whereas the remaining $W$-statistics are of the order $O(\sqrt{d_* n^{-1}\log (dn)})$. By definition, the knockoff threshold $\tau$ or $\tau_+$ cannot exceed $\min\limits_{g\in G_j^*} W_g$; see also the proof of Lemma 6 of \citep{fan2019rank}. Consequently, when we apply Algorithm \ref{alg:2} with response being the $j$th next state (or the immediate reward), the set $G_j^*$ will be included in the output with large probability. Moreover, uniformly across all responses, the inclusion of  $G_j^*$ occurs with probability $1-O\Big\{(nd)^{-C}\Big\}$, i.e.,
    \begin{eqnarray}\label{eqn:oracleinequality4}
        \mathbb{P}\Big( \bigcap_{1\le j\le d+1} \Big\{G_j^*\subseteq \widehat{G}_k(\{j\})\Big\} \Big)\ge 1-O\Big\{(nd)^{-C}\Big\},
    \end{eqnarray}
    where we recall that $\widehat{G}_k(\{j\})$ denotes the output of Algorithm \ref{alg:2} with response being the $j$th next state for $1\le j\le d$ and the immediate reward for $j=d+1$. 

    Under the minimal signal condition in Condition \ref{ass:signal2}, $G_M$ remains identifiable after removing edges in $\mathcal{G}(\bm{B})$ with moderate or weak signals. Consequently, using similar arguments in Section \ref{sec:indedataproof}, the event $G_M\subseteq \widehat{G}_k$ is implied by the union of the events $\Big\{G_j^*\subseteq \widehat{G}_k(\{j\})\Big\}$ for $1\le j\le d+1$. It follows from \eqref{eqn:oracleinequality4} that \eqref{eqn:lowerbound1} is proven. This completes the proof. 
\end{proof}

\section{Extensions and Discussions}\label{apdx:extension}
This section illustrates several practically meaningful extensions of the proposed method. Conceptionally, we extend the definition of minimal sufficient states to time-inhomogeneous MDPs (see Section~\ref{sec:extension-time-dependence}). Methodologically, we introduce a Bayesian method to control type-II errors in model-based selection (see Section~\ref{sec:control-type-II-error}). Algorithmically, we adapt the SEEK method to handle trajectories of varying horizons (see Section~\ref{sec:bestK-different-horizon}). Theoretically, we develop the theoretical guarantees under a weaker temporal dependence condition (see Section~\ref{sec:poly-beta-mixing}). Finally, we close this section with a comparison to model-free selection methods (see Section~\ref{sec:modelbasedmodelfree}).

\subsection{Extension I: Time-Dependent Minimal Sufficient State}\label{sec:extension-time-dependence}

This paper focus on a time-homogeneous Markov decision process where reward function and state transition function are invariant with respect to time. In practice, one possible scenario is that the reward or the transition function is nonstationary over time \citep{cheung2020reinforcement,li2022testing}. At this scenario, the minimal sufficient state becomes nonstationary as well. To address to this potential issue, below we have extended our proposal to handle nonstationary environments.
 
Specifically, for a given offline dataset covering the period from $T_0$ to $T$, the idea is to impose the following piecewise constant condition to approximate the nonstationray environment: 
\begin{condition}\label{cond:piecewise}
	Suppose there are $K$ many change points from $T_0 = 0$ to $T_{K+1}=T$ such that for any $0<k\le K$, the reward and state transition functions during the time interval $(T_k, T_{k+1}]$ remains constant over time. 
\end{condition}
Under Condition \ref{cond:piecewise}, each time interval $(T_k, T_{k+1}]$ has its own minimal sufficient state. To identify them, we first employ recently developed tools from the RL literature to identify the most recent change point $T_K$ in the reward and transition functions \citep{li2022testing,wang2023robust}. Let $\widehat{T}$ denote the estimated change point. We next apply the proposed SEEK to the data subset within the time interval $(\widehat{T}, T]$ to estimate its minimal sufficient state. Next, we adopt the following iterative procedure: 
\begin{itemize}
	\item Apply change point detection to the data subset within $[T_0, \widehat{T}]$ to identify the most recent change point that occured before $\widehat{T}$, denoted by $\widetilde{T}$.
	\item Apply SEEK to the data subset within the time interval $(\widetilde{T}, \widehat{T}]$ to estimate its minimal sufficient state.
	\item Update $\widehat{T}$ to $\widetilde{T}$.
\end{itemize}
This process continues until no further change points are detected within $[T_0, \widehat{T}]$. Finally, we output all the estimated change points and minimal sufficient states. 

Theoretically, we anticipate that our theories established in stationary settings will continue to hold with minor modifications in this context. Under Condition \ref{cond:piecewise}, the estimated change point is known to converge at a rate of $(NT)^{-1}$, which is much faster than the parameter rate of $(NT)^{-1/2}$ \citep[see e.g.,][]{fryzlewicz2014wild}. As a result, the FDR of the selected state will experience only slight inflation, and the power of the proposed method will be slightly decreased.

\subsection{Extension II: Controlling Type-II Error}\label{sec:control-type-II-error}

When the goal is solely to screen out the states that are not in the minimum sufficient state. In that case, one should consider metrics similar to type-II errors concerning the false negatives. For example, one may want to control the probability of having one or more false negatives within a target level $\alpha$, which can be viewed as a generalization of the familywise error rate to the type-II error. In that case, we may solve the problem under a Bayesian formulation. Specifically, the Markov property  allows us to factorize the conditional distribution of $(R_{it}, A_{it}, \mathbf{S}_{i,t+1})$, $t=1, \ldots, T-1$, $i=1, \ldots, N$, given $\mathbf{S}_{i1}, i=1, \ldots, N$, as
\begin{align*}
	\prod_{i=1}^{N}\prod_{t=1}^{T-1} f_1(R_{it}, \mathbf{S}_{i,t+1,G}|\mathbf{S}_{i,t,G},A_{it})f_2(\mathbf{S}_{i,t+1, G^{c}}\vert \mathbf{S}_{it}, A_{it},R_{it},\mathbf{S}_{i,t+1,G})f_3(A_{it}\vert \mathbf{S}_{it}),
\end{align*} 
where $f_1$, $f_2$, and $f_3$ are unknown functions.

We treat $G$ as an unknown random set, impose non-parametric Bayes priors on the functions $f_1$-$f_3$, as in \cite{ghosal2017fundamentals}, and a prior distribution on $G$ over all subsets of $\{1, \ldots, p\}$.  We then compute the posterior distribution of $G$ given the observed data (e.g., by Markov chain Monte Carlo methods) and select $\widehat{G}$ as the smallest subset of $\{1, \ldots, p\}$ satisfying the posterior probability $\mathbb{P}(|G \setminus \widehat{G}| \geq 1 \given \text{Data}) \leq \alpha$, which in turn guarantees the type-II familywise error rate to satisfy $\mathbb{P}(|G \setminus \widehat{G} | \geq 1) \leq \alpha$ in the Bayesian sense. Even if we evaluate this Bayesian procedure from a frequentist point of view, under suitable regularity conditions, one can still show the posterior distribution of $G$ to concentrate around the true sufficient minimal state, which further implies  $\mathbb{P}(|G \setminus \widehat{G}| \geq 1)$ to be asymptotically zero and the type II familywise error rate to be asymptotically controlled. For the posterior concentration to hold, we require regularities on the true functions $f_1$ to $f_3$, and the corresponding regularities on their priors. If $p$ diverges with $N$ and $T$, a condition is further needed on its diverging speed relative to $N$ and $T$. We refer to Chapter 10.5 in \cite{ghosal2017fundamentals}, for the establishment of such posterior concentration results. 

\subsection{Extension III: Trajectories with Different Horizons}\label{sec:bestK-different-horizon}
Our three-step algorithm in Section~\ref{sec:seek} can readily accommodate cases where the horizon differs across trajectories. Recall that the algorithm involves three key steps: data splitting, iterative selection, and majority voting. We have observed that the varying horizons only affect the initial data splitting step. Once the data is divided into several subsets, the iterative selection and majority voting steps proceed without modification. Specifically, the varying horizons influence the data splitting by affecting both (i) the determination of the optimal number of splits $K$ and (ii) the division of the data into $K$ subsets. Actually, both procedures~(i) and (ii) can be adapted to accommodate varying horizons as we can see below.

\textbf{\textit{Determining the optimal $K$}}. Assume the offline data consists of transition tuples $\cup_{i=1}^N\{(\mathbf{S}_{i, t}, A_{i, t}, R_{i, t}, \mathbf{S}_{i, t+1})\}_{i=1}^{T_i}$ where $T_i$ is the length of horizon of the $i$th trajectory and can vary for different $i$. When the horizon is constant across trajectories, the first two steps in Section~\ref{sec:bestK-algo} remain applicable to produce the least square estimators for the $\beta$-mixing coefficients $\{\widehat{\beta}(k)\}_k$, and in the last step, we set 
\begin{align*}
	\widehat{K} = \argmin\limits_k\{k\ge 1: \sum_{i=1}^N T_i \widehat{\beta}(k)/k \le \delta \}
\end{align*}
for some $\delta>0$. 
    
\textbf{\textit{Splitting into $K$ data subsets}}. For each $1\le k\le \widehat{K}$, we set $T_i(k)$ be the largest integer no larger than $T_i$ such that $T_i(k)-k$ is divisible by $\widehat{K}$. We next allocate to the data subset $\mathcal{D}_k$ to contain the $i$th trajectories' $k$th, $(k+\widehat{K})$th, $\cdots$, and $(k+T_i(k)\widehat{K})$th tuples for all $1\le i\le N$. It is immediate to see that within each $\mathcal{D}_k$, any two different tuples either come from the same trajectory with a time gap of at least $\widehat{K}$, or belong to two different trajectories, becoming ``approximately independent''. 

\subsection{Extensions IV: Polynomially $\beta$-mixing Condition}\label{sec:poly-beta-mixing}

This section introduces the polynomially $\beta$-mixing condition to relax the exponential $\beta$-mixing imposed in Condition~\ref{ass:mixing}. This condition is formally introduced below. 
\begin{condition}[Stationarity and polynomial $\beta$-mixing]\label{ass:poly-mixing}
	The process $\{(\mathbf{S}_t, A_t, R_t)\}_{t\ge 0}$ is stationary and polynomially $\beta$-mixing, i.e., its $\beta$-mixing coefficients $\beta(k)$ (see its expression in~\eqref{eq:beta-i}) satisfy
	\begin{equation}\label{eq:poly-beta-i}
		\beta(k) = O(k^{-\rho})
	\end{equation}
	for some constant $\rho > 1$.
\end{condition}
It is immediate to see that Condition \ref{ass:poly-mixing} is much weaker than exponential $\beta$-mixing, given that the $\beta$-mixing coefficients under this condition decay to zero at a much lower rate. Under this new condition, the procedure for estimating $K$ as well as Theorems~\ref{thm1}-\ref{thm2} shall be modified accordingly. We would present the modified procedure in Sections~\ref{sec:number-split-poly} and the updated theoretical results in Section~\ref{sec:theory-poly}. 

\subsubsection{Determine the Number of Splits}\label{sec:number-split-poly}
We next detail the methodology to determine the optimal number of splits under polynomial $\beta$-mixing. We impose the following model assumption for the first-step estimators: 
\begin{align}\label{eqn:polyinitial}
	\widetilde{\beta}(k) = \eta_0 + a_0 k^{-b_0} + \epsilon_k.
\end{align}
Compared to \eqref{eq:init-beta-estimator}, the main effect term in \eqref{eqn:polyinitial} decays polynomially in $k$. The parameters $a_0, b_0$ and $\eta_0$ 
can be similarly estimated by minimizing the following least square loss function
\begin{align*}
	(\widehat{\eta}, \widehat{a}, \widehat{b})=\arg\min\limits_{\eta_0, a_0, b_0} \sum_{k=1}^{K_0}\left(\widetilde{\beta}(k) - \eta_0 - a_0 k^{-b_0}\right)^2.
\end{align*}
Finally, we set the number of splits $\widehat{K}$ as
\begin{align*}
	\widehat{K} = \argmin \left\{k\geq 1: \frac{NT}{k} \widehat{a}_0 k^{-\widehat{b}_0}  \leq \delta \right\}. 
\end{align*}

\subsubsection{Theoretical Guarantees}\label{sec:theory-poly}
Under the new temporal dependence condition (Condition~\ref{ass:poly-mixing}), we can derive a new theoretical result concerning FDR. The result is presented below, along with its proofs.

\begin{theorem}[FDR]\label{thm:fdr-poly}
	Suppose Conditions \ref{ass:poly-mixing}, \ref{ass:external-dataset} and \ref{condnull} hold. Set the number of sample splits $K = k_0 (NT)^{\frac{1}{\rho-1}}$ for a constant $k_0>0$ where $\rho$ is defined in Condition \ref{ass:poly-mixing}. Then for any response $Y$ ($R$ or $S'_{l}$ for a given $1\le l\le d$), $\widehat{G}_k$ obtained by Algorithm \ref{alg:2} with standard knockoffs satisfies
	\begin{eqnarray*}
		\textrm{mFDR}(\widehat{G}_k)\leq q \exp(\epsilon)+\mathbb{P}\Big(\max_{j\in \mathcal{H}_0}\widehat{\textrm{KL}}_{j}>\epsilon\Big)+O\Big\{ K^{-1} (NT)^{-\frac{1}{\rho-1}} \Big\},
	\end{eqnarray*}
	for any $\epsilon>0$, where $\widehat{\textrm{KL}}_{j}$ is the same as that in Theorem~\ref{thm1}. In addition, $\widehat{G}_k$ obtained by implementing the knockoffs+ in Algorithm \ref{alg:2} satisfies 
	\begin{align*}
		\textrm{FDR}(\widehat{G}_k)\leq q \exp(\epsilon)+\mathbb{P}\Big(\max_{j\in \mathcal{H}_0}\widehat{\textrm{KL}}_{j}>\epsilon\Big)+O\Big\{ K^{-1} (NT)^{-\frac{1}{\rho-1}} \Big\}.
	\end{align*}
\end{theorem}
\begin{proof}
The proof for Theorems~\ref{thm:fdr-poly} follows similar steps in proving Theorems~\ref{thm1}. Specifically, the proof in Section~\ref{sec:proofthm1depdata} remains applicable under the polynomial $\beta$-mixing condition, with two key modifications: (i) $K$ is replaced by $k_0(NT)^{\frac{1}{\rho-1}}$, and (ii) the bound on $\beta_i$ (given in Equation~\eqref{eq:beta_i}) is changed to:
\begin{align}\label{eq:beta_i_poly}
\beta_i \leq \beta(i) \leq O(K^{-\rho}) = O(k_0^{-\rho} (NT)^{\frac{-\rho}{\rho-1}}),
\end{align}
under Condition~\ref{ass:poly-mixing}. Substituting this new bound for $\beta_i$ into the proof of Theorem~\ref{thm1} directly yields the result of Theorem~\ref{thm:fdr-poly}.
\end{proof}

Next, we present results regarding Type-I error, TPR, and Type-II error. For brevity, we omit the proofs of these theorems, as they follow similar lines of reasoning as those for Theorems~\ref{thm2}-\ref{thm:type-II}.
\begin{theorem}[Type-I error]\label{thm:type-I-error-poly}
	Under the same conditions in Theorem~\ref{thm:fdr-poly}, then as the FDR target $q$ approaches zero, the probability that Algorithm~\ref{alg:1} with knockoffs+ selects any null variable is upper bounded by
	$O(\alpha^{-1} d_0^2 q) + O\Big\{\alpha^{-1}K^{-1}(NT)^{-\frac{1}{\rho-1}}\Big\}$.
\end{theorem}
\begin{theorem}[TPR]\label{thm:power-lasso-poly}
	Support Conditions \ref{ass:external-dataset}, \ref{ass:signal}, and \ref{ass:knockoffs} to \ref{ass:poly-mixing} hold. Under the same notions in Theorem~\ref{thm:power-lasso}, the selected set $\widehat{G}_k$ returned by Algorithm \ref{alg:2} satisfies
    \begin{align}
        \mathbb{E}\left(\frac{|\widehat{G}_k\cap G(\{i\})|}{|G(\{i\})|}\right) \geq 1 -O(\kappa_n^{-1}) - O\left\{K^{-1}(NT)^{-\frac{1}{\rho-1}}\right\}-O\left\{(nd)^{-C}\right\},
    \end{align}
    for any large constant $C>0$. 
\end{theorem}
\begin{theorem}[Type-II error]\label{thm:type-II-poly}
    Under Conditions \ref{ass:mixing}, \ref{ass:external-dataset}, \ref{ass:signal2}, and \ref{ass:knockoffs} to \ref{ass:tuning}, the probability that the selected set $\widehat{G}$ returned by Algorithm \ref{alg:1} contains the minimal sufficient state is lower bounded by $$1-O\left\{(1-\alpha)^{-1}(nd)^{-C}\right\}-O\left\{(1-\alpha)^{-1}K^{-1}(NT)^{-\frac{1}{\rho-1}}\right\}.$$ 
\end{theorem}

\subsection{Discussion: Analytic Comparison on Model-free Selections}\label{sec:modelbasedmodelfree}
As commented in Section \ref{sec:minimal}, existing variable selection approaches are typically model-free. For instance, \citet{hao2021sparse} proposed to combine LASSO with fitted Q-iteration for selecting relevant variables in the optimal Q-function. Their algorithm is model-free, since it directly models the Q-function rather than the reward or state transition function. To the contrary, the proposed minimal sufficient state is model-based, defined by the reward and state transition functions (see Definition \ref{def:sufficient-state}). Consequently, the targets for variable selection between our proposal and \citet{hao2021sparse} differ, although both are sufficient for optimal policy learning. A closer look of the proof for Proposition 1 in Section \ref{sec:proofprop1} of the Supplementary Materials reveals that the optimal Q-function depends on all state variables only through the proposed minimal sufficient state only. Therefore, while their variable selection target is covered by the proposed minimal sufficient state, it can be strictly smaller than the minimal sufficient state itself. However, as commented in the introduction, their proposal requires linear function approximation and cannot handle complex nonlinear systems.

\end{document}